\documentclass[journal]{IEEEtran}

\usepackage[pdftex]{graphicx}
\usepackage[colorlinks=false]{hyperref}
\usepackage{subfigure}
\usepackage[space]{cite}
\usepackage{amsmath,amssymb,amsopn,amstext,amsfonts, amsthm}
\usepackage{amsmath,amsopn,amstext,amsfonts,mathtools}
\usepackage{booktabs}
\usepackage{xcolor}
\usepackage{threeparttable}
\usepackage{yfonts}
\newtheorem{theorem}{Theorem}
\newtheorem{corollary}{Corollary}[theorem]

\newtheorem{definition}{Definition}
\newtheorem{remark}{Remark}

\usepackage{multirow}
\usepackage{bbm}
\usepackage{soul}
\usepackage[ruled,vlined,linesnumbered]{algorithm2e}

\usepackage{pdfpages}

\hypersetup{
	hidelinks,
	colorlinks=false,
	linkcolor=black,
	citecolor=black
}

\begin{document}

\title{Efficient and Consistent Bundle Adjustment on Lidar Point Clouds}

\author{Zheng Liu, Xiyuan Liu and Fu Zhang
\thanks{Manuscript received September 15, 2022, revised January 20, 2023 and May 2, 2023, accepted August 28, 2023. This work is supported in part by the University Grants Committee of Hong Kong General Research Fund (project number 17206421) and in part by DJI under the grant number 200009538. (\textit{Corresponding author: Fu Zhang.})

The authors are with the Mechatronics and Robotic Systems (MaRS) Laboratory, Department of Mechanical Engineering, University of Hong Kong, Hong Kong SAR, China (e-mail: u3007335@connect.hku.hk; xliuaa@connect.hku.hk; fuzhang@hku.hk).
}}

% \markboth{{IEEE TRANSACTIONS ON ROBOTICS}}
% {Liu \MakeLowercase{\textit{et al.}}: Efficient and Consistent Bundle Adjustment on Lidar Point Clouds}

% make the title area
\maketitle
\begin{abstract}
Simultaneous determination of sensor poses and scene geometry is a fundamental problem for robot vision that is often achieved by Bundle Adjustment (BA). This paper presents an efficient and consistent bundle adjustment method for lidar sensors. The method employs edge and plane features to represent the scene geometry, and directly minimizes the natural Euclidean distance from each raw point to the respective geometry feature. A nice property of this formulation is that the geometry features can be analytically solved, drastically reducing the dimension of the numerical optimization. To represent and solve the resultant optimization problem more efficiently, this paper then {adopts and formalizes the concept of} \textit{point cluster}, which encodes all raw points associated to the same feature by a compact set of parameters, the \textit{point cluster coordinates}. We derive the closed-form derivatives, up to the second order, of the BA optimization based on the point cluster coordinates and show their theoretical properties such as the null spaces and sparsity. Based on these theoretical results, this paper develops an efficient second-order BA solver. Besides estimating the lidar poses, the solver also exploits the second order information to estimate the pose uncertainty caused by measurement noises, leading to consistent estimates of lidar poses. Moreover, thanks to the use of point cluster, the developed solver fundamentally avoids the enumeration of each raw point in all steps of the optimization: cost evaluation,  derivatives evaluation and uncertainty evaluation. \addtocounter{footnote}{1} The implementation of our method is open sourced to benefit the robotics community\footnote{\url{https://github.com/hku-mars/BALM}}.

% The proposed method is extensively evaluated at different levels: consistency, accuracy, and computation efficiency in both simulated and actual environments. Benchmark evaluation on 19 real-world open sequences covering various datasets (Hilti, NTU-VIRAL and UrbanLoco), environments (campus, urban streets, offices, laboratory, and construction sites), lidar types (Ouster OS0-64, Ouster OS1-16, Velodyne HDL 32E), and motion types (handheld, UAV-based, and ground vehicles-based) shows that our method achieves consistently and significantly higher performance than other state-of-the-art counterparts in terms of localization accuracy, mapping quality, and computation efficiency. In particular, our method achieves a mapping accuracy at a level of the lidar measurement noise (i.e., a few centimeters) while processing all sequences in less than half minute on a standard desktop CPU. Finally, we show how our proposed method effectively improves the accuracy and/or computation efficiency of some important robotic techniques, including lidar-inertial odometry, multi-lidar extrinsic calibration, and high-accuracy global mapping. The implementation of our method is open sourced to benefit the robotics community and beyond\footnote{\url{https://github.com/hku-mars/BALM}}.
\end{abstract}
% \vspace{-0.1cm}

\IEEEpeerreviewmaketitle
% \vspace{-0.5cm}
\section{Introduction}

%In recent years, simultaneous localization and mapping (SLAM) based on light detection and ranging (lidar) sensors has drawn an increasing amount of research interests \cite{lin2020loam, shan2020lio, li2021towards, xu2021fast, wang2021lightweight, jiao2021robust, park2021elasticity, xu2022fast}. Enabled by the direct, dense, active and accurate (DDAA) depth measurements of lidar sensors, lidar SLAM has the ability to build a dense and accurate 3D map of the environment in real-time and at relatively low computation cost. These unique advantages have made lidar SLAM an essential technique in applications requiring real-time, dense, and accurate 3D mapping of the environment, such as autonomous driving \cite{thrun2006stanley, urmson2008autonomous, levinson2011towards, li2020lidar}, unmanned aerial vehicles navigation \cite{gao2019flying, kong2021avoiding, ren2022bubble}, and real-time mobile mapping \cite{schwarz2010mapping, bosse2012zebedee, helmberger2022hilti}.  This trend becomes more evident with recent developments in lidar technologies which have enabled the commercialization and mass production of lightweight and high-performance solid-state lidars at a significantly lower cost \cite{wang2020mems, liu2021low}. 

\IEEEPARstart{L}{ight} detection and ranging (lidar) has become an essential sensing technology for robots to achieve a high level of autonomy \cite{thrun2006stanley, urmson2008autonomous}. Enabled by the direct, dense, active and accurate (DDAA) depth measurements, lidar sensors have the ability to build a dense and accurate 3D map of the environment in real-time and at a relatively low computation cost. These unique advantages have made lidar sensors essential to a variety of applications that require real-time, dense, and accurate 3D mapping of the environment, such as autonomous driving \cite{levinson2011towards, li2020lidar}, unmanned aerial vehicles navigation \cite{gao2019flying, kong2021avoiding, ren2022bubble}, and real-time mobile mapping \cite{schwarz2010mapping, bosse2012zebedee, helmberger2022hilti}.  This trend becomes even more evident with recent developments in lidar technologies which have enabled the commercialization and mass production of lightweight and high-performance solid-state lidars at a significantly lower cost \cite{wang2020mems, liu2021low}.

%been widely applied in simultaneous localization and mapping (SLAM), because of the ability to measure ranges directly \cite{schwarz2010mapping}. lidar SLAM has shown great potential in autonomous driving \cite{li2020lidar}, unmanned drone navigation \cite{kong2021avoiding} and robot vacuum cleaner \cite{hess2016real}. Except conventional mechanical spinning lidar, solid lidar has drawn much more attention recently due to its large-scale production, light-weight, lower price and high performance \cite{zhang2018comparison, wang2020mems, liu2021low}. Some lidar SLAM has been adapted to solid lidar and achieve good results \cite{lin2020loam, xu2021fast, xu2022fast, li2021towards}.

\iffalse
\begin{figure} [t]
	\centering
	\includegraphics[width=8.5cm]{figures/sparse.png}
	\caption{The point cloud in one lidar scan (Velodyne HDL 32E) \cite{wen2020urbanloco}. (a) Top view of lidar scan. (b) View of angle from white arrow in (a). (c) Zooming in from white frame in (b) and the white vehicle is the accumulated point cloud. Only several scanning line is on the vehicle and thus it is impossible to get accurate corresponding corner points between different lidar scans.}
	\label{fig sparse}
\end{figure}
\fi

The central task of many lidar-based techniques, such as lidar-based odometry, simultaneous localization and mapping (SLAM), and multi-lidar calibration,  is to register multiple point clouds, each measured by the lidar at different poses, into a consistent global point cloud map. However, the predominant point cloud registration methods, such as iterative closest point (ICP) \cite{besl1992method} and its variants (e.g, generalized-ICP \cite{segal2009generalized}), normal distribution transformation (NDT) \cite{biber2003normal, magnusson2009three}, and surfel registration \cite{behley2018efficient}, allow registration of two point clouds only. Such a pairwise registration leads to an incremental scan registration process for an odometry system (e.g., \cite{zhang2014loam, xu2022fast, yokozuka2021litamin2, behley2018efficient}), which would rapidly accumulate drift, or a repeated pairwise registration process for 3D mapping \cite{surmann2003autonomous} or multi-lidar calibration \cite{liu2021calib}, which would bring dramatic computation cost. All these necessitate an efficient concurrent multiple scan registration technique. %of multiple lidar scans is a fundamental technique to many applications, such as low-drift odometry based on sliding window optimization \cite{liu2021balm}, high-accuracy mapping, and multi-lidar calibration \cite{liu2022calib}. %crucially important to improve the mapping accuracy \cite{liu2022calib} and odometry drfit \cite{liu2021balm}. Simultaneously registering all lidar scans such as XXX based on (generalized-)ICP, XXX based on NDT, and XXX based on surfel registration, where points in each new scan are registered to past points in the map and then merged to the map for the registration of the next scan. Such an incremental mapping process will.  

Concurrent multiple lidar scan registration requires determining all lidar poses and the scene geometry simultaneously, a process referred to as {\it bundle adjustment (BA)} in computer vision. Compared to visual BA, which has been well-established in photogrammetry and played a fundamental role in various vital applications, including visual odometry (VO) \cite{klein2007parallel, mur2015orb, mur2017orb}, visual-inertial odometry (VIO) \cite{qin2018vins, campos2021orb}, 3D visual reconstruction \cite{schoenberger2016mvs, moulon2016openmvg} and multi-camera calibration \cite{li2013multiple, zaharescu2006multiple}, lidar BA has a similarly fundamental role but is much less mature due to two major challenges. First, lidar has a long measuring range but low resolution between scanning lines. The measured point cloud are sparsely (sometimes even not repeatedly \cite{liu2021low}) distributed in a large 3D space, making it difficult (almost impossible) to scan the same point feature in the space across different scans. This has fundamentally prevented the use of straightforward visual bundle adjustment formulation, which is largely based on point features benefiting from the high-resolution images accurately capturing individual point features. The second challenge lies in the large number of raw points (from tens of thousands to million points) collected by a practical lidar sensors. Processing all these points in the lidar BA is extremely computation intensive. 

In this work, we propose an efficient and consistent BA framework specifically designed for lidar point clouds. The framework follows our previous work BALM \cite{liu2021balm}, which formulates the lidar BA problem based on edge and plane features that are abundant in lidar scans. The BA formulation naturally minimizes the straightforward Euclidean distance of each point in a scan to the corresponding edge or plane, while the decision variables include the lidar poses and feature (edge and plane) parameters. Furthermore, it is shown that the geometry parameters (i.e., edge and plane) can be solved analytically, leading to an optimization that depends on the lidar poses only. Since the number of geometry features is often large, elimination of these geometry features from the optimization will drastically reduce the optimization dimension (hence time).
	
{A key concept our proposed BA framework adopts and formalizes is} the \textit{point cluster} \cite{ferrer2019eigen, zhou2020efficient, huang2021bundle}, which summarizes all points of a lidar scan associated to one feature by a compact set of parameters, \textit{point cluster coordinates}.
Based on the point cluster, we derive the closed-form derivatives (up to second order) of the BA optimization with respect to (w.r.t.) its decision variables (i.e., lidar poses). We prove that the formulated BA optimization and the closed-form derivatives can both be represented fully by the point cluster without enumerating the large number of individual points in a lidar scan. The removal of dependence on individual raw points drastically speeds up the evaluation of the cost function and derivatives, which further enables us to develop an efficient and consistent second-order solver, BALM2.0, which is also released on Github to benefit the community. Our experiment video is available on website\footnote{\url{https://youtu.be/MDrIAyhQ-9E}}.

We conduct extensive evaluations on the proposed BA method. Simulation study shows that the BA method produces consistent lidar pose estimate. Exhaustive benchmark comparison on 19 real-world open sequences shows that the BA method produces consistently higher performance (pose estimation accuracy, mapping accuracy, and computation efficiency) than other counterparts. We finally integrate the BA method in three vital lidar applications: lidar-inertial odometry, multi-lidar calibration, and global mapping, and show how their accuracy and/or computation efficiency are improved by the proposed BA. 

\vspace{-0.3cm}
\section{Related Works} \label{relateWork}

%Current lidar SLAM works \cite{zhang2014loam, shan2018lego, lin2020loam, xu2021fast, ye2019tightly, shan2020lio} mainly base on pairwise registration, aligning the newest lidar scan with previous point-cloud map. These methods have inherent limitation to accuracy and is unable to remove the accumulated drifts. It is true that loop-closure can serve as the back-end of SLAM to adjust all the poses \cite{lin2019fast, kim2018scan, kim2021scan}, but it is required to revisit past locations. Since the pose-graph optimization (PGO) is only depending on poses information rather than point-cloud, the point-cloud map may be still layered after PGO. In visual SLAM, bundle adjustment (BA) is adapted to deal with relevant problems and becomes a standard module \cite{klein2007parallel, mur2015orb, mur2017orb, qin2018vins}.

\subsection{Multi-view registration}

%Except these pair-wise ICP registration, point-cloud multi-view registration is investigated by many researchers \cite{bergevin1996towards, govindu2013averaging, pulli1999multiview, zhu2019efficient}. \cite{bergevin1996towards} built a well-balanced network of poses and used pair-wise registration repeatedly until the updating variance of poses small enough. The convergence of this method is slow and the accuracy of scan-to-scan match for lidar is poor due to its sparsity. \cite{govindu2013averaging} designed a sophisticated pose-graph optimization by using the averaging of relative motions, while 3D points are not considered in the optimization. \cite{pulli1999multiview} used the point pairs extracted from point-to-point ICP for multi-view registration. It is true that this method is efficient and memory-saving, but as mentioned above, lidar scan is sparse and can not find accurate corresponding points. \cite{zhu2019efficient} firstly segmented point cloud by K-means clustering and minimize the total distance from each point to centroid in each cluster. This approach requires K-means clustering and to recompute centroid for each iteration. Due to the large measurement range and considerable points of lidar, K-means clustering for each iteration is time-consuming. Moreover, the time complexity of transformation estimation is dependent on the number of points, which is also time-consuming for lidar SLAM. In summery, these multi-view registration method is designed for dense point-cloud and not applicable to lidar scans. 

The bundle adjustment problem is similar to the {\it multi-view registration} problem that has been previously researched \cite{bergevin1996towards, lu1997globally, pulli1999multiview, huber2003fully, borrmann2008globally, govindu2013averaging}. These methods all adopt a two layer framework: the first layer estimates the relative poses of a selected set of scan pairs using the pairwise registration methods (e.g., ICP \cite{besl1992method}); From the relative poses, the second layer constructs and solves a pose graph to obtain a maximum a posteriori estimate of all lidar poses. Such a two-layer framework decouples the raw point registration from the global pose estimation, so that each raw point registration only involves a small amount of local points contained in the two scans (instead of all scans sharing overlaps) and the pose graph optimization only involves a small amount of constraints arising from the relative poses (instead of raw points). The net effect is a significant saving of time, hence being largely used in online lidar SLAM systems \cite{shan2020lio, koide2021globally}. However, the advantage in computation efficiency comes with fundamental limitation in accuracy: the pairwise scan registration only considers the overlap among two scan at a time, while the overlap is really shared by all scans and should be registered concurrently. Moreover, the pose graph optimization only considers constraints from the relative poses, while the mapping consistency indicated by the raw points are completely ignored. Consequently, it is usually difficult to produce (or even be aware of) a globally consistent map that is necessary for high-accuracy localization and mapping tasks. 

Some early works in computer vision and computer graphics have proposed multi-view registration methods that directly optimize the mapping consistency from multiple range images, aiming for consistent surface  modeling of 3D  objects. \cite{blais1995registering} is a direct extension of the ICP method, it minimizes the Euclidean distance between a pre-known control point in one scan to all matched control points in the rest scans. Within this framework, \cite{benjemaa1998solution} uses a quaternion representation in the optimization, and \cite{neugebauer1997reconstruction} extends the distance between control points to the distance between surfaces around the respective control points. More recently, Zhu {\it et al.} \cite{zhu2019efficient} proposes a two step registration method: the first step uses a K-means clustering to cluster points from all scans, and the second step estimates the scan poses by minimizing the Euclidean distance between each point in a cluster to the centroid. Since these methods rely on point features in the scan, they require densely populated point cloud (e.g., by depth camera) for extracting such salient point features. While this is not a problem for small object reconstruction for which these methods are designed, it is not the case for scene reconstruction where the LiDAR measurements are very sparse (sometimes even non-repetitive) as explained above. 

%\subsection{Correlation-based scan registration}

\subsection{Bundle or plane adjustment}

In recent years, researchers in the robotics community have shown increasing interests to address the bundle adjustment problem on (lidar) point clouds more formally. Kaess \cite{kaess2015simultaneous} exploits the plane features in the bundle adjustment and minimizes the difference between the plane measured in a scan and the plane predicted from the optimization variables: scan poses and plane parameters. This formulation was later integrated into a key-frame-based online SLAM system \cite{hsiao2017keyframe}. Since the method minimizes the plane-to-plane distance, it requires to segment each scan and estimate the contained local planes in advance. Such plane segmentation and estimation usually require dense point clouds measured by RGB-D cameras on which the work were demonstrated.  

A more formal bundle adjustment method on lidar point cloud, termed as the {\it plannar (bundle) adjustment}, was later proposed in \cite{zhou2020efficient} which minimizes the natural Euclidean distance between each point in a scan to the plane predicted from the scan poses and plane parameters (the optimization variables). Compared with plane-to-plane distance in \cite{kaess2015simultaneous}, the point-to-plane metric is faster, more accurate, and more suitable for lidar sensors, where local plane segmentation or estimation are less reliable due to sparse point clouds. Moreover, the direct use of raw points in the point-to-plane metric could also lead to a more consistent estimate of the optimization variables by considering the measurement noises in the raw points. Then, the formulated non-linear least square problem is solved by a Levenberg-Marquardt (LM) algorithm. To lower the computation load caused by the large number of points measurements associated to the same plane feature, \cite{zhou2020efficient} propose a reduction technique to eliminate the enumeration of individual points in the evaluation of the residual and Jacobian. Furthermore, due to the very similar structure to the visual bundle adjustment, the proposed bundle adjustment is also compatible with the Schur complement trick \cite{triggs1999bundle}, which eliminates the plane parameters in each iteration of the LM algorithm. This plane adjustment method is largely used in many online lidar SLAM systems developed subsequently \cite{zhou2021pi, zhou2021lidar} or before \cite{geneva2018lips, trevor2012planar}. 

On the other hand, Ferrer \cite{ferrer2019eigen} exploits plane features similar to \cite{zhou2020efficient} and minimizes any deviation of each raw point from the plane equation. The resultant optimization cost then reduces to the minimum eigenvalue of a covariance matrix and is thus termed as the eigen-factor. The author further derived the closed-form gradient of the cost function w.r.t. to both the scan poses and plane parameters and employed a gradient-based method to solve the optimization iteratively. Due to the second order nature of the eigenvalue (as confirmed in \cite{liu2021balm}, see below), the gradient method converges very slowly (requiring a few hundreds of iterations) \cite{ferrer2019eigen, zhou2021lidar}. 

Our previous work BALM \cite{liu2021balm} takes another step towards more efficient bundle adjustment. Similar to \cite{zhou2020efficient}, BALM minimizes the natural Euclidean distance between each point in a scan to the plane (i.e., point-to-plane metric). Based on this cost metric, BALM proved that all plane parameters can be analytically solved with closed-form solutions in advance, hence the large number of plane parameters can be completely removed from the resultant optimization. {Such an elimination of plane parameters is analogous to the well-studied separable least-squares problem \cite{strelow2012general, golub1973differentiation, wiberg1976computation, ruhe1980algorithms}  in general, but is specifically designed for the LiDAR BA problem. A prominent advantage of the feature elimination is the significant reduction of optimization dimension, which} poses a fundamental difference from all the previous plannar adjustment methods {\cite{geneva2018lips, zhou2020efficient, zhou2021pi, zhou2021lidar}} and visual bundle adjustment methods \cite{triggs1999bundle}. The feature elimination also removed the various issues caused by plane representation in the optimization, such as the normal constraints in the Hesse normal representation $(\mathbf n, d)$  \cite{zhou2020efficient}, singularity issue in the closest-point (CP) representation $\mathbf n d$ \cite{zhou2021pi, zhou2021lidar,geneva2018lips} and over-parameterization issue in the quaternion representation \cite{kaess2015simultaneous}. {With the feature elimination, BALM further proved that the point-to-plane (or edge) distance is essentially the eigenvalues of the covariance matrix used in \cite{ferrer2019eigen}, thus unifying the two metrics in \cite{zhou2020efficient} and \cite{ferrer2019eigen}. While both BALM and \cite{ferrer2019eigen} eliminate the feature from the BA optimization, \cite{ferrer2019eigen} uses a gradient method to solve this optimization, which leads to very slow convergence as reviewed above. In contrast}, BALM \cite{liu2021balm} derived the second order derivatives of the cost function and developed a LM-like second-order solver. {The developed solver requires significantly less iterations to converge}, achieving real-time sliding window optimization when integrated to LOAM \cite{zhang2014loam}. A further advantage of BALM against previous method  \cite{geneva2018lips, ferrer2019eigen, zhou2020efficient} is that the whole framework is naturally extendable to edge features besides plane features. 

A major drawback of the BALM \cite{liu2021balm} is that the evaluation of the second-order derivatives including Jacobian and Hessian requires to enumerate each individual lidar point, leading to a computational complexity of $O(N^2)$ where $N$ is the number of points  \cite{zhou2021lidar}. Consequently, the method is hard to be used in large-scale problems where the lidar points are huge in number. This problem is partially addressed in \cite{huang2021bundle}, which aggregates all points associated to the same plane feature in a scan in the scan local frame. However, to ensure convergence, \cite{huang2021bundle} modifies the cost function by including an extra heuristic penalty term, which is not a true representation of the map consistency. Moreover, the cost function in \cite{huang2021bundle} still involves the plane feature similar to \cite{geneva2018lips, ferrer2019eigen, zhou2020efficient, zhou2021pi, zhou2021lidar}. To lower the computation load caused by optimizing the large number of feature parameters, the method further fixes the feature parameters in the optimisation, which could slow down the optimisation speed.

Our BA formulation in this paper is based on BALM \cite{liu2021balm}, hence inheriting the fundamental feature elimination advantage when compared to {\cite{geneva2018lips, zhou2020efficient, zhou2021pi, zhou2021lidar} and the fast convergence advantage when compared to \cite{ferrer2019eigen}}. To address the computational complexity of $O(N^2)$ in BALM, we {adopt and formalize the concept of} {\it point cluster}, which fundamentally eliminates the enumeration of each individual point in the evaluation of the cost function, Jacobian and Hessian matrix. Consequently, the computational complexity is irrelevant to both the feature dimension (similar to BALM \cite{liu2021balm}) and the point number (similar to \cite{zhou2020efficient, zhou2021pi}). The point cluster in our method is similar to the point aggregation used in \cite{huang2021bundle} {(and also used in \cite{ferrer2019eigen, zhou2020efficient})}, but the overall BA formulation is fundamentally different: (1) it minimizes the true map consistency (the point-to-plane distance) without trading off with any other heuristic penalty; and (2) it performs exact feature elimination with rigorous proof instead of empirical fixation. Based on these nice theoretical results, we develop an efficient second-order solver, termed as BALM2.0. Besides solving the nominal lidar poses, the solver also estimates the uncertainty of the estimated lidar pose by leveraging the second-order derivative information, which is another new contribution compared with existing works.

%Besides the formulation, previous BA methods also differ significantly in feature extraction. Except \cite{zhou2020efficient, ferrer2019eigen}, which focus on the BA formulation without any consideration of the plane feature extraction or matching, previous methods segment planes in a local scan by methods like \cite{holz2011real} (e.g., \cite{kaess2015simultaneous, hsiao2017keyframe}), RANSAC (e.g., \cite{geneva2018lips}), or based on region growing \cite{poppinga2008fast} (e.g., \cite{zhou2021pi, zhou2021lidar}). Our previous work BALM \cite{liu2021balm} proposes an adaptive voxelization method, which is able to extract plane (and edge) features at different scale, from large planes (e.g. wall, the ground) to small cluttered plannar patches (e.g. on tree crowns, trunk). Due to the versatility and adaptability to different environments, this method has been adopted by many following-up applications, including BA \cite{huang2021bundle}, odometry \cite{yuan2022efficient}, and multi-lidar calibration \cite{liu2022calib}, and proved very effective in various environments, such as outdoor urban environments, cluttered indoor environments, and unstructured woods environments. Our method in this work utilises the same adaptive voxelization in \cite{liu2021balm}, but further incorporates a new bottom-up merging process, which is able to combine multiple smaller features into large ones if possible. In this way, the method is able to further lower the amount of features in the adaptive voxelization without degrading the environment adaptability.

\vspace{-0.1cm}
\section{Bundle Adjustment Formulation and Optimization}\label{bundle_adjustment_formulation}

\begin{table}[t]
	\caption{Nomenclatures} \label{table:Notation}
	\centering
	{\begin{tabular}{rl}
		\toprule
		Notation & \qquad\qquad\qquad\qquad Explanation \\
		\midrule
		$\mathbb R^{m \times n}$ & The set of $m \times n$ real matrices. \\
		$\mathbb S^{m \times m}$ & The set of $m \times m$ symmetric matrices. \\
		$\boxplus$  & The encapsulated ``boxplus" operations on manifold.  \\
		$(\cdot)_f$  & The value of $(\cdot)$ expressed in lidar local frame, \\
		$(\cdot)$ & The value of $(\cdot)$ expressed in global frame. \\
		$\lfloor \cdot \rfloor$  & The skew symmetric matrix of $(\cdot)$. \\
		$\exp(\cdot)$  & Exponential of $(\cdot)$, which could be a matrix. \\
		$\mathbbm{1}_{i=j}$  & Indicator function which is equal to ``$1$" if $i=j$, \\
		& otherwise equal to ``$0$".  \\
		$M_f, M_p$ & The number of features and poses, respectively. \\
		$i,j,k$  & The indexes of features, poses and points, respectively. \\
		$l$ & The index of eigenvalue and eigenvector of a matrix. \\
		$\text{p,q}$  & The indexes of (block) row and column in a matrix. \\
		$\mathbf e_l$ & The vector in $\mathbb R^4$ with all elements being zeros except \\ 
		& the $l$-th element being one ($l\in \{1,2,3,4\}$). \\
		$\mathbf S_{\text{\bf P}}$ & $\mathbf S_{\text{\bf P}} = \begin{bmatrix}
		    \mathbf I_{3 \times 3} & \mathbf 0_{3 \times 1}
		\end{bmatrix} \in \mathbb{R}^{3 \times 4}$. \\
		$\mathbf S_{\text{\bf v}}$ & ${\mathbf S_{\text{\bf v}}} = \begin{bmatrix}
		    \mathbf 0_{1 \times 3} & 1
		\end{bmatrix} \in \mathbb{R}^{1 \times 4}$. \\
		$\mathbf E_{kl}$ & $\mathbf E_{kl} = \mathbf e_k \mathbf e_l^T + \mathbf e_l \mathbf e_k^T \in \mathbb{S}^{4 \times 4}, \ k,l \in \{1,2,3,4\}$. \\
		\bottomrule 
	\end{tabular}}
\end{table}

In this chapter, we derive our BA formulation and optimisation. First, following \cite{liu2021balm}, we formulate the BA as minimizing the the point-to-plane (or point-to-edge) distance (Sec. \ref{BAformulation}) and show that the feature parameters can be eliminated from the formulated optimisation (Sec. \ref{sec:feature_elimination}). Then, we introduce the point cluster in Sec. \ref{PointCluster}, based on which the first and second order derivatives are derived in Sec. \ref{derivative}.  Based on these theoretical results, we present our second-order solver in Sec. \ref{second-order-solver}. 
% Then the time complexity of the proposed BA is analyzed in the Sec. \ref{sec:time_analysis}. 
Finally, in Sec. \ref{covariance}, we show how to estimate the uncertainty of the BA solution. Throughout this paper, we use notations summarized in Table \ref{table:Notation} or otherwise specified in the context.

\subsection{BA formulation}\label{BAformulation}

\begin{figure} [t]
	\centering
	\includegraphics[width=0.7\linewidth]{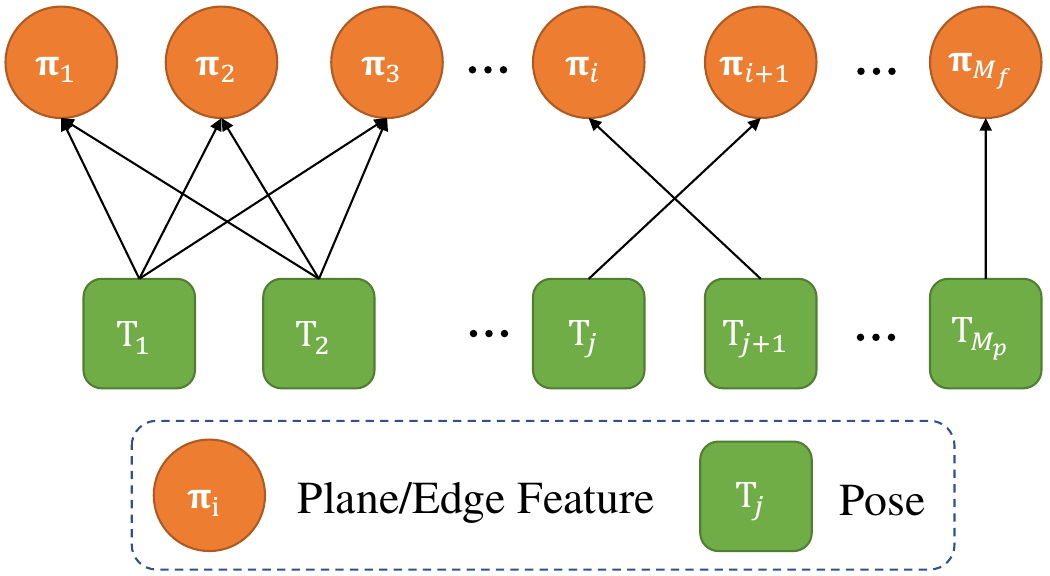}
	\caption{Factor graph representation of the bundle adjustment formulation.}
	\label{fig-BAfactor}
\end{figure}

\begin{figure} [t]
	\centering
	\includegraphics[width=1\linewidth]{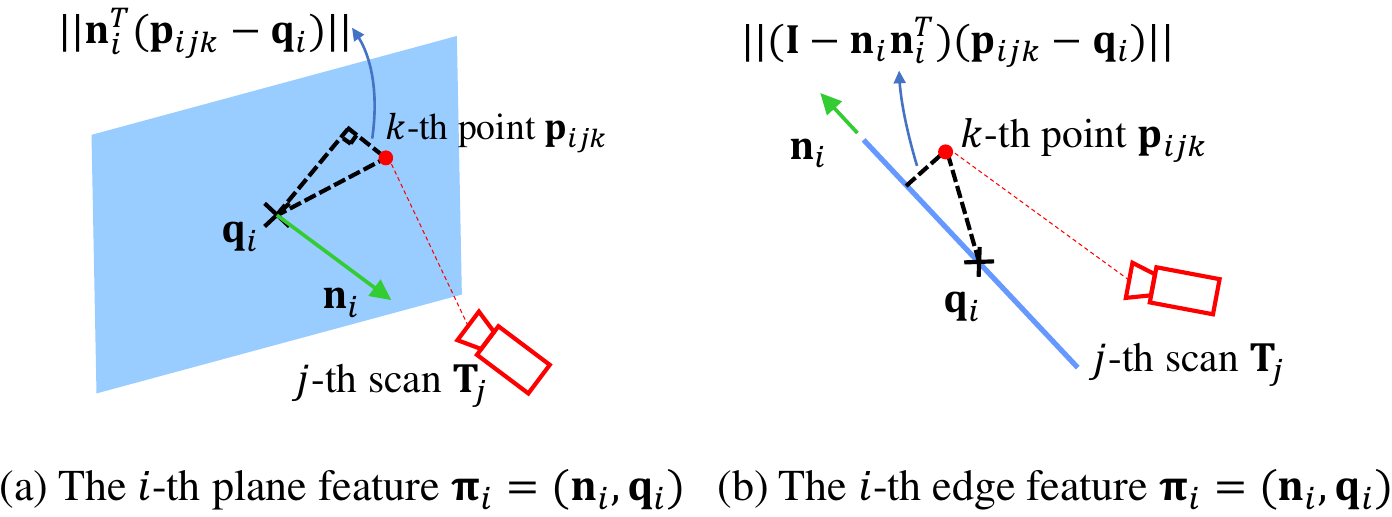}
	\caption{Plane and edge features used in the lidar BA. (a) The plane formulation. $\mathbf q_i$ is a point in the plane and $\mathbf n_i$ is the plane normal. (b) The line formulation. $\mathbf q_i$ is a point on the edge and $\mathbf n_i$ is the edge direction.}
	\label{fig-formulation}
\end{figure}

Shown in Fig. \ref{fig-BAfactor}, assume there are $M_f$ features, each denoted by parameter $\boldsymbol{\pi}_i$ ($i=1,...,M_f$), observed by $M_p$ lidar poses, each denoted by $\mathbf T_j = (\mathbf R_j, \mathbf t_j)$ ($j=1,...,M_p$), the bundle adjustment refers to simultaneously determining all the lidar poses (denoted by $\mathbf T = (\mathbf T_1, \cdots, \mathbf T_{M_p})$) and feature parameters (denoted by $\boldsymbol{\pi} = (\boldsymbol{\pi}_1, \cdots, \boldsymbol{\pi}_{M_f})$), such that reconstructed map agrees with the lidar measurements to the best extent. Denote $c(\boldsymbol{\pi}_i, \mathbf T)$ the map consistency due to the $i$-th feature, a straightforward BA formulation is 

\begin{align}
    \min_{\mathbf T, \boldsymbol{\pi}} 
	\Big(\sum\nolimits_{i=1}^{M_f} c (\boldsymbol{\pi}_i, \mathbf T)
	\Big). \label{BA-formulation}
\end{align}

In our BA formulation, we make use of plane and edge features that are often abundant in lidar point cloud and minimize the natural Euclidean distance between each measured raw lidar point and its corresponding plane or edge feature. Specifically, assume a total number of $N_{ij}$ lidar points are measured on the $i$-th feature at the $j$-th lidar pose, each denoted by $\mathbf p_{f_{ijk}}$ ($k=1,...,N_{ij}$). Its predicted location in the global frame is
\begin{align}\label{point_transform}
    \mathbf p_{ijk} = \mathbf R_j \mathbf p_{f_{ijk}} + \mathbf t_j.
\end{align}

For a plane feature, it is parameterized by $\boldsymbol{\pi}_i = (\mathbf n_i, \mathbf q_i)$ with $\mathbf n_i$ the plane normal vector and $\mathbf q_i$ an arbitrary point on the plane, both in the global frame (see Fig. \ref{fig-formulation} (a)). Then, the Euclidean distance between a measured point $ \mathbf p_{f_{ijk}} $ to the plane is $\| \mathbf n_i^T (\mathbf p_{ijk}- \mathbf q_i) \|_2$. Aggregating the distance for all points observed in all poses leads to the total map consistency corresponding to this plane feature: 
\begin{align}\label{plane_feature_cost}
  c (\boldsymbol{\pi}_i, \mathbf T) = \frac{1}{N_i} \sum_{j=1}^{M_p}\sum_{k=1}^{N_{ij}} 
\left \| \mathbf n_i^T (\mathbf p_{ijk}- \mathbf q_i) \right \|_2^2
\end{align}
where $N_i = \sum_{j = 1}^{M_p} N_{ij}$ is the total number of lidar points observed on the plane feature by all poses. 

For an edge feature, it is parameterized by $\boldsymbol{\pi}_i = (\mathbf n_i, \mathbf q_i)$ with $\mathbf n_i$ the edge direction vector and $\mathbf q_i$ an arbitrary point on the edge, both in the global frame (see Fig. \ref{fig-formulation} (b)). Then, the Euclidean distance between a measured point $ \mathbf p_{f_{ijk}} $ to the edge is $\left \| (\mathbf I - \mathbf n_i\mathbf n_i^T) (\mathbf p_{ijk}- \mathbf q_i) \right \|_2$. Aggregating the distance for all points observed in all poses leads to the total map consistency corresponding to this edge feature: 
\begin{align}\label{edge_feature_cost}
  c (\boldsymbol{\pi}_i, \mathbf T) = \frac{1}{N_i} \sum_{j=1}^{M_p}\sum_{k=1}^{N_{ij}} 
\left \| (\mathbf I - \mathbf n_i\mathbf n_i^T) (\mathbf p_{ijk}- \mathbf q_i) \right \|_2^2
\end{align}
where $N_i = \sum_{j = 1}^{M_p} N_{ij}$ is the total number of lidar points observed on the edge feature by all poses.

\subsection{Elimination of feature parameters}\label{sec:feature_elimination}

In this section, we show that in the BA optimization (\ref{BA-formulation}), the feature parameter $\boldsymbol{\pi}$ can really be solved with a closed-form solution. The key observation is that one cost item $c(\boldsymbol{\pi}_i, \mathbf T)$ depends solely on one feature parameter, so that the feature parameter can be optimized independently. Concretely,

\begin{align}
    \min_{\mathbf T, \boldsymbol{\pi}} 
	\Big(\sum_{i=1}^{M_f} c (\boldsymbol{\pi}_i, \mathbf T)
	\Big) & = \min_{\mathbf T} \Big( \min_{\boldsymbol{\pi}} 
	\Big(\sum_{i=1}^{M_f} c (\boldsymbol{\pi}_i, \mathbf T)
	\Big) \Big) \nonumber \\
	& = \min_{\mathbf T} \Big( 
	 \sum_{i=1}^{M_f} \min_{\boldsymbol{\pi}_i} c (\boldsymbol{\pi}_i, \mathbf T) \Big).
\end{align}

In case of a plane feature, we substitute (\ref{plane_feature_cost}) into $c (\boldsymbol{\pi}_i, \mathbf T)$:

\begin{align}
    \min_{\boldsymbol{\pi}_i} c (\boldsymbol{\pi}_i, \mathbf T) &= \min_{\boldsymbol{\pi}_i} \Big(  \frac{1}{N_i} \sum_{j=1}^{M_p}\sum_{k=1}^{N_{ij}} 
	\left \| \mathbf n_i^T (\mathbf p_{ijk}- \mathbf q_i) \right \|_2^2  \Big) \nonumber \\
	&= \lambda_3 (\mathbf A_i), \text{ when } \mathbf n_i^{\star} =\mathbf u_{3} (\mathbf A_i), \mathbf q_i^{\star} =\bar{\mathbf p}_i \label{eq:plane_margin}
\end{align}
where $\lambda_l(\mathbf A_i)$ denotes the $l$-th largest eigenvalue of matrix $\mathbf A_i$, $\mathbf u_{l} (\mathbf A_i)$ denotes the corresponding eigenvector, the matrix $\mathbf A_i$  and vector $\bar{\mathbf p}_i$ are defined as:

\begin{align}\label{A_q_def}
    \mathbf A_i & \triangleq \frac{1}{N_i} \sum_{j=1}^{M_p}\sum_{k=1}^{N_{ij}}
    (\mathbf p_{ijk} - \bar{\mathbf p}_i)(\mathbf p_{ijk} - \bar{\mathbf p}_i)^T, \notag
    \\
    \bar{\mathbf p}_i & \triangleq \frac{1}{N_i} \sum_{j=1}^{M_p} \sum_{k=1}^{N_{ij}}
    \mathbf p_{ijk}.
\end{align}

%RelateSupp
The proof will be given in Supplementary III-A \cite{LiuZheng2022supplementary}. Note that the optimal solution $\mathbf q_i^{\star}$ in (\ref{eq:plane_margin}) is not unique, any deviation from $\mathbf q_i^{\star}$ along a direction perpendicular to $\mathbf n_i^{\star}$ will equally serve the optimal solution. However, these equivalent optimal solution will not change the plane nor the optimal cost (hence the results that follow next). Indeed, the point $\mathbf q_i^{\star}$ could be an arbitrary point on the plane as it is defined to be. 

In case of an edge feature, we substitute (\ref{edge_feature_cost}) into $c (\boldsymbol{\pi}_i, \mathbf T)$:
\begin{align}
    & \min_{\boldsymbol{\pi}_i} c (\boldsymbol{\pi}_i, \mathbf T) \!\!=\! \min_{\boldsymbol{\pi}_i}\! \Big(  \frac{1}{N_i} \sum_{j=1}^{M_p}\sum_{k=1}^{N_{ij}} 
	\left \| (\mathbf I - \mathbf n_i \mathbf n_i^T) (\mathbf p_{ijk}- \mathbf q_i) \right \|_2^2 \Big) \nonumber \\
	&= \lambda_2(\mathbf A_i) + \lambda_3(\mathbf A_i); \text{ when } \mathbf n_i^{\star}=\mathbf u_{1}(\mathbf A_i), \mathbf q_i^{\star}=\bar{\mathbf p}_i. \label{eq:edge_margin}
\end{align}

Again, the optimal solution $\mathbf q_i^{\star}$ in (\ref{eq:edge_margin}) is not unique, any deviation from $\mathbf q_i^{\star}$ along the direction $\mathbf n_i^{\star}$ will equally serve the optimal solution. However, these equivalent optimal solution will not change the edge nor the optimal cost (hence the results that follow next). 

As can be seen from (\ref{eq:plane_margin}) and (\ref{eq:edge_margin}), the parameter $\boldsymbol{\pi}_i$ for each feature, either it is a plane or edge, can be analytically solved and hence removed from the BA optimization process. Consequently, the original BA optimization in (\ref{BA-formulation}) reduces to

\begin{align}
    \min_{\mathbf T} 
	\Big(\sum_{i=1}^{M_f} \lambda_l(\mathbf A_i)
	\Big) \label{BA-formulation-reduced}
\end{align}
where $l \in \{2,3\}$ and we omitted the exact number of eigenvalues in the cost function for brevity. 

Note that the matrix $\mathbf A_i$ in (\ref{BA-formulation-reduced}) depends on the lidar pose $\mathbf T$ since each involved point $\mathbf p_{ijk}$ depends on the pose (see (\ref{A_q_def}) and (\ref{point_transform})). Hence the decision variables of the resultant optimization in (\ref{BA-formulation-reduced}) involve the lidar pose $\mathbf T$ only, which dramatically reduces the optimization dimension (hence computation time). %This elimination of feature parameters is rooted in the three-dimensional measurements of the lidar point cloud, which ensures a closed-form representation of the contained plane and edge features.

\subsection{Point cluster} \label{PointCluster}

With the feature parameters eliminated, another difficulty remaining in the BA optimization  (\ref{BA-formulation-reduced}) is that the evaluation of matrix $\mathbf A_i$ (and its Jacobian or Hessian necessary for developing a numerical solver) requires to enumerate every point observed at each lidar pose. Such an enumeration is extremely computationally expensive due to the large number of points in a lidar scan. In this section, we show such point enumeration can be avoided by {\it point cluster}, which is detailed as follows. 

A {\it point cluster} is a finite point set denoted by set $\boldsymbol{\mathcal{C}} = \{ \mathbf p_{{k}} \in \mathbb{R}^3 | k = 1, \cdots, n\}$, the corresponding {\it point cluster coordinate}, denoted as $\boldsymbol{\Re}\!\left( \boldsymbol{\mathcal{C}} \right)$, is defined as:
\begin{align}
    \boldsymbol{\Re}\!\left( \boldsymbol{\mathcal{C}} \right) \triangleq 
    \sum_{k=1}^n
    \begin{bmatrix}
		\mathbf p_{k} \\ 1
	\end{bmatrix}
	\begin{bmatrix}
		\mathbf p_{k}^T & 1
	\end{bmatrix}
	&=
	\begin{bmatrix}
		\mathbf P & \mathbf v \\
		\mathbf v^T & n
	\end{bmatrix}  \in \mathbb{S}^{4 \times 4} \notag
	\\
	\mathbf P = \sum_{k=1}^n \mathbf p_{k} \mathbf p_{k}^T,
	\quad &
	\mathbf v = \sum_{k=1}^n \mathbf p_{k},
	\label{def-pc}
\end{align}
where $\mathbb{S}^{4 \times 4}$ denotes the set of $4 \times 4$ symmetric matrix. 

A \textit{point cluster} can be thought as a generalized point, for which a rigid transform could be applied. Similarly, we can define rigid transformation on a point cluster as follows.

% \begin{figure} [t]
% 	\centering
%  \includegraphics[width=8.5cm]{figures/transform.pdf}
% 	\caption{Rigid transformation of a \textit{point cluster}.}
% 	\label{fig-transform}
% 	\centering
%  \includegraphics[width=8.5cm]{figures/add.pdf}
% 	\caption{Addition of two \textit{point clusters}.}
% 	\label{fig-add}
% \end{figure}

\begin{figure}[t]

% \centering
\subfigure[Rigid transform]
{
    \begin{minipage}[t]{0.9\linewidth}
    % \centering
    \includegraphics[width=0.9\linewidth]{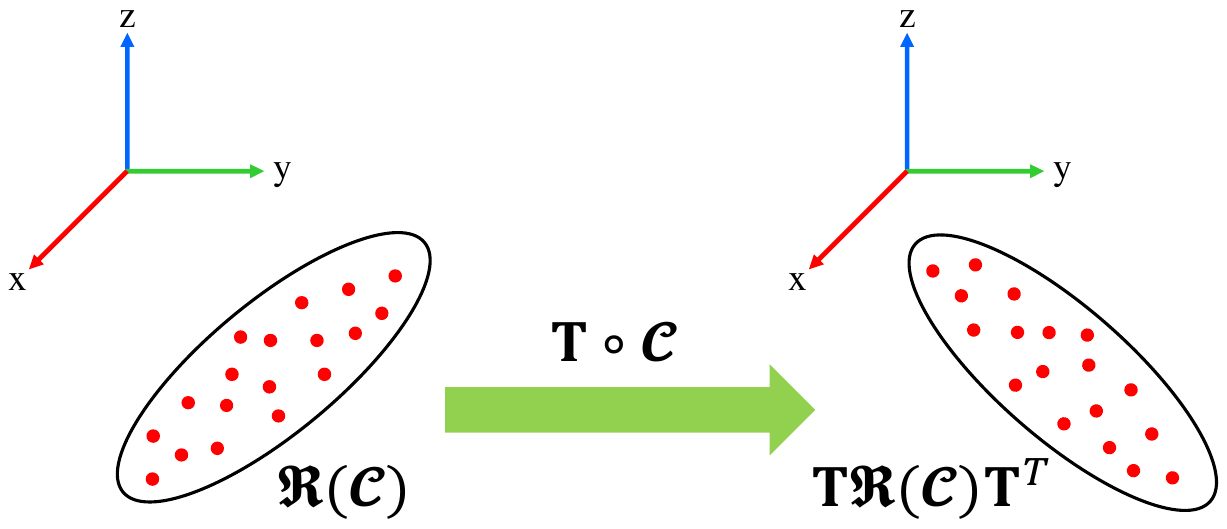}
    \end{minipage}
}

\subfigure[Cluster merging]
{
    \begin{minipage}[t]{0.9\linewidth}
    % \centering
    \includegraphics[width=1\linewidth]{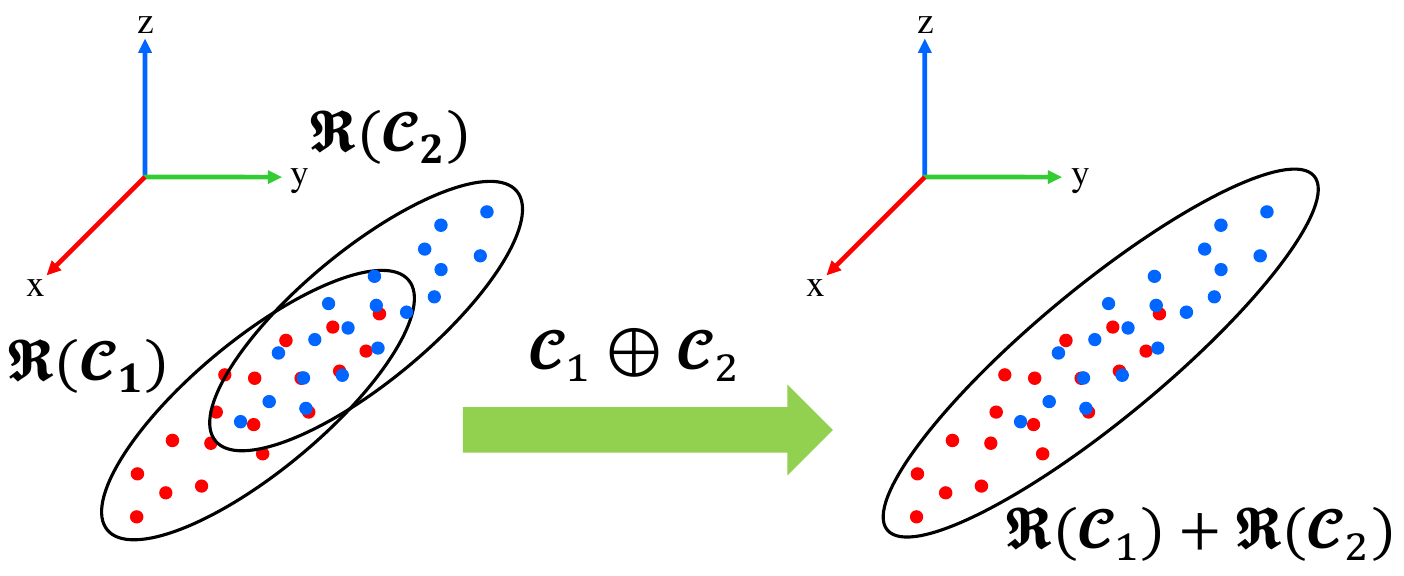}
    \label{fig-add}
    \end{minipage}%
}

\caption{Two operations on {\it point cluster} (a) Rigid transform (b) Cluster merging.}
\label{fig-operation}
\end{figure}

\begin{definition} \label{def_transform}
(\textbf{Rigid transform}) Given a point cluster with point collection $\boldsymbol{\mathcal{C}} = \{ \mathbf p_{{k}} \in \mathbb{R}^3 | k = 1, \cdots, n \}$ and a pose $\mathbf T = \begin{bmatrix}
    \mathbf R & \mathbf t \\
    0 & 1
\end{bmatrix}\in SE(3)$. The rigid transformation of the point cluster $\boldsymbol{\mathcal{C}}$, denoted by $\mathbf T \circ \boldsymbol{\mathcal{C}}$, is defined as
\begin{align}
   \mathbf T \circ \boldsymbol{\mathcal{C}} \triangleq \{ \mathbf R \mathbf p_{{k}} + \mathbf t \in \mathbb{R}^3 | k = 1, \cdots, n \}.
\end{align}
\end{definition}

Besides rigid transformation, we also define cluster merging operation, as follows.

\begin{definition} \label{def_merge}
(\textbf{Cluster merging}) Given two point clusters with point collections $\boldsymbol{\mathcal{C}}_1 = \{ \mathbf p_{k}^1 \in \mathbb{R}^3 | k = 1, \cdots, n_1 \}$ and $\boldsymbol{\mathcal{C}}_2 = \{ \mathbf p_{{k}}^2 \in \mathbb{R}^3 | k = 1, \cdots, n_2 \}$ {in the same reference frame}, respectively. The merged cluster, denoted by $\boldsymbol{\mathcal{C}}_1 \oplus \boldsymbol{\mathcal{C}}_2$, is defined as
\begin{align}
    \boldsymbol{\mathcal{C}}_1 \oplus \boldsymbol{\mathcal{C}}_2 \triangleq \{ \mathbf p_{{k}}^l \in \mathbb{R}^3 | l = 1, 2; k = 1, \cdots, n_i \}.
\end{align}
\end{definition}

Next we will show that the two operations defined above can be fully represented by their point cluster coordinates. 

\begin{theorem} \label{theorem_transform}
Given a point cluster $\boldsymbol{\mathcal{C}} $ and a pose $\mathbf T = \begin{bmatrix}
    \mathbf R & \mathbf t \\
    0 & 1
\end{bmatrix}\in SE(3)$. The rigid transformation of the point cluster satisfies 
\begin{align}
   \boldsymbol{\Re} (\mathbf T \circ \boldsymbol{\mathcal{C}}) =  \mathbf T \boldsymbol{\Re}(\boldsymbol{\mathcal{C}}) \mathbf T^T
\end{align}
\end{theorem}

%RelateSupp
\begin{proof}
See Supplementary III-B \cite{LiuZheng2022supplementary}.
\end{proof}

\begin{theorem} \label{theorem_merge}
Given two point clusters $\boldsymbol{\mathcal{C}}_1 $ and $\boldsymbol{\mathcal{C}}_2$ {in the same reference frame}. The merged cluster satisfies
\begin{align}
    \boldsymbol{\Re} (\boldsymbol{\mathcal{C}}_1 \oplus \boldsymbol{\mathcal{C}}_2) =  \boldsymbol{\Re}(\boldsymbol{\mathcal{C}}_1) + \boldsymbol{\Re}(\boldsymbol{\mathcal{C}}_2)
\end{align}
\end{theorem}

%RelateSupp
\begin{proof}
See Supplementary III-C \cite{LiuZheng2022supplementary}.
\end{proof}

As can be seen, rigid transformation and cluster merging operations on point clusters can be represented by usual matrix multiplication and addition on the point cluster coordinates. A visual illustration of the two operations and their coordinate representations are shown in Fig. \ref{fig-operation}. These results are crucially important: Theorem \ref{theorem_transform} indicates that the point cluster can be constructed in one frame (e.g., local lidar frame) and transformed to another (e.g., the global frame) without enumerating each individual points; Theorem \ref{theorem_merge} indicates that two (and by induction more) point clusters can be further merged to form a new point cluster. A particular case of Theorem \ref{theorem_merge} is when the second point cluster contains a single point, indicating that the point cluster can be constructed incrementally as lidar points arrives sequentially. 

\iffalse
Enabled by the above theoretical results, the \textit{point cluster} is implemented as the \textbf{Data Structure}, where the $4 \times 4$ matrix is the point cluster coordinate $\boldsymbol{\Re}(\boldsymbol{\mathcal{C}})$ and the two functions \texttt{Transform} and \texttt{Merging} correspond to the rigid transformation and cluster merging, respectively. 

\begin{algorithm} [h]
	\SetAlgoLined
	\NoCaptionOfAlgo
	\caption{\textbf{Data Structure}: Point cluster structure}
	\SetKwProg{pointCluster}{Struct}{:}{}
	\pointCluster{PointCluster}
	{
		\textit{Matrix44} \texttt{C}; \\
		\textbf{function} \texttt{Transform}(\textit{Matrix33} \texttt{R}, \textit{Vector3} \texttt{t}); \\
		\textbf{function} \texttt{Merging}(\textit{PointCluster} \texttt{pc});  
	}
\end{algorithm}
\fi

{
\begin{remark}
The concept of point cluster and its two operations are not new and have been used in previous works such as \cite{ferrer2019eigen, zhou2020efficient, huang2021bundle}. In this paper, we formalized this concept by 1) introducing the point cluster coordinate composing of $\mathbf P, \mathbf v$, and $n$ as in (\ref{def-pc}), 2) formalizing the two operations: rigid transform and cluster merging, and 3) explicitly showing the relation between point cluster operations and their coordinates. 
\end{remark}
}

\begin{remark} A point set and its coordinate is not a one-to-one mapping. While it is obvious that the coordinate is uniquely determined from the point set as shown in (\ref{def-pc}), the reverse way does not hold: the point set cannot be recovered from its coordinate uniquely. Since different point sets may lead to the same coordinate, the raw points must be saved if a re-clustering is needed. 
\end{remark}

\begin{figure} [t]
	\centering
	\includegraphics[width=1.0\linewidth]{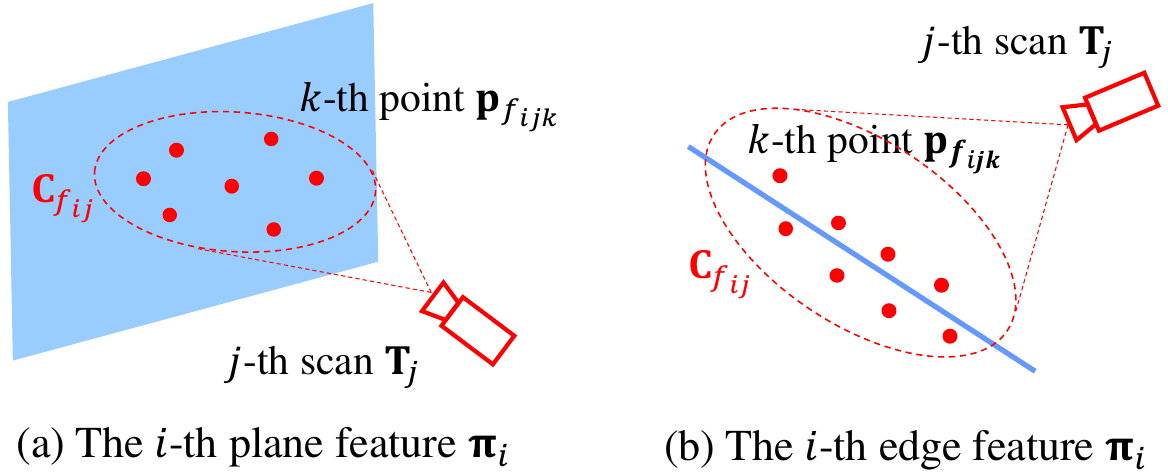}
	\caption{For the $i$-the feature (either plane or edge), all points observed at the $j$-th pose are clustered as a point cluster and is represented by ${\mathbf C}_{f_{ij}}$ in its local frame.}
	\label{fig-formulation-PC}
	\vspace{-0.2cm}
\end{figure}

Now, we apply the point cluster to the BA problem concerned in this paper. To start with, we group all points on the same feature as a point cluster. For example, the point cluster for the $i$-th feature is $\boldsymbol{\mathcal{C}}_i \triangleq \{ \mathbf p_{ijk} | j = 1, \cdots, M_p,  k = 1, \cdots, N_{ij} \}$. Denote $\mathbf C_i$ the coordinate of the {point cluster}, following (\ref{def-pc}), we obtain
\begin{align}
    \mathbf C_i & = \boldsymbol{\Re}( \boldsymbol{\mathcal{C}}_i ) \triangleq 
    \sum_{j=1}^{M_p}\sum_{k=1}^{N_{ij}}
    \begin{bmatrix}
		\mathbf p_{ijk} \\ 1
	\end{bmatrix}
	\begin{bmatrix}
		\mathbf p_{ijk}^T & 1
	\end{bmatrix}
	=
	\begin{bmatrix}
		\mathbf P_i & \mathbf v_i \\
		\mathbf v_i^T & N_i
	\end{bmatrix}  \notag
	\\
	&\quad \mathbf P_i = \sum_{j=1}^{M_p}\sum_{k=1}^{N_{ij}} \mathbf p_{ijk} \mathbf p_{ijk}^T,
	\quad 
	\mathbf v_i = \sum_{j=1}^{M_p}\sum_{k=1}^{N_{ij}} \mathbf p_{ijk}
\end{align}

A key result we show now is that this point cluster coordinate $\mathbf C_i$ is completely sufficient to represent the matrix $\mathbf A_i$ required in the BA optimization (\ref{BA-formulation-reduced}). According to (\ref{A_q_def}), we have
\begin{align}
	\mathbf A_i 
	=& \frac{1}{N_i} \sum_{j=1}^{M_p}\sum_{k=1}^{N_{ij}} 
	(\mathbf p_{ijk} - \bar{\mathbf p}_i)(\mathbf p_{ijk} - \bar{\mathbf p}_i)^T
	\notag \\
	\iffalse
	=& \frac{1}{N_i} \sum_{j=1}^{M_p}\sum_{k=1}^{N_{ij}} 
	(\mathbf p_{ijk} \mathbf p_{ijk}^T - \mathbf p_{ijk}  \bar{\mathbf p}_i^T - \bar{\mathbf p}_i\mathbf p_{ijk}^T + \bar{\mathbf p}_i\bar{\mathbf p}_i^T)
	\notag \\
	\fi
	=& \frac{1}{N_i} \sum_{j=1}^{M_p}\sum_{k=1}^{N_{ij}} 
	(\mathbf p_{ijk} \mathbf p_{ijk}^T) - \bar{\mathbf p}_i\bar{\mathbf p}_i^T
	\notag \\
	=& \frac{1}{N_i} \mathbf P_i - \frac{1}{N_i^2} \mathbf v_i \mathbf v_i^T \triangleq \mathbf A(\mathbf C_i) 
	\label{simpleA}
\end{align}
where we denote $\mathbf A_i$ as a function of $\mathbf C_i$ since the $\mathbf A_i$ is fully represented (and uniquely determined) by $\mathbf C_i$. We slightly abuse the notation here by denoting the function as $\mathbf A(\cdot)$. 

On the other hand, since the point set $\boldsymbol{\mathcal{C}}_i$ is defined on points in the global frame, the coordinate $\mathbf C_i$ is dependent on the lidar pose, which remains to be optimized. To explicitly parameterize the lidar pose, we note that 
\begin{align}
    \boldsymbol{\mathcal{C}}_i & \triangleq \{ \mathbf p_{ijk} | j = 1, \cdots, M_p,  k = 1, \cdots, N_{ij} \} \\
    &= \cup_{j=1}^{M_p} \{ \mathbf p_{ijk} | k = 1, \cdots, N_{ij} \} \\
    &\stackrel{\text{Def.} \ref{def_merge}}{=} \oplus_{j=1}^{M_p} \boldsymbol{\mathcal{C}}_{ij} , \quad \boldsymbol{\mathcal{C}}_{ij} \triangleq  \{ \mathbf p_{ijk} | k = 1, \cdots, N_{ij} \} \\
    &\stackrel{\text{Def.} \ref{def_transform}}{=} \oplus_{j=1}^{M_p} \left (\mathbf T_j \circ  \boldsymbol{\mathcal{C}}_{f_{ij}}  \right), \ \boldsymbol{\mathcal{C}}_{f_{ij}} \triangleq  \{ \mathbf p_{f_{ijk}} | k = 1, \cdots, N_{ij} \} \label{eq:coord_relation}
\end{align}
where $ \mathbf p_{f_{ijk}}$ is a point represented in the lidar local frame (see (\ref{point_transform})), $\boldsymbol{\mathcal{C}}_{{ij}}$ is the point set that is composed of all points on the $i$-th feature (either plane or edge) observed at the $j$-th lidar pose, and $\boldsymbol{\mathcal{C}}_{f_{ij}}$ is the same point set as $\boldsymbol{\mathcal{C}}_{{ij}}$, but represented in the lidar local frame (see Fig. \ref{fig-formulation-PC}).

The relation between the point cluster $\boldsymbol{\mathcal{C}}_{i} $ and $\boldsymbol{\mathcal{C}}_{f_{ij}}$ shown in (\ref{eq:coord_relation}) will lead to a relation between their coordinates $\mathbf C_i$ and $\mathbf C_{f_{ij}}$ as below:
\begin{align}
    \mathbf C_i \stackrel{\text{Thm.} \ref{theorem_merge}}{=} \sum_{j=1}^{M_p} \mathbf C_{ij} \stackrel{\text{Thm.} \ref{theorem_transform}}{=} \sum_{j=1}^{M_p} \mathbf T_j \mathbf C_{f_{ij}} \mathbf T_j^T
\end{align}

As a result, the BA optimization in (\ref{BA-formulation-reduced}) further reduces to

\begin{align}
    \min_{\mathbf T_j \in SE(3), \forall j} 
	\underbrace{\left(\sum_{i=1}^{M_f} \lambda_l \left(\mathbf A \left(\sum_{j=1}^{M_p} \mathbf T_j \mathbf C_{f_{ij}} \mathbf T_j^T \right) \right)
	\right)}_{c(\mathbf T)} \label{BA-formulation-reduced-reduced}
\end{align}
where the function $\mathbf A (\cdot)$ is defined in (\ref{simpleA}). Note that the cost function in (\ref{BA-formulation-reduced-reduced}) only requires the knowledge of point cluster coordinate $\mathbf C_{f_{ij}}$ without enumerating each individual points. The coordinate $\mathbf C_{f_{ij}}$ is computed as (following (\ref{def-pc})):
\begin{align}
    \mathbf C_{f_{ij}} &= 
	\begin{bmatrix}
		\mathbf P_{f_{ij}} & \mathbf v_{f_{ij}} \\
		\mathbf v_{f_{ij}}^T & N_{ij}
	\end{bmatrix}   \notag
	\\
	\mathbf P_{f_{ij}} &= \sum_{k=1}^{N_{ij}} \mathbf p_{f_{ijk}} \mathbf p_{f_{ijk}}^T, \quad
	\mathbf v_{f_{ij}} = \sum_{k=1}^{N_{ij}} \mathbf p_{f_{ijk}} \label{feature-pc}
\end{align}
which can be constructed during the feature association stage before the optimization. In particular, if the $j$-th pose observes no point on the $i$-th feature, $\mathbf C_{f_{ij}} = \mathbf 0_{4 \times 4}$, which will naturally remove the dependence on the $j$-th pose for the $i$-th cost item as shown in (\ref{BA-formulation-reduced-reduced}).

\begin{theorem} \label{theorem:invariance}
Given a matrix function $\mathbf A (\mathbf C ) \triangleq \frac{1}{N} \mathbf P - \frac{1}{N^2} \mathbf v \mathbf v^T$ with $ \mathbf C = \begin{bmatrix}
        \mathbf P & \mathbf v \\
        \mathbf v^T & N
    \end{bmatrix} \in \mathbb{S}^{4 \times 4}$, $\lambda_l(\mathbf A)$ denotes the $l$-th largest eigenvalue of $\mathbf A$, then $\lambda_l(\mathbf A(\mathbf C))$ is invariant to any rigid transformation $\mathbf T_0 \in SE(3)$. That is, 
    \begin{align}
        \lambda_l \left (\mathbf A \left( \mathbf T_0 \mathbf{C} \mathbf T_0^T \right) \right) = \lambda_l \left(\mathbf A \left(\mathbf{C} \right) \right), \forall \mathbf T_0 \in SE(3). 
    \end{align}
\end{theorem}

%RelateSupp
\begin{proof}
See Supplementary III-D \cite{LiuZheng2022supplementary}.
\end{proof}

Theorem \ref{theorem:invariance} implies that left multiplying all poses $\mathbf T_j, \forall j,$ by the same transform $\mathbf T_0$ does not change the optimization at all. That is, the BA optimization is invariant to the change of the global reference frame, which is the well-known gauge freedom in a bundle adjustment problem. 

\subsection{First and Second Order Derivatives} \label{derivative}

As shown in the previous section, the BA optimization problem in (\ref{BA-formulation-reduced-reduced}) is completely equivalent to the original formulation (\ref{BA-formulation}), where each cost item standards for the squared Euclidean distance from a point to a plane (or edge). This squared distance is essentially a quadratic optimization, which requires the knowledge of the second-order information of the cost function for efficient solving. In this section, we derive such first and second order derivatives. Without loss of generality, we only discuss the $i$-th feature, which contributes a cost item in the form of
\begin{align}\label{cost_item}
    c_i(\mathbf T) = \lambda_l \left(\mathbf A \left(\sum_{j=1}^{M_p} \mathbf T_j \mathbf C_{f_{ij}} \mathbf T_j^T \right) \right)
\end{align}
with $\mathbf C_{f_{ij}} \in \mathbb{R}^{4 \times 4}$ being a pre-computed matrix (see (\ref{feature-pc})). 

To derive the derivative of the cost item (\ref{cost_item}) w.r.t. the pose $\mathbf T_j$, which is an element of the Special Euclidean group $SE(3)$, we parameterize its perturbation by a special addition, called boxplus ($\boxplus$-operation). For the pose vector $\mathbf T = (\cdots, \mathbf T_j, \cdots)$, we define the the  $\boxplus$ operation as below
\begin{align}\label{eq:input_perturbation}
    \mathbf T \boxplus \delta \mathbf T &\triangleq (\cdots, \mathbf T_j \boxplus \delta \mathbf T_j, \cdots), \\
    \mathbf T_j \boxplus \delta \mathbf T_j & \triangleq (\exp{(\lfloor \delta \boldsymbol{\phi}_j \rfloor)} \mathbf R_j, \delta \mathbf t_j + \exp{(\lfloor \delta \boldsymbol{\phi}_j \rfloor)} \mathbf t_j ), \label{eq:input_perturbation_per}
\end{align}
where $\delta \mathbf T \triangleq (\cdots, \delta \mathbf T_j, \cdots) \in \mathbb{R}^{6M_p}$ with $\delta \mathbf T_j \triangleq (\delta \boldsymbol{\phi}_j, \delta \mathbf t_j) \in \mathbb{R}^6, \forall j \in {1, \cdots, M_p}, $ is the perturbation on the pose vector.

For a scalar function $f(\mathbf T): SE(3) \times \cdots \times SE(3) \mapsto \mathbb{R}$, denote $\left( \frac{\partial{f(\mathbf T)}}{\partial \mathbf T} \right)  (\mathbf T_0) $ its first-order derivative and $\left( \frac{\partial{f^2(\mathbf T)}}{\partial \mathbf T^2} \right) (\mathbf T_0)$ its second-order derivative, both at a chosen point of the input $\mathbf T_0$. The $\boxplus$ operation enables us to parameterize the input of function $f (\cdot) $,  $\mathbf T$, by its perturbation $\delta \mathbf T$ from a given point $\mathbf T_0$: $\mathbf T = \mathbf T_0 \boxplus \delta \mathbf T$. Since the map between $\mathbf T$ and $\delta \mathbf T$ is bijective if $\| \delta \boldsymbol{\phi}_j \| < \pi, \forall j$, the scalar function $f(\mathbf T)$ in terms of $\mathbf T$ can be equivalently written as a function $f(\mathbf T_0 \boxplus \delta \mathbf T)$ in terms of $\delta \mathbf T$. As a consequence, the derivatives of $f(\mathbf T)$ w.r.t. $\mathbf T$ at the point $\mathbf T_0$ can be defined as the derivatives of $f(\mathbf T_0 \boxplus \delta \mathbf T )$ w.r.t. $\delta \mathbf T$ at zero, where the latter is a normal derivative w.r.t. Euclidean vectors: 
\begin{align}
    &\left( \frac{\partial f(\mathbf T)}{\partial \mathbf T} \right) (\mathbf T_0) \triangleq \left(\frac{\partial f(\mathbf T_0 \boxplus \delta \mathbf T)}{\delta \mathbf T } \right) ( \mathbf 0) \label{eq:first_def} \\
    &\left( \frac{\partial^2 f(\mathbf T)}{\partial \mathbf T^2} \right) (\mathbf T_0) \triangleq \left( \frac{\partial }{\partial \delta \mathbf T } {\left( \frac{\partial f(\mathbf T_0 \boxplus \delta \mathbf T )}{\partial \delta \mathbf T} \right)}  \right) (\mathbf 0) \label{eq:second_def}  \\
    &\quad \quad \quad \quad \quad \quad \quad \quad \quad \quad \forall \mathbf T_0 \in SE(3) \times \cdots \times SE(3). \notag
\end{align}

In the following discussion, we use $\mathbf T$ as the reference point to replace $\mathbf T_0$ in the derivatives and omit it for the sake of notation simplification. 

Based on the derivatives defined in (\ref{eq:first_def}) and (\ref{eq:second_def}), we have the following results for the first and second-order derivatives of the cost item (\ref{cost_item}). 

\begin{theorem} \label{theorem:HJ_matrix}
Given
\begin{itemize}
    \item[(1)] Matrices $\mathbf C_j = \begin{bmatrix}
        \mathbf P_j & \mathbf v_j \\
        \mathbf v_j^T & N_j
    \end{bmatrix} \in \mathbb{S}^{4 \times 4}, j = 1, \cdots, M_p; $
    \item[(2)] Poses $\mathbf T_j \in SE(3), j = 1, \cdots, M_p;$
    \item[(3)] A matrix $ \mathbf C = \begin{bmatrix}
        \mathbf P & \mathbf v \\
        \mathbf v^T & N
    \end{bmatrix} \triangleq \sum_{j=1}^{M_p} \mathbf T_j \mathbf C_{{j}} \mathbf T_j^T  \in \mathbb{S}^{4 \times 4}$, {which is the aggregation of $\mathbf C_j$,} and a matrix function $\mathbf A (\mathbf C ) \triangleq \frac{1}{N} \mathbf P - \frac{1}{N^2} \mathbf v \mathbf v^T \in \mathbb{S}^{3 \times 3}$;  and 
    \item[(4)] A function $ \lambda_l \left(\mathbf A \left( \mathbf C \right) \right)$, $\lambda_l(\mathbf A)$ denotes the $l$-th ($l=1,2,3$) largest eigenvalue of $\mathbf A$ with corresponding eigenvector $\mathbf u_l$;
\end{itemize}
The Jacobian matrix {$\mathbf J_l$} and Hessian matrix {$\mathbf H_l$} of the function $\lambda_l(\mathbf A(\mathbf C))$ with respect to the poses $\mathbf T$ are
\begin{align}\label{eq: Jacobian_Hessian}
	&{\mathbf J_l} = \mathbf g_{ll} \in \mathbb{R}^{1 \times 6 M_p}, 
	\\ 
	&{\mathbf H_l} = {\mathbf W_l} + \sum_{k=1, k \neq l}^{3} \frac{2}{\lambda_l - \lambda_k} \mathbf g_{kl}^T \mathbf g_{kl} \in \mathbb{R}^{6 M_p \times 6 M_p}, \label{eq: Hessian}
\end{align}
{where $\mathbf g_{ll}$ is $\mathbf g_{kl}$ with $k\!=\!l$. $\mathbf g_{kl}$ and $\mathbf W_l$ are matrices partitioned as}
\begin{align}
    & \mathbf g_{kl} = 
	\begin{bmatrix}
		\cdots & {\mathbf g_{kl}^{j}} & \cdots
	\end{bmatrix} \in \mathbb R^{1 \times 6M_p},
    \\
    & {\mathbf W_l}
    = \begin{bmatrix}
        & \vdots & \\
        \cdots & {\mathbf W_{l}^{ij}} & \cdots \\
        & \vdots &
    \end{bmatrix}  \in \mathbb R^{6 M_p \times 6M_p},  
\end{align}
{with block elements $\mathbf g_{kl}^{j} \in \mathbb R^{1 \times 6}, \mathbf W_{l}^{ij} \in \mathbb R^{6 \times 6}, \forall i,j \in \{1, \cdots, M_p\}$ defined as} 
\begin{align}
	& {\mathbf g_{kl}^{j} 
	= \frac{1}{N} \mathbf u_l^T \mathbf S_{\text{\bf P}} (\mathbf T_j \!-\! \frac{1}{N} \mathbf C \mathbf F) \mathbf C_j \mathbf T_j^T \mathbf V_k^T} \notag
	\\
	& \qquad \quad {+ \frac{1}{N} \mathbf u_k^T \mathbf S_{\text{\bf P}} (\mathbf T_j \!-\! \frac{1}{N} \mathbf C \mathbf F) \mathbf C_j \mathbf T_j^T \mathbf V_l^T,} \label{eq:g_kl block} \\
    & {\mathbf W_{l}^{ij} \!= \! -\frac{2}{N^2} \mathbf V_l \mathbf T_i \mathbf C_i \mathbf F \mathbf C_j \mathbf T_j^T \mathbf V_l^T \!+\! \mathbbm{1}_{i = j}} \notag
	\\ 
	&\qquad \quad {\cdot \Big( \frac{2}{N} \mathbf V_l \mathbf T_j \mathbf C_j \mathbf T_j^T \mathbf V_l^T  + \begin{bmatrix}
		\mathbf K_{l}^j & \mathbf 0_{3\times 3} \\ 
		\mathbf 0_{3\times 3} & \mathbf 0_{3\times 3}
	\end{bmatrix}
	\Big),} \label{eq:H_ll block}
    \\
	& {\mathbf K_{l}^j
	= \frac{1}{N} \lfloor \mathbf S_{\text{\bf P}} \mathbf T_j \mathbf C_j (\mathbf T_j - \frac{1}{N}\mathbf C\mathbf F)^T \mathbf S_{\text{\bf P}}^T \mathbf u_l \rfloor \lfloor \mathbf u_l \rfloor} \notag
	\\
	&\qquad \quad {+ \frac{1}{N} \lfloor \mathbf u_l \rfloor
	\lfloor \mathbf S_{\text{\bf P}} \mathbf T_j \mathbf C_j (\mathbf T_j - \frac{1}{N}\mathbf C\mathbf F)^T \mathbf S_{\text{\bf P}}^T \mathbf u_l \rfloor }, 
	\\
	& {\mathbf V_l} = 
	\begin{bmatrix}
		-\lfloor \mathbf u_l \rfloor & \mathbf 0_{3\times 1} \\
		\mathbf 0_{3\times 3} & \mathbf u_l
	\end{bmatrix}
	\quad
	\mathbf F = 
	\begin{bmatrix}
    	\mathbf 0_{3 \times 3} & \mathbf 0_{3 \times 1} \\
    	\mathbf 0_{1 \times 3} & 1
    \end{bmatrix}, 
    \\
    & \mathbf S_{\text{\bf P}} = \begin{bmatrix}
		\mathbf I_{3 \times 3} & \mathbf {0}_{3 \times 1}
	\end{bmatrix} 
    \qquad
    {\mathbbm{1}_{i = j} = \left\{\begin{aligned}
    1, \quad i = j \\
    0, \quad i\neq j
    \end{aligned}\right.}.
\end{align}
\end{theorem}

\begin{proof}
See {Supplementary III-E} \cite{LiuZheng2022supplementary}.
\end{proof}

\begin{corollary} \label{theorem:zero_space}
{The Jacobian matrix $\mathbf J_l$ and Hessian matrix $\mathbf H_l$ in Theorem \ref{theorem:HJ_matrix} satisfy that, for any $l=1,2,3$,
\begin{align}\label{eq:J_H_null_space}
    & \mathbf J_l \cdot \delta \mathbf T = 0, \quad \delta \mathbf T^T \cdot \mathbf H_l \cdot \delta \mathbf T = 0, \notag \\
    &\quad \quad \quad \forall \delta \mathbf T \in \boldsymbol{\mathcal{W}} \triangleq \left\{ \left. \begin{bmatrix}
    \mathbf w \\ \vdots \\ \mathbf w
\end{bmatrix} \right| \forall \mathbf w \in \mathbb{R}^{6} \right\}, \\
& \mathbf J_l^j = \mathbf 0_{1 \times 6}, \text{ if } \mathbf C_j = \mathbf 0, \label{eq:J_H_sparsity-1} \\
& \mathbf H_l^{ij} = \mathbf 0_{6 \times 6}, \text{ if } \mathbf C_i = \mathbf 0 \text{ or } \mathbf C_j = \mathbf 0, \label{eq:J_H_sparsity-2}
\end{align}
where $\mathbf J_l^j$ is the $j$-th column block of $\mathbf J_l$ and $\mathbf H_l^{ij}$ is the $i$-th row, $j$-th column block of $\mathbf H_l$. }
\end{corollary}

%RelateSupp
\begin{proof}
{See {Supplementary III-F} \cite{LiuZheng2022supplementary}}.
\end{proof}

\begin{remark} {(\ref{eq:J_H_null_space}) implies that the Jacobian and Hessian matrices have null space containing the space spanned by $\boldsymbol{\mathcal{W}}$. This essentially means that } the cost function (\ref{cost_item}) in the BA optimization does not change along the direction where all the poses are perturbed by the same quantity $\mathbf w$, which agrees with gauge freedom stated in Theorem \ref{theorem:invariance}.
\end{remark}

\begin{remark}  The results in (\ref{eq:J_H_sparsity-1}, \ref{eq:J_H_sparsity-2}) imply that the blocks in Jacobian and Hessian matrices are zeros and hence their computation can be saved if any of the related poses does not observe the current feature {(i.e., $\mathbf C_i = \mathbf 0$ or $\mathbf C_j = \mathbf 0$)}. This sparse structure could save much computation time if a feature is observed only by a sparse set of poses. 
\end{remark}

\begin{remark} The derivatives in Theorem \ref{theorem:HJ_matrix} are obtained based on the pose perturbation defined in (\ref{eq:input_perturbation}), which multiplies the perturbation $\delta \mathbf T$ on the left of the current pose (i.e., a perturbation in the global frame). If other perturbation (denoted by $\delta \breve{\mathbf T}$) is preferred (e.g., a perturbation in the local frame to integrate with other measurements such as IMU pre-integration), where $\delta \mathbf T = \mathbf L \delta \breve{\mathbf T}$ with $\mathbf L$ the Jacobian between the two perturbation parameterization, its first and second order derivatives can be computed as $\breve{\mathbf J} = \mathbf J \cdot \mathbf L$ and $\breve{\mathbf H} = \mathbf L^T \cdot \mathbf H \cdot \mathbf L$, respectively. It can be seen that  $\breve{\mathbf J}$ and $\breve{\mathbf H}$ preserves a nullspace of $\mathbf L \cdot \boldsymbol{\mathcal{W}}$ with $\boldsymbol{\mathcal{W}}$ defined in (\ref{eq:J_H_null_space}). 
\end{remark}

\subsection{Second Order Solver} \label{second-order-solver}

The Jacobin and Hessian matrix from Theorem \ref{theorem:HJ_matrix} are computed for one cost item (\ref{cost_item}) that corresponds to one feature in the space. Denote $\mathbf J_i, \mathbf H_i$ the Jacobian and Hessian matrix for the $i$-th feature (or cost item), to determine the incremental update $\Delta \mathbf T$, we make use of the second order approximation of the total cost function $c(\mathbf T)$ in (\ref{BA-formulation-reduced-reduced}):
\begin{align}
	c(\mathbf T \boxplus \Delta \mathbf T) \approx c(\mathbf T) + {\mathbf J} \Delta \mathbf T + \frac{1}{2}\Delta \mathbf T^T {\mathbf H} \Delta \mathbf T \label{approxition}
\end{align}
where $\mathbf J = \sum_{i=1}^{M_f} \mathbf J_i, \mathbf H = \sum_{i=1}^{M_f} \mathbf H_i$. For any $\mathbf d \in \boldsymbol{\mathcal{W}}$, since $\mathbf J_i \mathbf d = 0, \mathbf d^T \mathbf H_i \mathbf d = 0, \forall i$, we have $\mathbf J \mathbf d = 0$ and $\mathbf d^T \mathbf H \mathbf d = 0, \forall i$, which means that any additional update along $\mathbf d \in \boldsymbol{\mathcal{W}}$ does not change the approximation at all. One way to resolve the gauge freedom is fixing the first pose at its initial value throughout the optimization. That is, setting $\Delta \mathbf T_1 = \mathbf 0$ in (\ref{approxition}). Then, setting the differentiation of the cost approximation in (\ref{approxition}) w.r.t. $\Delta \mathbf T$ (excluding $\Delta \mathbf T_1$) to zero leads to the optimal update $\Delta \mathbf T^{\star}$:
\begin{align}\label{linear_equ}
    \Delta \mathbf T^{\star} = - \left( \mathbf H + \mu \mathbf I \right)^{-1} \mathbf J^T,
\end{align}
where we used a Levenberg-Marquardt (LM) algorithm-like method to re-weight the gradient and Newton's direction by the damping parameter $\mu$. The complete algorithm is summarized in Supplementary (Algorithm 1) with time analyses detailed in Supplementary (Section IV). Overall, the solver has a complexity of $O\left(M_f M_p + M_f M_p^2 + M_p^3\right)$, which is linear to the number of feature $M_f$, irrelevant to the number of points $N$, and cubic to the number of pose $M_p$. The term $M_f M_p$ and $M_f M_p^2$ are due to the calculation of Jacobian and Hessian, respectively and the term $M_p^3$ is due to (\ref{linear_equ}).

\subsection{Covariance Estimation} \label{covariance}

Assume the solver converges to an optimal pose $\mathbf T^{\star}$, it is often  useful to estimate the confidence level of the estimated pose. Let $\mathbf T^{\text{gt}}$ be the ground-true pose, which is unknown, and $\delta \mathbf T^{\star}$ be the difference between the optimal estimate $\mathbf T^{\star}$ and the ground-true $\mathbf T^{\text{gt}}$, where $\mathbf T^{\text{gt}} = \mathbf T^{\star} \boxplus \delta \mathbf T^{\star}$.  The aim is to estimate the covariance of the error $ \delta \mathbf T^{\star}$, denoted by $\boldsymbol{\Sigma}_{ \delta \mathbf T^{\star}}$. 

Ultimately, the estimation error $\delta \mathbf T^{\star}$ is caused by the measurement noise in each raw point. Denote  $\mathbf p_{f_{ijk}}^{\text{gt}}$ the ground-true location of the $k$-th point observed on the $i$-th feature at the $j$-th lidar pose, with measurement noise $\delta \mathbf p_{f_{ijk}} \in \mathcal N(\mathbf 0, \boldsymbol \Sigma_{\mathbf p_{f_{ijk}}})$, the measured point location, denoted by $\mathbf p_{f_{ijk}}$, is
\begin{align} \label{eq:point-noise}
     \mathbf p_{f_{ijk}} = \mathbf p_{f_{ijk}}^{\text{gt}} + \delta \mathbf p_{f_{ijk}}.
\end{align}

Aggregating the ground-true points and the measured ones lead to the ground-true point cluster, denoted by $\mathbf C_{f_{ijk}}^{\text{gt}}$,  and the measured point cluster, denoted by $\mathbf C_{f_{ijk}}$, respectively (see (\ref{feature-pc})):
\begin{align}
    \mathbf C_{f_{ij}}^{\text{gt}} &= 
	\begin{bmatrix}
		\sum_{k=1}^{N_{ij}} \mathbf p_{f_{ijk}}^{\text{gt}} (\mathbf p_{f_{ijk}}^{\text{gt}})^T & \sum_{k=1}^{N_{ij}} \mathbf p_{f_{ijk}}^{\text{gt}} \\
		\left(\sum_{k=1}^{N_{ij}} \mathbf p_{f_{ijk}}^{\text{gt}} \right)^T & N_{ij}
	\end{bmatrix} \label{Cgt_def} \\
	&\approx \mathbf C_{f_{ij}} - \delta \mathbf C_{f_{ij}},
\end{align}
where
\begin{align}  \label{eq:delta_C_ij}
    \delta \mathbf C_{f_{ij}} = \sum_{k=1}^{N_{ij}} \mathbf B_{f_{ijk}} \delta \mathbf p_{f_{ijk}}, \quad \text{see Supplementary III-G}, %RelateSupp
\end{align}
which can be constructed in advance along with the point cluster $\mathbf C_{f_{ij}}$ during the feature associations stage. 

In the following discussion, to simplify the notation, we denote $\mathbf C^{\text{gt}}_f = \{ \mathbf C^{\text{gt}}_{f_{ij}} \}$, $\mathbf C_{f} = \{ \mathbf C_{f_{ij}} \}$, $\delta \mathbf C_{f} = \{ \delta \mathbf C_{f_{ij}} \}$ the ground-truth, measurements, and noises of all point clusters observed on any features at any lidar poses.  

Although the ground-true pose $\mathbf T^{\text{gt}}$ and point cluster $\mathbf C_{f}^{\text{gt}}$ are unknown, they are genuinely the optimal solution of (\ref{BA-formulation-reduced-reduced}) and hence the Jacobian evaluated there should be zero, i.e., 
\begin{align}\label{eq:Jacobian_eq_0}
    & \mathbf J^T \left(\mathbf T^{\text{gt}}, \mathbf C_{f}^{\text{gt}} \right) = \mathbf 0
\end{align}
where we wrote the Jacobian as an explicit function of the pose and point clusters. Now,  we approximate the left hand side of (\ref{eq:Jacobian_eq_0}) by its first order approximation:
\begin{align}
     &\mathbf J^T (\mathbf T^{\text{gt}}, \mathbf C_{f}^{\text{gt}}) = \mathbf J^T \left(\mathbf T^{\star} \boxplus \delta \mathbf T^{\star}, \mathbf C_{f} - \delta \mathbf C_{f} \right) = \mathbf J^T \left(\mathbf T^{\star} , \mathbf C_{f} \right) \notag \\
     &\ + \! {\frac{\partial \mathbf J^T  \left(\mathbf T^{\star} \boxplus \delta \mathbf T, \mathbf C_{f} \right)}{\partial \delta \mathbf T}} \delta \mathbf T^{\star}  - \frac{\partial \mathbf J^T  \left(\mathbf T^{\star}, \mathbf C_{f} \right)}{\partial \mathbf C_{f}}  \delta \mathbf C_{f}.
\end{align}

Noticing that $ {\frac{\mathbf J^T  \left(\mathbf T^{\star} \boxplus \delta \mathbf T, \mathbf C_{f} \right)}{\partial \delta \mathbf T}} = \mathbf H \left( \mathbf T^{\star}, \mathbf C_{f}\right)$, {the Hessian matrix of (\ref{BA-formulation-reduced-reduced}) evaluated at $\mathbf T^{\star}$ (also see (\ref{approxition}))}, we have
\begin{align}\label{eq:first_order_approx}
    \mathbf 0 &= \mathbf J^T \left(\mathbf T^{\text{gt}}, \mathbf C_{f}^{\text{gt}} \right) =  \mathbf J^T \left(\mathbf T^{\star} , \mathbf C_{f} \right) \notag \\
    & \quad +  \mathbf H \cdot \delta \mathbf T^{\star} -  \frac{\partial \mathbf J^T  \left(\mathbf T^{\star}, \mathbf C_{f} \right)}{\partial \mathbf C_{f}}   \delta \mathbf C_{f},
\end{align}
which implies
\begin{align}
    \delta \mathbf T^{\star} & =- \mathbf H^{-1} \mathbf J^T\left(\mathbf T^{\star} , \mathbf C_{f}\right) + \mathbf H^{-1}  \frac{\partial \mathbf J^T  \left(\mathbf T^{\star}, \mathbf C_{f} \right)}{\partial \mathbf C_{f}}   \delta \mathbf C_{f} 
\end{align}

Since $\mathbf T^{\star}$ is the converged solution using the measured cluster $\mathbf C_{f}$, they should lead to zero update, i.e., $ \mathbf H^{-1} \mathbf J^T \left(\mathbf T^{\star} , \mathbf C_{f}\right) = \mathbf 0$ {(see (\ref{linear_equ}) with zero $\mu$ at convergence)}. Therefore,
\begin{align}
    \delta \mathbf T^{\star} &= \mathbf H^{-1}  \frac{\partial \mathbf J^T  \left(\mathbf T^{\star}, \mathbf C_{f} \right)}{\partial \mathbf C_{f}}   \delta \mathbf C_{f} \sim \mathcal{N} \left(\mathbf 0, \boldsymbol{\Sigma}_{\delta \mathbf T^{\star}} \right),  \label{noise-TtoC}\\
    \boldsymbol{\Sigma}_{\delta \mathbf T^{\star}} &= \mathbf H^{-1}  \frac{\partial \mathbf J^T  \left(\mathbf T^{\star}, \mathbf C_{f} \right)}{\partial \mathbf C_{f}}  \boldsymbol{\Sigma}_{\delta \mathbf C_f}  \frac{\mathbf J  \left(\mathbf T^{\star}, \mathbf C_{f} \right)}{\partial \mathbf C_{f}}  \mathbf H^{-T}.  \label{pose_convariance}
\end{align}

%RelateSupp
We defer the exact derivation and results of $\frac{\mathbf J^T  \left(\mathbf T^{\star}, \mathbf C_{f} \right)}{\partial \mathbf C_{f}} \delta \mathbf C_f$, $\boldsymbol{\Sigma}_{\delta \mathbf C_f}$ and $\boldsymbol{\Sigma}_{\delta \mathbf T^{\star}}$ to Supplementary III-G \cite{LiuZheng2022supplementary}. Note that the evaluation of $\boldsymbol{\Sigma}_{\delta \mathbf T^{\star}}$ only requires the covariance $\delta \mathbf C_{f_{ij}}$, which has been constructed in advance according to (\ref{eq:delta_C_ij}), avoiding the enumeration of each raw point during the optimization.

\section{Implementations}\label{implementation}

We implemented our proposed method in C++ and tested it in Unbuntu 20.04 running on a desktop equipped with Intel i7-10750H CPU and 16Gb RAM. Since the reduced optimization problem (\ref{BA-formulation-reduced-reduced}) is not in a standard least square problem, which existing solvers (e.g., Google Ceres \cite{AgarwalCeresSolver2022}) applies to, we implemented the optimization algorithm with steps and parameters described in {Supplementary (Algorithm 1)}. When solving the linear equation on Line 9 at each iteration, we use the LDLT Cholesky decomposition decomposition method implemented in Eigen library 3.3.7. The termination conditions on Line 19 are iteration number below 50 (i.e., $j_{\text{max}}$= 50), rotation update below $10^{-6}$ rad, and translation update below $10^{-6}$ m.

\section{Consistency Evaluation} \label{simulation}

\begin{figure} [t]
	\centering
	\includegraphics[width=8.5cm]{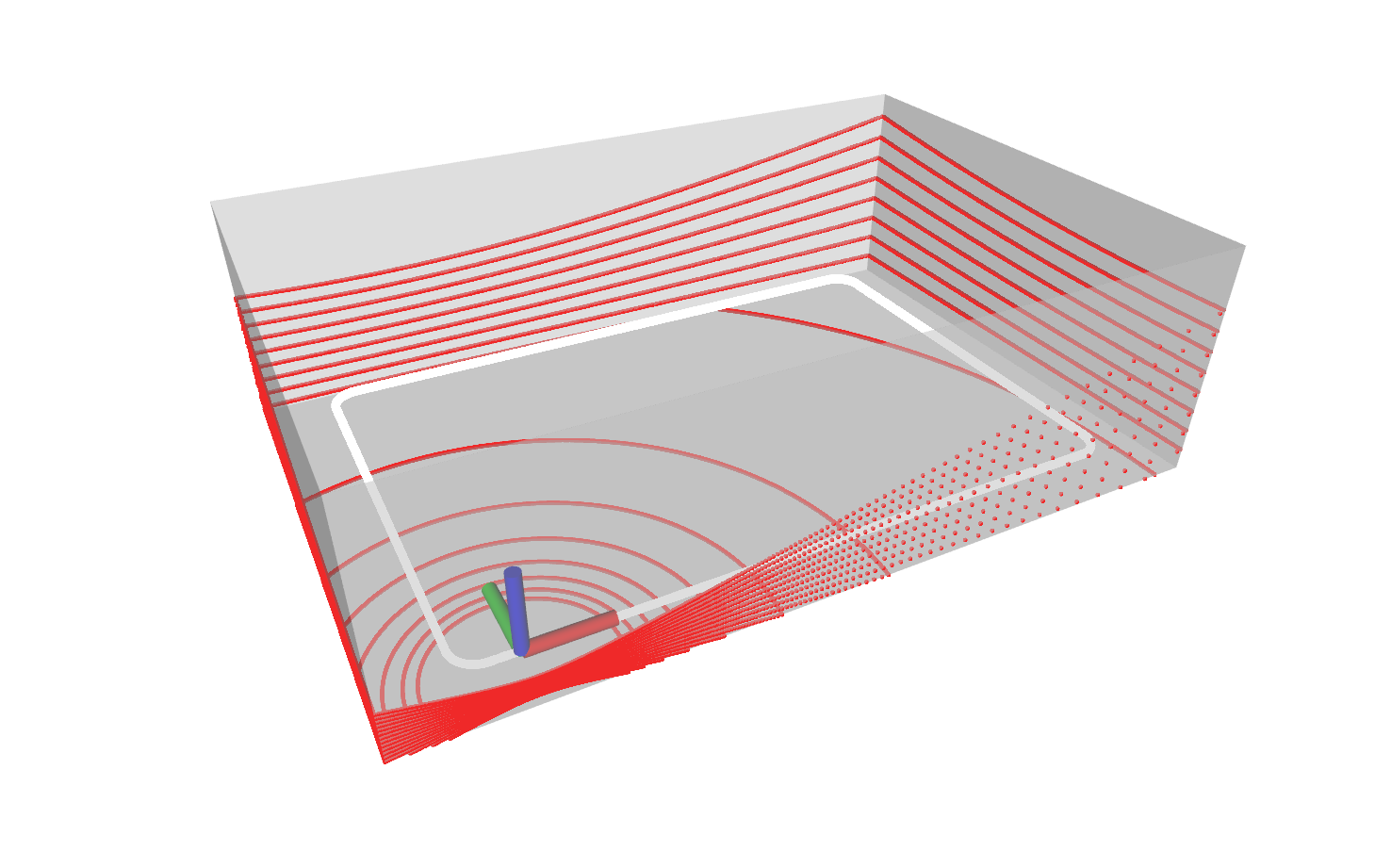}
	\caption{Simulation setup: A 16-channel lidar moves along a rectangular trajectory in a cuboid semi-closed space. The white line is the trajectory and the red lines are the laser points.}
	\label{fig simulation}
\end{figure}

This study aims to verify the consistency of the proposed BA method. That is, whether the estimated covariance $\boldsymbol{\Sigma}_{\delta \mathbf T^{\star}}$ from (\ref{pose_convariance}) agrees with the ground-true covariance of the pose estimation error $\delta \mathbf T^{\star}$.  As the ground-true covariance is unknown, we refer to a standard measure of consistency, the normalized estimation error squared (NEES) \cite{bar2004estimation, forster2016manifold}, which is defined below: 
\begin{align*}
	\eta = (\delta \mathbf T^{\star})^T \boldsymbol{\Sigma}_{\delta \mathbf T^{\star}}^{-1} \delta \mathbf T^{\star},
\end{align*}
where $\delta \mathbf T_j^{\star}$ is the estimation error of the pose defined according to (\ref{eq:input_perturbation_per}):
\begin{align*}
    \delta \mathbf T^{\star} & \triangleq (\cdots, \delta \mathbf T_j^{\star}, \cdots) \in \mathbb{R}^{6M_p},
    \\
	\delta \mathbf T_j^{\star} &= \begin{bmatrix}
	    \text{Log} \left(\mathbf R_j^{\rm{gt}} \left( \mathbf R_j^{\star} \right)^T \right), & 
	    \mathbf t_j^{\rm{gt}} - \mathbf R_j^{\rm{gt}}  (\mathbf R_j^{\star})^T \mathbf t^{\star}_j
	\end{bmatrix}^T,
\end{align*}
where the superscript ``gt" denotes the ground-true poses and (${\mathbf R}_j^{\star}$, ${\mathbf t}_j^{\star}$) denotes the estimated pose for the $j$-th scan.  Assume the pose estimate (${\mathbf R}_j^{\star}$, ${\mathbf t}_j^{\star}$) is unbiased (i.e., $E \left( \delta \mathbf T_j^{\star} \right) = \mathbf 0$), if the computed covariance $ \boldsymbol{\Sigma}_{\delta \mathbf T^{\star}}$ is the ground-truth, we can obtain the expectation 
\begin{align}
    & \! \! \! \! E(\eta) \! = \! E \! \left( \! (\delta \mathbf T^{\star})^T \boldsymbol{\Sigma}_{\delta \mathbf T^{\star}}^{-1} \delta \mathbf T^{\star} \right) \! = \! \text{trace} \! \left(  \! E \! \left( \boldsymbol{\Sigma}_{\delta \mathbf T^{\star}}^{-1} \delta \mathbf T^{\star} (\delta \mathbf T^{\star})^T \right) \! \right) \notag \\
    &\!= \! \text{trace} \! \left(  \! \boldsymbol{\Sigma}_{\delta \mathbf T^{\star}}^{-1} E \! \left( \delta \mathbf T^{\star} (\delta \mathbf T^{\star})^T \right) \! \right) \! = \! \text{trace} (\mathbf I) \! = \! \text{dim}(\delta \mathbf T^{\star}). 
\end{align}
That is, if the solver is consistent, the expectation of NEES should be equal to the dimension of the optimization variable. If the expectation of NEES is far higher than the dimension, the estimator is over-confident (i.e., the computed covariance is less than the ground-truth). Conversely, it is conservative. 

\begin{figure} [t]
	\centering
	\includegraphics[width=8.8cm]{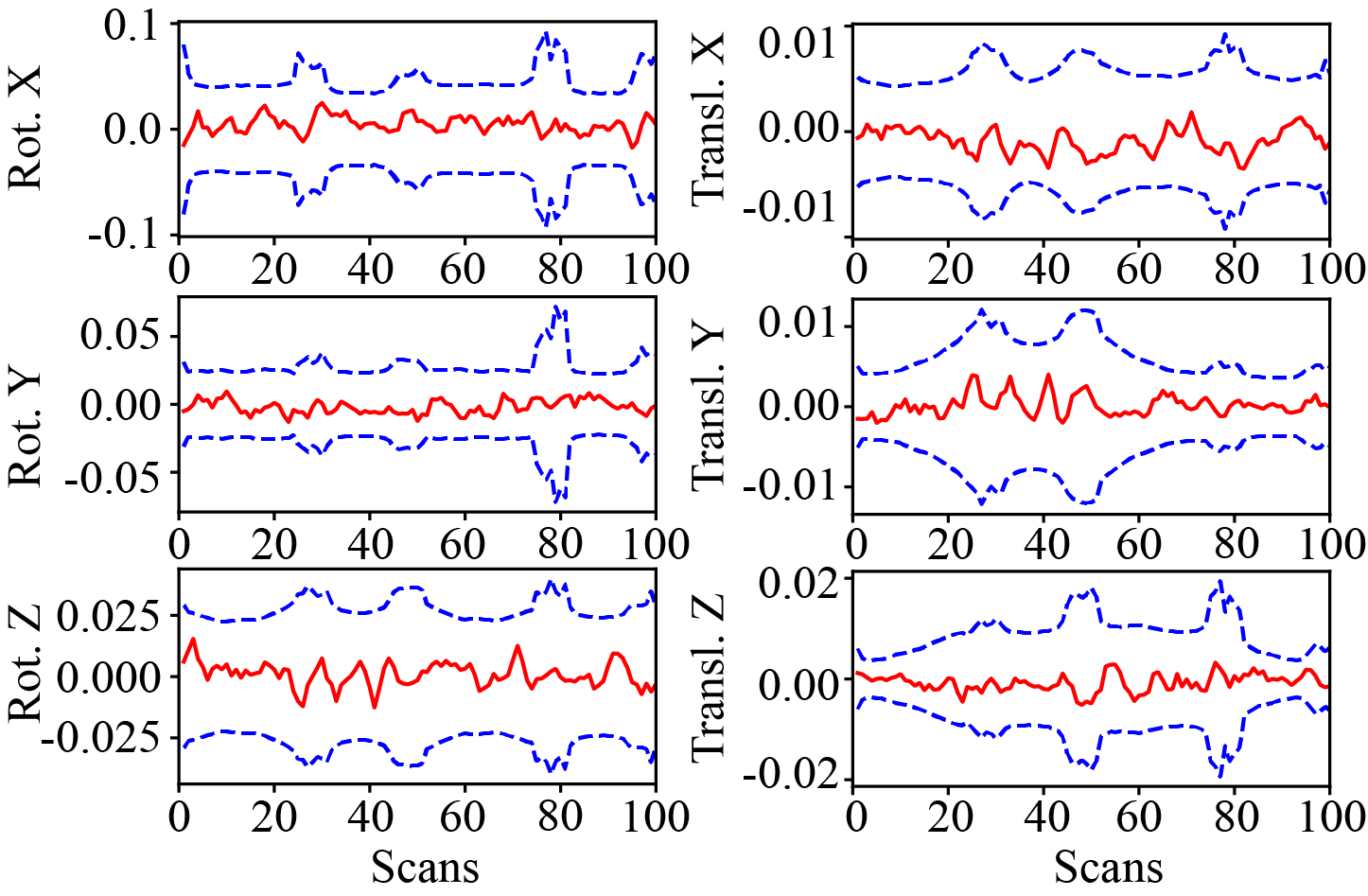}
	\caption{The error (red) of rotation (deg) and position (m) with $3\sigma$ bounds (blue) for one simulation run.}
	\label{3sigma}
\end{figure}

In practice, the expectation of NEES is evaluated by Monte Carlo method, where the NEES is computed for many runs and then averaged to produce the empirical expectation. 
\begin{align*}
	\bar{\eta} = \frac{1}{N} \sum_{i=1}^{N} \eta^{(i)}
\end{align*}
% \forall j \in \{ 1, \cdots, M_p\}
where $\eta^{(i)}$ is the NEES computed at the $i$-th Monte Carlo run. 

To conduct the Monte Carlo evaluation, we simulate a 16-channel lidar along a rectangular trajectory in a cuboid semi-closed space shown in Fig. \ref{fig simulation}. The size of the space is 30 m $\times$ 20 m $\times$ 8 m and the length of trajectory is about 92 m. 100 scans are equally sampled on the trajectory and the number of points in each scan is 28,800. {To simulate realistic measurements, each point is corrupted with an independent isotropic Gaussian noise with multiple standard deviations $\sigma_p \in \{0.05, 0.10, \cdots, 1.00\}$ m and for each value of the standard deviation {$\sigma_p$}, we performed 100 Monte Carlo experiments, leading to a total number of 2000 experiments.} In each run, we compute the optimal pose estimate from the {Supplementary (Algorithm 1)} with the same parameters specified in Section \ref{implementation} and the covariance matrix from (\ref{pose_convariance}). The initial trajectory required by {the algorithm} is obtained by perturbing the ground-true trajectory with a Gaussian noise with standard deviation {$\delta \boldsymbol \phi=2$ deg and $\delta t = 0.1$ m} on each pose. To avoid unnecessary errors, we use the ground-true plane association across different scans and ignore the in-frame motion distortion in the simulation. 

Fig. \ref{3sigma} shows orientation and position errors with the corresponding 3$\sigma$ bounds in one Monte Carlo experiment with $\sigma_p = 0.05$ m. As can be seen, the pose estimation errors are very small and they all remain within the $3\sigma$ bounds very well, which suggests that our new method is consistent. 

% The full Monte Carlo evaluation is presented in Fig. \ref{nees}, which shows the NEES averaged over all the 100 runs for rotation, translation, and the full pose. As can be seen, the average NEES for rotation and translation are very close their respective dimension 3 and the average NEES for the full pose is very close to 6, the pose dimension. These results confirm that our approach is consistent. 

Furthermore, we test the consistency of our BA method under different levels of point noise, where the standard deviation of a point noise ranges from $\sigma_p = 0.05 $ m to $1$ m. The results are shown in Fig. \ref{nees}(a) for the NEES averaged over 100 runs for each noise level and in Fig. \ref{nees}(b) for the average pose error. For better visualization, the average NEES is normalized by the pose dimension (i.e., 600 for 100 poses on the trajectory) in Fig. \ref{nees}(a). As can be seen, the normalized average NEES is very close to one, which suggests that our method is consistent, when the point noise is up to 0.3 m. Beyond this noise level, the first order approximation in (\ref{eq:first_order_approx}) no longer holds, which undermines the accuracy of the computed covariance. We should note that this noise level rarely occurs in actual lidar sensors, which are well below 0.1 m. Moreover, from Fig. \ref{nees}(b), we can see that our method produces accurate pose estimation even when the point noise are unrealistically large (up to 1 m, see {Fig.} \ref{nees}(c) for the point cloud map at this point noise level). 

\begin{figure} [t]
	\centering
	\includegraphics[width=8.8cm]{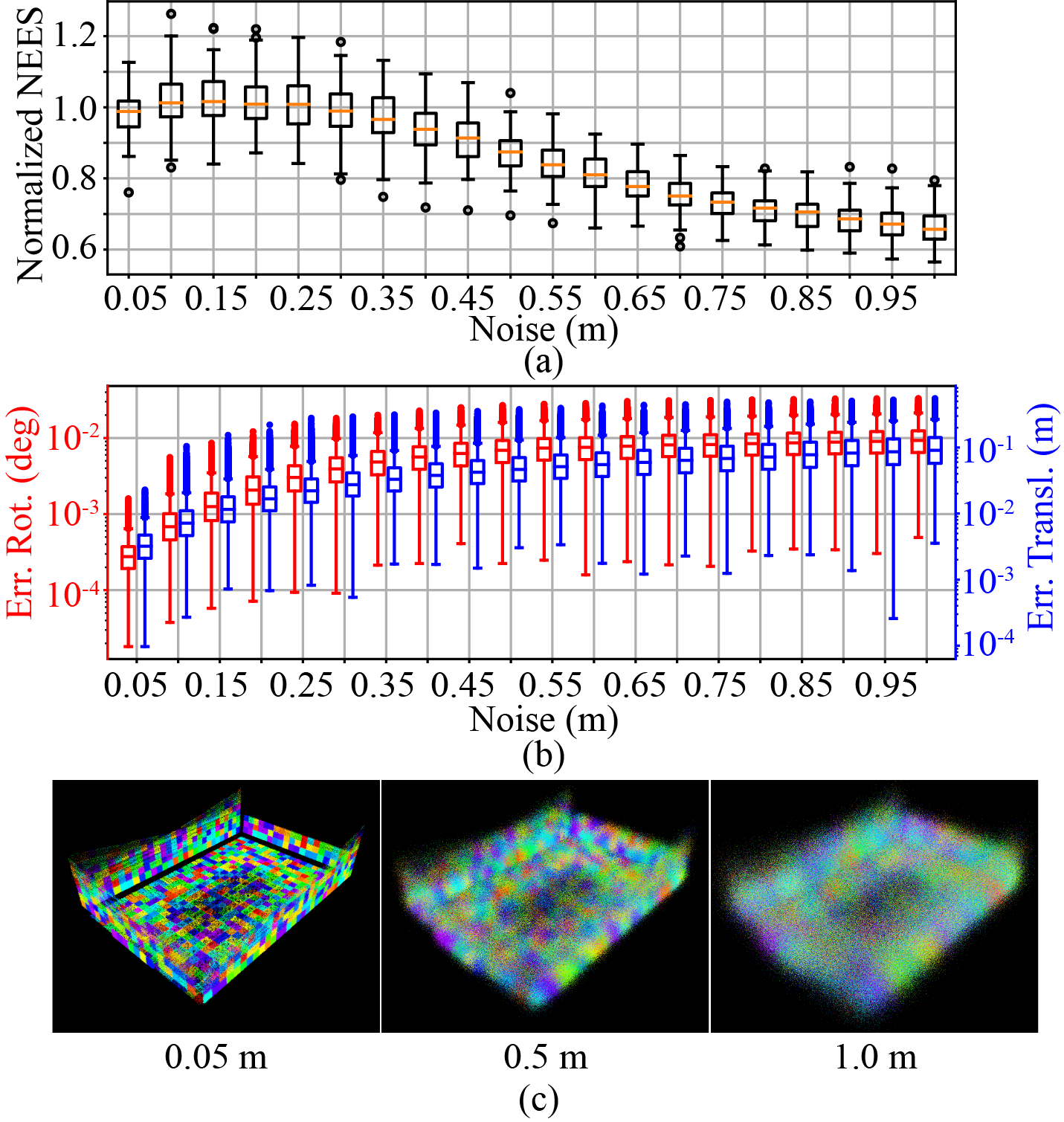}
	\caption{(a) The normalized NEES averaged over 100 Monte Carlo runs at different point noise levels. The NEES is normalized by the pose dimension (i.e., 600) for better visualization. (b) The rotation (red) and translation (blue) errors at different point noise levels. (c) The point cloud map with ground-true poses at noise levels $\sigma_p = $ 0.05m, 0.5m and 1m, respectively.
	}
	\label{nees}
\end{figure}

% We choose four snapshots in Fig. \ref{nees}(c), which shows the registered point cloud at convergence when the noise level is \hl{XXX}, respectively. As can be seen, the registered global point agree very well even though the noise level is unrealistically high. 

\section{Benchmark Evaluation} \label{benchmark}
In this section, we compare our method with other multi-view registration methods for lidar point clouds. The experiment will be divided into two parts: Section \ref{virtual_point_cloud} evaluates all methods with known feature association on {synthetic} point clouds, and Section \ref{real_dataset} evaluates the overall BA pipeline including both optimization solver and feature association on various real-world open datasets.

To verify the effectiveness of our method, we compare it with four state-of-the-art methods that focus on the lidar bundle adjustment (or similar) problem: Eigen-Factor (EF) \cite{ferrer2019eigen}, BALM \cite{liu2021balm}, Plane Adjustment (PA) \cite{zhou2021pi}, and BAREG \cite{huang2021bundle}. Among them, EF\footnote{\url{https://gitlab.com/gferrer/eigen-factors-iros2019}}, BALM\footnote{\url{https://github.com/hku-mars/BALM}}, BAREG\footnote{\url{https://hyhuang1995.github.io/bareg/}} are open sourced, so we use the available implementation on Github. PA is not available anywhere, so we re-implemented it in C++. To reduce the time cost of PA, we used the reduced Jacobian and residual technique in \cite{zhou2020efficient} (we derived it based on the cost function in Equation (10) of \cite{zhou2021pi}), which avoids the enumeration of each individual point. The re-implemented PA is solved by the Ceres solver with ``DENSE\_SCHUR" \cite{AgarwalCeresSolver2022}, which leverages the Schur complement trick to reduce the linear equation dimension at each optimization iteration. To better exploit the separable structure reducing the solving time, {we also compare with PA with inner iterations enabled in Ceres (denoted as ``PA (inner)")}.
% we also enabled the inner iterations in Ceres for PA.

For the solver parameters, our method and the re-implemented PA {(and its variant PA (inner))} use the parameters specified in {Section \ref{implementation} and \cite{zhou2020efficient}, respectively}, while EF, BALM and BAREG use their default parameters as available on their open source implementation. All methods use the same termination condition shown in Sec. \ref{implementation} (i.e. maximal iteration number below 200, rotation update below $10^{-6}$ rad, and translation update below $10^{-6}$ m), except for EF, which we found it converges too slowly and hence set the maximal iteration steps to 2000. In addition to the open source version of BALM (denoted by BALM), which samples only three points from each plane to lower the computation load, we also evaluated another vision (denoted by {BALM (full)}) which keeps all the points on a plane and use the same default parameters as its open sourced version. All solvers use the same initial pose trajectories detailed later.

\subsection{{Synthetic} point cloud}\label{virtual_point_cloud}

\begin{figure} [t]
    \centering
    \includegraphics[width=8.8cm]{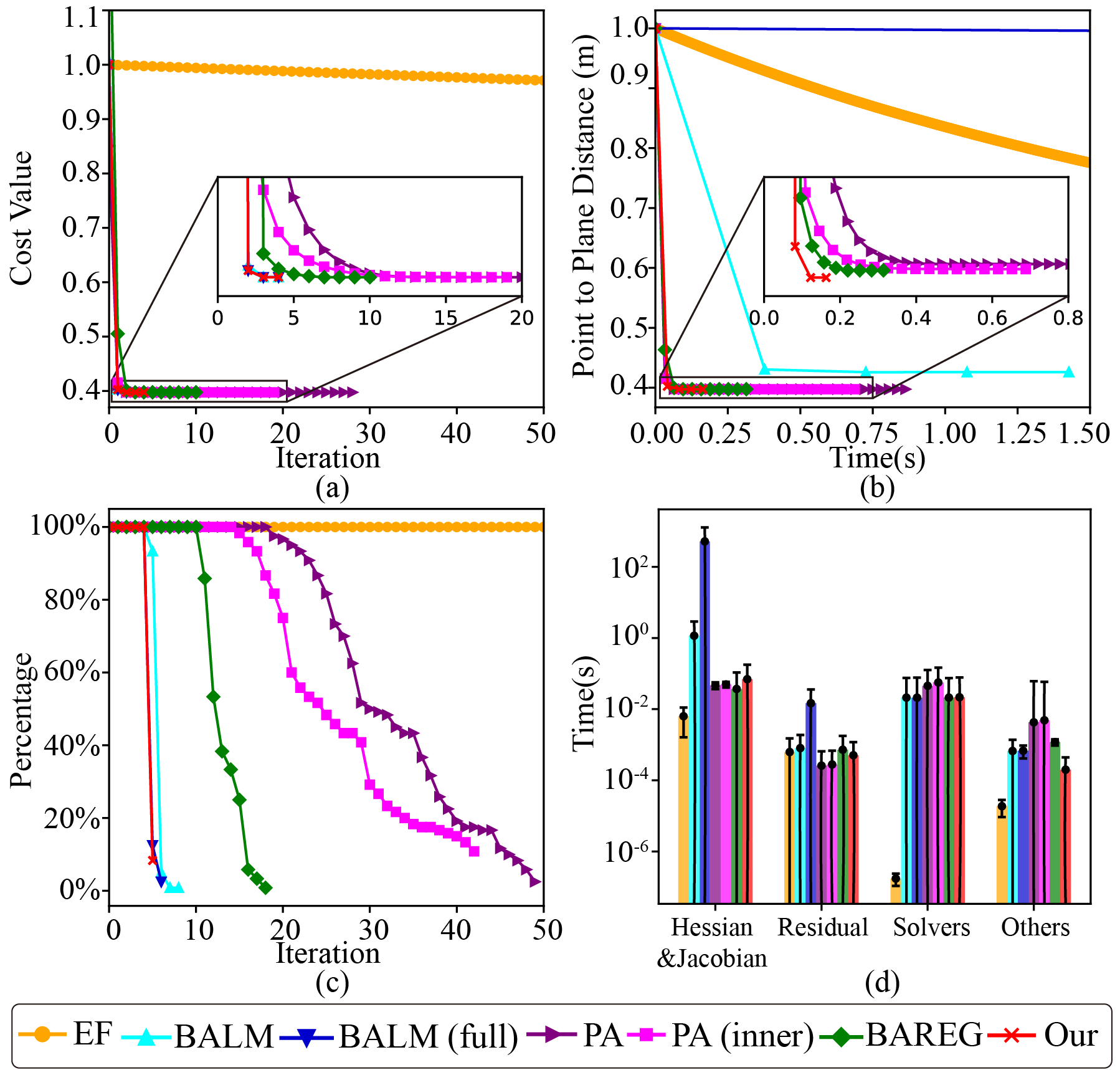}
    \caption{Convergence of different methods for BA optimization. (a) Cost value versus iterations in one repeat experiment with the nominal settings $M_f=100$, $M_p=100$, $N=100$ and initial pose error $10\times$. (b) Point-to-plane distance versus optimization time in one repeat experiment. (c) Iteration steps experienced by each method in all repeat experiments (i.e., 10) of all scenes (i.e., 21). The y-axis value represents how many experiments out of the 210 total experiments has experienced the iteration number indicated by the x-axis. (d) Breakdown of time spent on each iteration of all BA methods. The time is averaged among all experiments that all methods have participated.}
    \label{benchmark1_converge}
\end{figure}
%{except the ones with initial pose error $50\times$ and $100\times$}

To verify the effectiveness of the optimization solvers and their scalability to the number of pose $M_p$, number of feature $M_f$, and number of points $N$ per feature, we design a point-cloud generator which generates $M_f$ random planes and $M_p$ lidar scans at random poses. Each pose corresponds to one group of point-cloud whose number of points on each plane is $N$. Hence, there are totally $NM_f$ points at each scan. We use the ground-true plane association provided by the simulator. To mimic the real lidar point noises, we also corrupt the points sampled on each plane by an isotropic Gaussian noise with standard deviation $\sigma_p=0.05$m, the typical noise level for existing lidar sensors. The initial poses are perturbed from the ground-true poses with errors randomly sampled from a Gaussian distribution. The base standard deviation of the Gaussian distribution is $\| \delta \boldsymbol \phi \|=0.1$ deg for rotation and $\| \delta \mathbf t\| = 1$ cm for translation. In the nominal settings, $M_f =  M_p = N = 100$ and the initial pose error standard deviation is $10\times$ the base value. 
From the nominal settings, we enumerate each of the $M_f, M_p$ and $N$ at values $\{10, 30, 100, 300, 1000, 3000\}$ and the initial pose error standard deviation at values {$1\times$, $5\times$, $10\times$, $15\times$, $20\times$, $25\times$} of the base value to investigate the performance of each solver at different scales. 
This makes a total number of 21 scenes. In each scene, the experiment is repeated for 10 times with separately sampled poses, planes, and point noises, leading to a total 210 experiments.
% The initial pose trajectories for all the optimization methods are obtained by perturbing the ground-true poses by a Gaussian noise with standard deviation {$\delta \boldsymbol \phi=2$ deg and $\delta t = 0.1$ m}.

\subsubsection{Convergence}

\begin{figure} [t]
	\centering
	\includegraphics[width=1.0\linewidth]{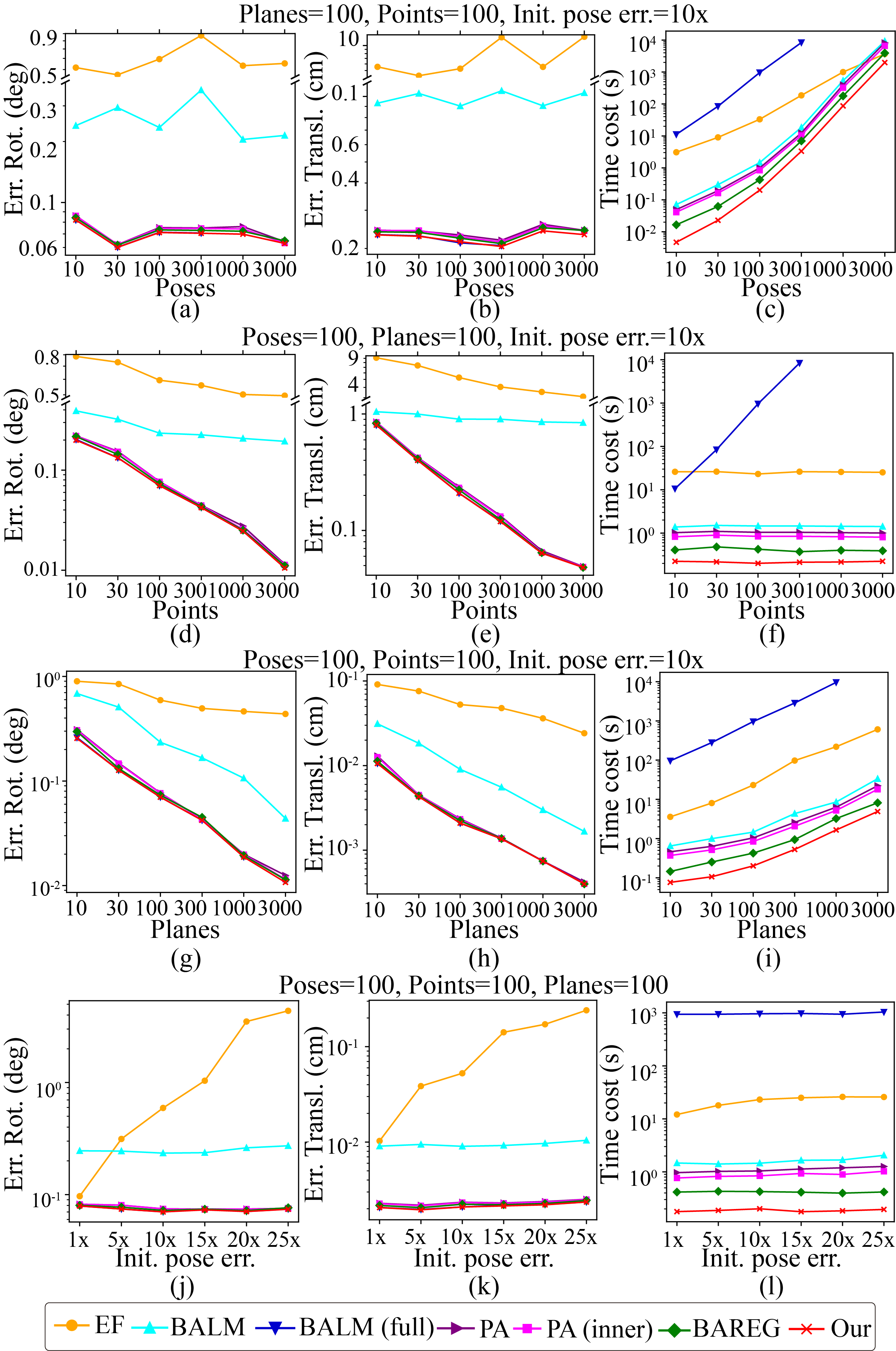}
	\caption{Benchmark results on synthetic point cloud.}
	\label{benchmark1}
\end{figure}

First we investigate the convergence performance of all methods. Fig. \ref{benchmark1_converge}(a) and (b) respectively shows the convergence of cost and point-to-plane distance in one repeat experiment with the nominal settings (i.e., $M_f=M_p=N=100$ and initial pose error $10\times$). Since different method uses different cost function, to compare them in one figure, the cost value of each method is normalized by its initial cost and then it is re-based such that the converged cost value of all methods are aligned at the same value. {Similarly, we normalize the point-to-plane distance by its initial value as well, which is valid to do because all methods have the same initial pose leading to the same initial point-to-plane distance}. As can be seen, EF converes rather slowly and requires the most number of iterations. This is because EF optimizes a cost function similar to ours in (\ref{BA-formulation-reduced}), which is essentially a quadratic function, but uses only the gradient information for optimization. Indeed, slow convergence of the gradient descent method on a quadratic cost function is a very typical phenomenon \cite{boyd2004convex}. PA, {PA (inner)}, and BAREG converge fast at the beginning but slowly when approaching the final convergence value. This is because PA optimizes both the plane parameters and scan poses, leading to a very large number of optimization variables that significantly slow down the speed at convergence. {PA (inner) converges faster than the original PA due to the inner iteration, but still slower than our method}. For BAREG, the empirical fixation of plane parameters also causes the optimization to slow down. In contrast, BALM, {BALM (full)} and our method eliminates the plane parameters exactly and the resultant optimization problem is only in dimension of the pose number. Further leveraging the exact Hessian information in their optimization update, BALM, {BALM (full)} and our method converge in a few iterations, which often represent the fastest convergence.

Fig. \ref{benchmark1_converge}({c}) shows the iteration experienced by each BA optimization method, where for each data point, the y-axis value represents how many out of the total experiments experienced the iteration number indicated by the x-axis. As can be seen, the overall trend agrees with the results in Fig. \ref{benchmark1_converge} (a) very well: our proposed method and {BALM (full)} require only four or five iterations, while BAREG, PA, and {PA (inner)} require up to {20, 40, and 50} iterations, respectively. EF requires even more iterations beyond 100. 

Fig. \ref{benchmark1_converge}(d) shows the computation time in each iteration. As can be seen, EF consumes the least time for each iteration due to the lack of Hessian computation and linear equation solving. {BALM (full)} consumes the most time since the computation of Jacobian, Hessian and residuals require to enumerate each individual point, leading to a complexity of $O(N^2)$. The other methods, PA,  {PA (inner)}, BAREG, and ours, consume similar time for each iteration.

\subsubsection{Accuracy}

% Fig. \ref{benchmark1}(a)-(f) shows the statistic values of the pose errors for scenes with different $M_p$. As can be seen, the errors of rotation and translation RMSE decrease monotonically with the number of features and and the number of points on each feature, which creates more constraints to the pose estimation. The accuracy of BALM does not improve with the point number since it always samples three points from each plane, thus cannot benefit from the increased point measurements. Due to the same reason, the overall accuracy of BALM is also lower than {BALM (full)}. EF has even lower accuracy than BALM due to the very slow convergence, the solver did not fully converge after the maximum iterations step (i.e., 500). On the other hand, {BALM (full)}, PA, BAREG, and our method converge well and achieve the highest accuracy with very little differences due to the use of all point measurements. 

Fig. \ref{benchmark1}({a,b,d,e,g,h,j,k}) shows the statistic values of the pose estimation accuracy in terms of RMSE. In each subplot, we fix three parameters of $M_f, M_p$, $N$ and initial pose error at the nominal values and change the fourth parameter to investigate its effect on the pose accuracy. Since the error of EF is much larger than the others, we used a broken y-axis to better display all the RMSE. As can be seen, overall the accuracy increases with the points per plane $N$ (in ({d}) and ({e})) or number of plane features $M_f$ (in ({g}) and ({h})) since both increases the number of pose constraints. In contrast, no such monotonic accuracy improvement is found for the number of poses (in (a) and ({b})) as the pose number increases because increasing the pose number itself does not gives more pose constraints. Likewise, the accuracy also remains similar for different initial poses error for all methods except EF, which did not converge at the maximum iteration number. Relatively speaking, our proposed method and {BALM (full)} achieves the same highest accuracy, since they essentially optimizes the same cost using the same exact Hessian information. The next best methods are BAREG, PA, and {PA (inner)}. While optimizing the same point to plane distance with our method (and {BALM (full)}), PA has significantly more optimization variables, which cause a much slower convergence where the solution is still slightly premature at the preset iteration number. {Although PA (inner) has used inner iterations to alleviate this problem, its iterations are still larger than our methods and BAREG.} The next accurate method is BALM, which samples only three instead of all points (as in our method, {BALM (full)}, BAREG, PA, and {PA (inner)}) and hence has higher RMSE. Finally, EF has the highest RMSE due to the very slow convergence, the solution is much premature even at the preset iteration number.

\subsubsection{Computation time}
Finally, we show the total computation time of different solvers at different feature number $M_f$, pose number $M_p$, point number $N$ and initial pose error. The results are shown in Fig. \ref{benchmark1}({c,f,i,l}). As can be seen, the time consumption of all methods increases with the number of poses $M_p$ (see ({c})) and plane features $M_f$ (see ({i})), which is reasonable since more poses or planes lead to a higher optimization dimension or more number of cost items, respectively. On the other hand, as the point number $N$ increases {(see ({f}))}, the method {BALM (full)} increases rapidly since its time complexity involves $O(N^2)$ while the rest methods (including ours) do not increase notably since they do not need to evaluate every raw point. {For the effect of initial pose errors in (see (l)), they do not affect the solving time significantly. }

Relatively speaking, our method achieves the lowest total computation time in all cases due to the small number of iteration numbers {(Fig. \ref{benchmark1_converge}(c))} and low time-complexity per iteration {(Fig. \ref{benchmark1_converge}(d))}. The next efficient method is BAREG, which has very low time-complexity per iteration due to the empirical feature parameter fixation but significantly more iteration numbers due to the same reason. Compared with our method, PA {(and PA (inner))} has similar time complexity per iteration as discussed in Supplementary (Section IV), but requires more iterations to converge. Hence their time costs are a little higher than ours and BAREG. BALM requires more iterations than our method and {more time in each iteration due to the enumeration of the sampled points}. Collectively, it leads to a computation time higher than our method and also BAREG and PA {(and PA (inner))}. The slow convergence problem is more severe in EF, leading to an even higher computation time. Finally, the most time-consuming method is {BALM (full)}, which, although has very small iteration numbers, consumes large time in each iteration. 

% {Finally, it is noted from Fig. \ref{benchmark1} (l) that, for large initial pose errors (i.e., larger than $50\times$), the computation time of our method (and also PA and BALM) quickly increased above the BAREG. This phenomenon is not surprising as our method (and PA and BALM) are second order method, when the initial pose errors are large, the second order approximation would exhibit large approximation errors that slow down the convergence. BAREG overcame this issue by introducing an additional cost item that align a local plane normal to the global plane normal. This additional cost could quickly align the initial pose hence accelerating the optimization especially at the first few iterations.}

\vspace{-0.2cm}
\subsection{Real-world datasets}\label{real_dataset}

\begin{table*}[ht]
\caption{Absolute trajectory error (RMSE,meters) for different methods.}
\centering
{
\begin{tabular}{clrrrrrrrrrrrr}
    \toprule
    Datasets & Sequence & ICP & GICP & NDT & EF & BALM & {PA} & PA & BAREG & Ours & Ours & Ours\\
    &&& &&& && {(inner)} && (float) & (edge)& \\
    \midrule
    \multirow{6}{*}{Hilti} 
    & Basement1     & 0.058 & 0.063 & 0.076 & 0.047 & 0.042 & {0.038} & 0.036 & 0.040 & \textit{0.0359} & {0.0361} & \textbf{0.0353} \\
    & Basement4     & 0.084 & 0.089 & 0.098 & 0.071 & 0.058 & {0.048} & 0.045 & 0.054 & \textit{0.0444} & {0.0448} & \textbf{0.0443} \\
    & Campus2       & 0.105 & 0.109 & 0.124 & 0.080 & 0.066 & {0.058} & 0.054 & 0.063 & 0.0535 & {\textbf{0.0530}} & \textit{0.0531} \\
    & Construction2 & 0.108 & 0.104 & 0.113 & 0.086 & 0.068 & {0.060} & 0.059 & 0.063 & \textit{0.0563} & {0.0577} & \textbf{0.0553} \\
    & LabSurvey2    & 0.066 & 0.069 & 0.072 & 0.046 & 0.025 & {0.019} & 0.019 & 0.023 & \textit{0.0185} & {0.0189} & \textbf{0.0181} \\
    & UzhArea2      & 0.182 & 0.191 & 0.211 & 0.161 & 0.141 & {0.122} & 0.121 & 0.127 & 0.1205 & {\textbf{0.1102}} & \textit{0.1171} \\
    \midrule
    \multirow{9}{*}{VIRAL}
    & eee01 & 0.159 & 0.163 & 0.172 & 0.102 & 0.073 & {0.052} & 0.040 & 0.061 & \textit{0.0390} & 0.0401 & \textbf{0.0382} \\
    & eee02 & 0.153 & 0.154 & 0.163 & 0.092 & 0.062 & {0.043} & 0.037 & 0.057 & \textit{0.0362} & 0.0378 & \textbf{0.0356} \\
    & eee03 & 0.171 & 0.175 & 0.180 & 0.113 & 0.081 & {0.056} &  0.053 & 0.068 & \textit{0.0522} & 0.0548 & \textbf{0.0517} \\
    & nya01 & 0.139 & 0.136 & 0.163	& 0.107 & 0.082 & {0.042} & 0.038 & 0.054 & \textit{0.0368} & 0.0372 & \textbf{0.0362} \\
    & nya02 & 0.160 & 0.159 & 0.124	& 0.097 & 0.067 & {0.050} & 0.048 & 0.061 & 0.0474 & \textit{0.0472} & \textbf{0.0468} \\
    & nya03 & 0.142 & 0.143 & 0.146	& 0.085 & 0.074 & {0.044} & 0.042 & 0.067 & \textit{0.0418} & 0.0425 & \textbf{0.0413} \\
    & sbs01 & 0.133 & 0.142 & 0.147	& 0.083 & 0.077 & {0.052} & 0.043 & 0.068 & \textit{0.0397} & 0.0404 & \textbf{0.0385} \\
    & sbs02 & 0.127 & 0.127 & 0.121	& 0.094 & 0.062 & {0.040} & 0.039 & 0.059 & \textit{0.0378} & 0.0393 & \textbf{0.0377} \\
    & sbs03 & 0.146 & 0.149 & 0.150	& 0.108 & 0.072 & {0.051} & 0.046 & 0.068 & 0.0440 & \textit{0.0432} & \textbf{0.0427} \\
    \midrule
    \multirow{4}{*}{UrbanLoco}
    & 0117   & 1.382 & 1.364 & 1.372 & 0.728 & 0.625 & {0.525} & 0.506 & 0.594 & {\textit{0.4964}} & {0.5324} & \textbf{0.4956} \\
    & 0317   & 1.384 & 1.299 & 1.289 & 0.878 & 0.732 & {0.661} & 0.657 & 0.682 & {0.6491} & {\textit{0.6449}} & \textbf{0.6488} \\
    & 0426-1 & 1.436 & 1.457 & 1.566 & 1.014 & 0.875 & {0.708} & \textit{0.689} & 0.733 & {0.6891} & {0.7135} & \textbf{0.6886} \\
    & 0426-2 & 1.676 & 1.693 & 1.543 & 1.113 & 0.924 & {0.864} & 0.837 & 0.905 & {\textit{0.8322}} & {0.8536} & \textbf{0.8223} \\
    \midrule
    Average && 0.411 & 0.410 & 0.412 & 0.268 & 0.221 & {0.186} & 0.179 & 0.203 & {\textit{0.1775}} & {0.1826} & \textbf{0.1763} \\
    \bottomrule
\end{tabular}
}
\label{benchmark2 ate}
\end{table*}

\begin{table*}[ht]
\caption{Occupied cells of point-cloud map for different methods.}
\centering
{
\begin{tabular}{clrrrrrrrrrrr}
\toprule
Datasets & Sequence & ICP & GICP & NDT & EF & BALM & {PA} & PA & BAREG  & Ours & Ours  & Ours \\
&&& &&& && {(inner)} && (float) & (edge)& \\
&  & (inc.)  & (inc.) & (inc.) & (inc.) & (inc.) & (inc.) & (inc.) & (inc.) & {(inc.)} & {(inc.)} & (base) \\
\midrule
\multirow{6}{*}{Hilti} 
& Basement1     & +20300 & +20954 & +21354 & +16692 & +6285 &  {+963} & +332 & +5864 & {+\textit{132}} & {+257} & \textbf{391962} \\
& Basement4     &  +7826 &  +7283 &  +8178 &  +6683 & +4762 & {+1028} & +223 & +3752 & {+\textit{112}} & {+197} & \textbf{558823} \\
& Campus2       & +14459 & +15511 & +21146 &  +8028 & +2863 &  {+977} & +248 & +2862 & { +68} & {-\textbf{97}} & \textit{1319482} \\
& Construction2 &  +6235 &  +9371 & +10032 &  +6397 & +1789 & {+1047} & +394 &  +986 & {\textit{ +95}} & {+181} & \textbf{979614} \\
& LabSurvey2    &  +1680 &  +3141 &  +6331 &  +5043 & +1375 &  {+410} & +210 & +1228 & {\textit{ +83}} & {+204} & \textbf{139682} \\
& UzhArea2      &  +9490 &  +9623 & +10832 &  +6371 & +2688 &  {+734} & +484 & +2785 & {\textit{+102}} & {+344} & \textbf{628951} \\
\midrule
\multirow{9}{*}{VIRAL} 
& eee01 & +43185 & +43439 & +44578 & +22731 & +2564 &  {+996} & +392 & +1321 & {\textit{ +85}} & {+289} & \textbf{1166482} \\
& eee02 & +10339 & +14573 & +15848 &  +6985 & +5938 & {+1538} & +177 & +5635 & {\textit{ +91}} & {+181} & \textbf{892168} \\
& eee03 &  +8584 &  +9419 &  +7418 &  +5720 & +2016 & {+1823} & +286 & +1060 & {\textit{+193}} & {+630} & \textbf{594921} \\
& nya01 & +53004 & +56370 & +48669 & +26087 & +7368 & {+1717} & +457 & +4246 & {\textit{ +46}} & {+585} & \textbf{571365} \\
& nya02 & +38056 & +37718 & +38435 & +24752 & +4710 & {+1980} & +692 & +3902 & {+238} & {\textit{+232}} & \textbf{572960} \\
& nya03 & +14282 & +13896 & +16325 & +10688 & +5922 & {+2178} & +308 & +2614 & {\textit{+172}} & {+446} & \textbf{562583} \\
& sbs01 & +10069 & +12196 & +16597 &  +9635 & +4224 & {+2056} &+1064 & +3691 & {\textit{+319}} & {+717} & \textbf{794228} \\
& sbs02 & +16573 & +16446 & +21046 & +10577 & +9278 & {+3451} & +488 & +5238 & {\textit{ +95}} & {+502} & \textbf{808235} \\
& sbs03 & +12257 & +11154 &  +8974 &  +4682 &  +877 &  {+1492} & +687 &  +763 & {+481} & {\textit{+332}} & \textbf{867174} \\
\midrule
\multirow{4}{*}{UrbanLoco}
& 0117   & +46718 & +47572 & +50969 & +16327 &  +7237 & {+3420} & +1016 & +5412 & {\textit{ +98}} & {+1987} & \textbf{1743775} \\
& 0317   & +37635 & +33676 & +41367 & +20072 & +13102 & {+4783} & +1453 & +8521 & {\textit{+103}} & {+1011} & \textbf{1709823} \\
& 0426-1 &  +9165 & +10242 & +13695 &  +9539 &  +2331 & {+1364} &  +525 & +1026 & {\textit{ +33}} & {+1413} & \textbf{1632662} \\
& 0426-2 & +31870 & +30461 & +29568 & +13827 &  +3428 & {+2021} &  +799 & +4451 & {\textit{+472}} & {+1252} & \textbf{2176302} \\
\midrule
Average & & +21617 & +21002 & +22703 & +12146 & +4671 & {+1788} & +539 & +3439 & {\textit{+159}} & {+561} & \textbf{953215} \\
\bottomrule
\end{tabular}
}
\label{benchmark2 occupied}
\end{table*}

In this experiment, we conduct benchmark comparison on three real-world datasets. The first dataset is ``\emph{Hilti}" \cite{helmberger2021hilti} which is a handheld SLAM dataset including indoor and outdoor environments. We use the lidar data collected by Ouster OS0-64 in the dataset. The ground-true lidar pose trajectory  is captured by a total station or motion capture system. The second dataset ``\emph{VIRAL}" \cite{nguyen2021ntu} is collected on an unmanned aerial vehicle (UAV) equipped with two 16-channel OS1 lidars. One lidar is horizontal and the other is vertical. We will use the horizontal one in this experiment. The ground-true positions are provided by a Leica Nova MS60 MultiStation tracking a crystal prism on the UAV. The last dataset ``\emph{UrbanLoco}" \cite{wen2020urbanloco} is collected by a car driving on urban streets. The lidar is a Velodyne HDL 32E and the ground truth is given by the Novatel SPAN-CPT, a navigation system incorporating Real Time Kinematic (RTK) and precisional IMU measurements.

Two preprocessing are performed for all sequences: motion compensation and scan downsample. To compensate the points distortion caused by continuous lidar movements within a scan, we run a tightly-coupled lidar-inertial odometry, FAST-LIO2 \cite{xu2022fast}, which estimates the IMU bias (and other state variables) and compensates the point motion distortion in real-time. We kept all points in a scan whose distortion has been compensated by FAST-LIO2 and discard the odometry output. The processed data are then downsampled from the original 10 Hz to 2 Hz for all sequences. This is because the BA methods need to process all scans at once, a 10Hz scan rate causes prohibitively high computation load for all BA methods. The downsampling is also similar to the keyframe selection in common SLAM frameworks. %It is also the imitation of keyframe concept in real-world BA application. 

We compare our method with EF, BALM, PA, {PA (inner)} and BAREG. Noticing that the computation time of {BALM (full)} is prohibitively high due to the extremely large number of lidar points, we hence remove it from the benchmark comparison. For the rest methods, their solver parameters are kept the same for all sequences with values detailed in previous sections. 
% For feature association, the parameters are $l_{\text{max}} = 3, n_{\text{min}} = 20, \gamma = \frac{1}{25}$. For the parameter $L$, it is set to $1$ m for  ``\emph{Hilti}" and ``\emph{VIRAL}" and $2$ m for ``\emph{UrbanLoco}".

{For feature association, we use the adaptive voxelization proposed in BALM \cite{liu2021balm}, which registers all points in the world frame (using an initial trajectory) and recursively cuts the space into smaller sub-voxels until the sub-voxel contains only one feature (either plane or edge) that associates points from different scans. EF did not address the feature association problem and PA did not open relevant codes, so we use this method for them too. BAREG used a similar adaptive voxelization method but has its own implementation, so we retain its own implementation. All feature associations have the same set of parameters: the root voxel size $L=1$ m for ``\emph{Hilti}" and $L = 2$ m for ``\emph{VIRAL}" and ``\emph{UrbanLoco}", the maximum voxelization layer $l_{\text{max}} = 3$, the minimum number of points $n_{\text{min}}=20$ for a feature test, and the feature test thresholds $\gamma =\frac{1}{25}$.}

{The above feature association method is able to extract and associate both plane and edge features. Since the other BA methods, including EF, PA {(and PA (inner))}, and BAREG, are only designed for plane features, we use only plane features for them. For our  method, it is applicable to both plane and edge features, so we test two variants: the one with only plane features, denoted as ``Ours", for comparison with other BA methods, and the one with both plane and edge features,  denoted as ``Ours (edge)". Moreover,  besides the default implementation of our method with double-precision numbers, we test the stability of our method with single-precision floating number implementation, denoted as  ``Ours(float)". Note that all other BA methods were implemented with double-precision. }

In addition to the multi-view registration methods, we also compare with classic pairwise registration methods, including ICP, GICP, and NDT offered in PCL library.  We run the pairwise registration methods in an incremental manner, where each new scan is registered and merged to previous scans incrementally. To constrain the computation time, in each new scan registration, only the last 20 scans are used.  The pose estimation from the ICP is then used as the initial trajectory for feature association and optimization of the BA methods, including EF, BALM, PA, {PA (inner)}, BAREG, and ours.

\begin{figure} [t]
	\centering
	\includegraphics[width=8.5cm]{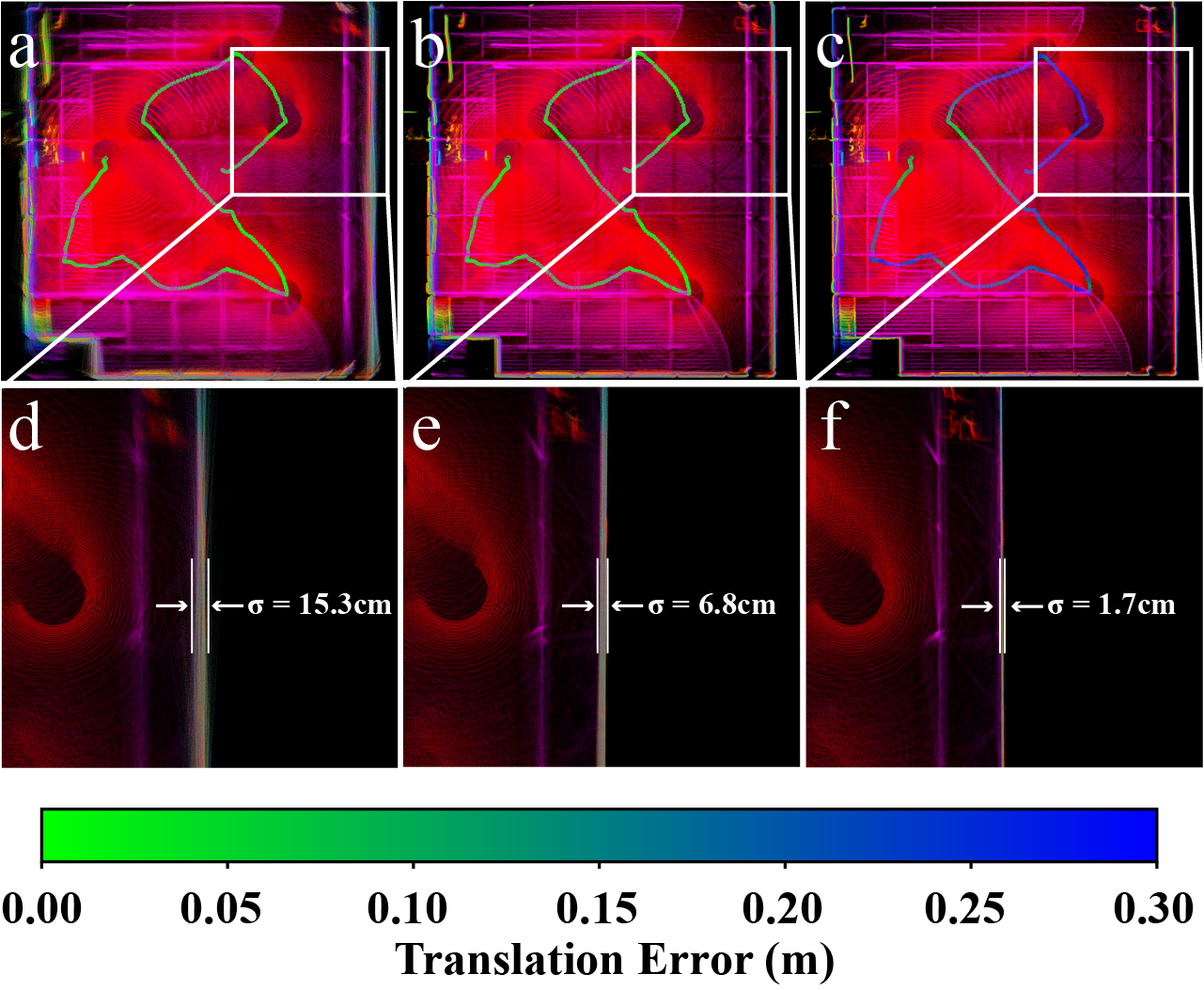}
	\caption{Point cloud map of the UzhArea2 sequence in ``\emph{Hilti}". (a) registered by ground-true pose trajectory. (b) registered by ground-true position with rotation optimized by our BA method. (c) registered by poses fully optimized by our BA method. (d), (e) and (f) points on one side wall in (a), (b) and (c), respectively.  }
	\label{uzh}
\end{figure}

\subsubsection{Accuracy} 

% This trend is in great agreement with the converged point-to-plane distances in Fig. \ref{fig:benchmark2_converge} (the lower point-to-plane distance, the better accuracy), while is a little inconsistent with the results on synthetic point cloud for BAREG.

Table \ref{benchmark2 ate} shows the ATE results. As can be seen, our method consistently achieves the best results in all 19 sequences {even with single-precision}. The next accurate method is {PA (inner)}, PA and BAREG, followed by BALM and EF. This trend is in great agreement with the results on synthetic point cloud shown in Section \ref{virtual_point_cloud}-2. In particular, our method achieves an accuracy within a few centimeters in all sequences of ``\emph{Hilti}" and ``\emph{VIRAL}", with only one exception (i.e., UzhArea2), which will be analyzed later. The centimeter level accuracy achieved by our method is at the same level of lidar point noises. Moreover, using only lidar measurements, our method achieved an average accuracy of $4.2$ cm on all VIRAL dataset sequences, which outperforms the accuracy $4.7$ cm reported in VIRAL-SLAM \cite{nguyen2021ntu} that fuses all data from stereo camera, IMU, lidar, and UWB. The accuracy on ``\emph{UrbanLoco}" is lower (analyzed later) than other datasets, but still outperforms the other BA methods. Finally, we can notice that the BA methods (i.e., EF, BALM, PA, PA (inner), BAREG, and ours) generally outperforms the pairwise registration methods (i.e., ICP, GICP and NDT) due to the full consideration of multi-view constraints. 

\vspace{-0.2cm}
When comparing among different variants of our method, the single-precision implementation has a lower accuracy than double-precision as expected, but it offers significant time savings as discussed later. The incorporation of edge features leads to no noticeable accuracy improvement. The accuracy difference with and without edge features are as small as 6 mm. This is because in real-world point clouds, edge features extracted based on local smoothness (e.g., \cite{zhang2014loam}) are very noisy because the laser pulse emitted by lidars can barely hit an edge exactly due to the limited angular resolution. The situation is further exacerbated when the edge is located at far or when the lidar has increased laser beam divergence, which creates many bleeding points behind an edge and degrades the edge points extraction more\cite{yuan2021pixel}. On the other hand, in real-world environments, edge features are often created by depth discontinuity at the edge of a foreground object, which meanwhile makes a good plane feature, so adding the edge feature does contribute many new effective constraints.

% Our proposed optimization method and the adaptive voxelization method in BALM \cite{liu2021balm} are applicable to both edge and plane features. However, we notice that in real-world point cloud data, edge features extracted based on local smoothness (e.g., \cite{zhang2014loam}) are very noisy because the laser pulse emitted by lidars can barely hit an edge exactly due to the limited angular resolution. The situation is even worsened when the edge is at far or when the lidar has increased laser beam divergence, which creates many bleeding points behind an edge and degrades the edge points extraction \cite{yuan2021pixel}. As a consequence, we found that in practice, adding the edge features does not contribute much to the accuracy. Moreover, since the other methods, including EF, PA, and BAREG are only designed for plane features,  for a fair comparison, we use only plane features in the following experiments for all methods. In fact, an edge is often created by an foreground object, which also makes a good plane feature, so dropping the edge feature does not reduce the number of constraints significantly. 

{It is noted that BAREG has an accuracy obviously lower than other methods (e.g., PA, PA (inner) and our method), which disagrees with results obtained previously from the synthetic data. The reason is that BAREG first extracts eigenvectors $\mathbf u_1$ and $\mathbf u_2$ ($\lambda_1 > \lambda_2 > \lambda_3$) of points corresponding to a plane feature in each local lidar scan. The two eigenvectors were assumed to be normal to the true plane normal and hence used to construct a cost item $\lambda_1 \left \| \mathbf R \mathbf u_{1} \cdot \mathbf n \right \|^2 + \lambda_2 \left \| \mathbf R \mathbf u_{2} \cdot \mathbf n \right \|^2$ in addition to the point to plane residual. The additional cost item could bias the optimization results if the extracted eigenvectors $\mathbf u_1$ and $\mathbf u_2$ are not accurate (i.e., they are not really perpendicular to the true plane normal), a presumption for the optimality of BAREG. Unfortunately, such optimality presumption did not hold well in real-world datasets, where the points density varies considerably: points on planes further from the sensor exhibit sparser distributions compared to those closer. This sparsity in distant planes leads to significant errors in the calculation of $\mathbf u_1$ and $\mathbf u_2$. Moreover, in real-world datasets, due to the imperfections of plane extraction, the extracted planes utilized for BA optimization may not be perfect planes (e.g., slightly curved walls or ground), and the point noise cannot be guaranteed isotropic Gaussian noise. All these factors contribute to errors in the extracted $\mathbf u_1$ and $\mathbf u_2$ and bias the optimization results.}

Now we investigate the performance degradation on ``\emph{UrbanLoco}" and the sequence UzhArea2 in ``\emph{Hilti}" more closely. For the ``\emph{UrbanLoco}" dataset, we found that the RTK ground-truth had some false sudden jumps, which contributes the large ATEs. This sudden jump may be caused by tall buildings in the crowded urban area which lowers the quality of the ground-truth. For the sequence UzhArea2, we register the point cloud with the ground-true pose trajectory and compare it with the point cloud registered with our BA method in Fig. \ref{uzh}. As can be seen, with the ground-true pose, points on the side wall are very blurry and points on the wall form a plane with standard deviation up to 15.3 cm (Fig. \ref{uzh}(d)); with the ground-true translation but with rotations optimized by our BA method, the points on the side wall are much thinner and form an apparent plane of standard deviation 6.8 cm (Fig. \ref{uzh}(e)); with poses fully optimized by our method, the points are even more consistent and the standard deviation is 1.7 cm (Fig. \ref{uzh}(f)). From these results, we suspect that the ground-truth may be affected by some unknown errors (e.g., marker position change during the data collection). Indeed, we found similar problem on this sequence also occurred in other works \cite{camurrihilti}. Moreover, the standard deviation of 1.7cm achieved by our method is exactly the ranging accuracy of the lidar sensor, which confirms that our method achieves a mapping accuracy at the lidar noise level as if the sensor had no motion. 

\subsubsection{Mapping quality} 

A significant advantage of the BA method is the direct optimization of the map consistency (i.e., point-to-plane residuals). To evaluate the map quality without a ground-true map, we adopt a method proposed by Anton {\it et al.} \cite{filatov20172d}. The method cuts the space into small cells and then counts the number of cells that lidar points occupy. The less the occupied cells, the higher the map quality. This indicator is intuitive: if points from different scans are registered accurately, they should agree with each other to the best extent, hence occupying the minimum possible number of cells. Based on this method, Table \ref{benchmark2 occupied} presents the number of occupied cell with size 0.1 m. {To better show the difference among different methods, the number of occupied cells are subtracted by our method for each sequence. We show the number of occupied cells by our method and the difference value of other methods.} As can be seen, our methods consistently achieved the best performance in all sequences and the next best is {PA (inner)}, PA, and BAREG. This trend also agrees with the ATE results very well.

\begin{figure*} [!t] 
    \centering
    \includegraphics[width=0.85\linewidth]{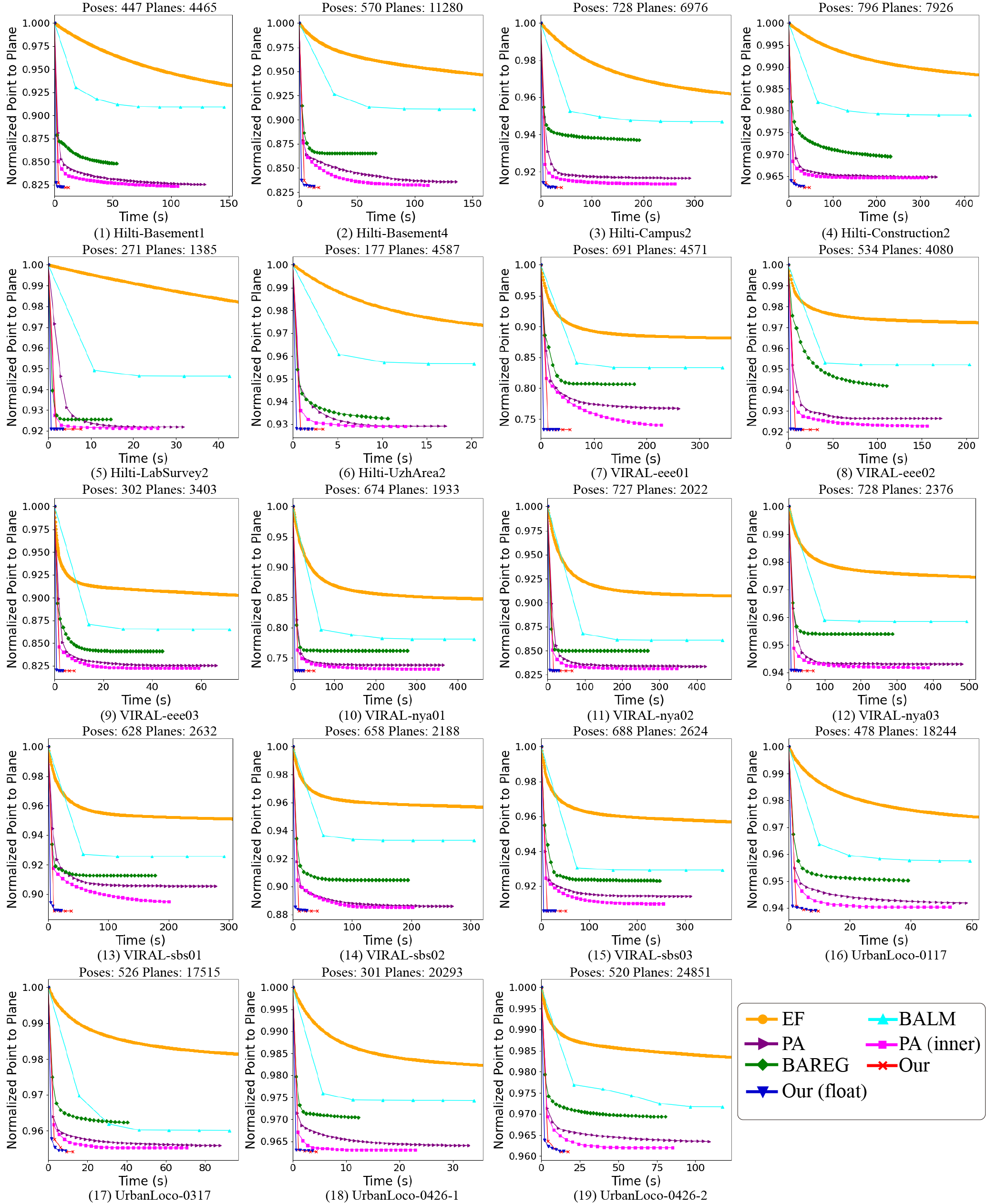}
    \caption{Point-to-plane distance versus optimization time in real-world datasets including Hilti, VIRAL, and UrbanLoco. All methods have the same initial pose (hence the same initial point-to-plane distance) and have their point-to-plane distance all normalized by the initial values.}
    \label{fig:benchmark2_converge}
\end{figure*}

\subsubsection{Computation time} 

Finally, we compare the computation time.  Since the pairwise registration methods, including ICP, GICP, and NDT, perform repetitive incremental registration at each scan reception, its computation time is very different from the BA methods that perform batch optimization on all scans at once. Therefore, we only compare the computation time of BA methods. {Fig. \ref{fig:benchmark2_converge} shows the convergence  of all methods and} Table \ref{benchmark2 time} shows the total optimization time. As can be seen, {when all using double-precision}, our method consumes the least computation time, about one fourth of the BAREG, one sixth of PA {and PA (inner)}, one eighth of BALM, and one twentieth of EF. The overall trend agrees well with the results on synthetic point cloud in Section \ref{virtual_point_cloud}-3 with explanations detailed therein. {Besides, our single-precision implementation reduces $40\%$ further optimization time while still outperforming the other BA methods as detailed in previous section. Finally, the inclusion of extra edge features increases the number of cost items, resulting in an increased optimization time.}

\subsubsection{{Plane merging}}

\begin{figure} [ht]
    \centering
    \includegraphics[width=0.95\linewidth]{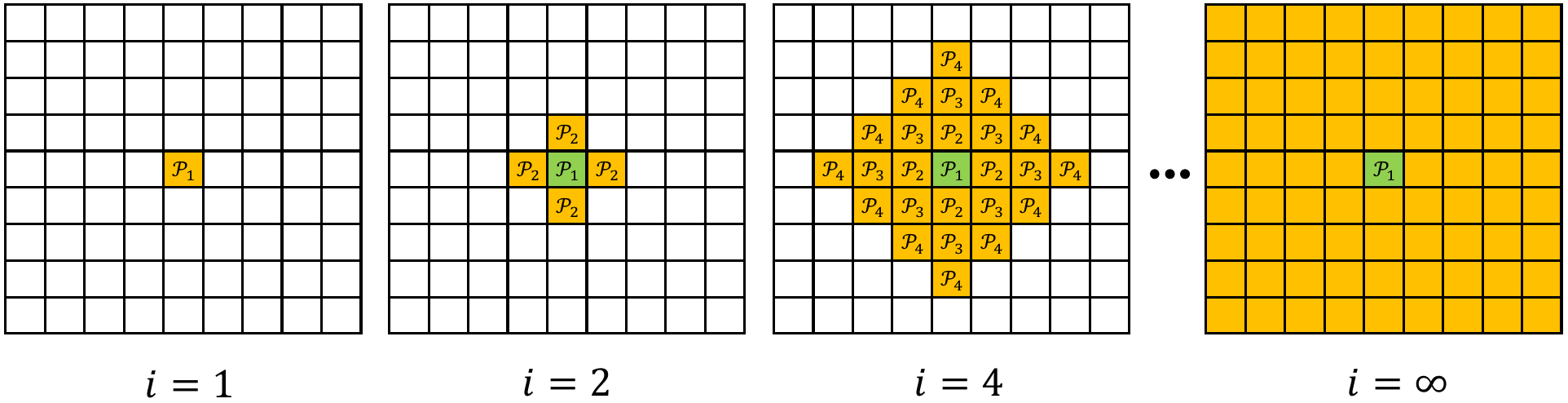}
    \caption{{Plane merging at degree $i$: ``$i=1$" indicates only merging the planes in the same root voxel. ``$i=2$" means merging maximum layers to $\mathcal P_2$ and ``$i=4$" means merging maximum layers to $\mathcal P_4$. ``$i=\infty$" means no boundary layer is specified, the merging can go as far as possible.}}
    \label{fig:plane-merging}
\end{figure}

We further evaluate the performance of all BA methods at different number of plane features. To change the number of planes in real-world datasets, we develop a merging procedure in addition to the adaptive voxelization introduced in the experiment setup above. Starting from the root voxels, the adaptive voxelization recursively cuts the space into smaller sub-voxels until the sub-voxel contains only one plane feature. Then the merging process merges planes in small sub-voxels into larger planes. The merging proceeds at different degrees denoted by $i$ (see Fig. \ref{fig:plane-merging}), where a plane is merged with planes within up to $i-1$ layers of neighboring root voxels. In the merging process, the two candidate planes $\boldsymbol{\mathcal{P}}_i$ and $\boldsymbol{\mathcal{P}}_j$ must satisfy

\begin{align}
	&\Big | \left\langle \mathbf n_i, \mathbf n_j \right\rangle  \Big |
	< \epsilon_1 
    \\d
	\Big | \left\langle \mathbf c_i - \mathbf c_j, \mathbf n_{i}\right\rangle - \frac{\pi}{2} & \Big | < \epsilon_2 
    \quad
	\Big | \left\langle \mathbf c_i - \mathbf c_j, \mathbf n_{j}\right\rangle - \frac{\pi}{2} \Big | < \epsilon_2 \label{two-plane-judge}
\end{align}
{where $\mathbf n$ and $\mathbf c$ are the normal vector and center of a plane respectively, symbol $\left\langle \cdot \right\rangle$ denotes the angle of two vectors, $\epsilon_1 = \epsilon_2 = 10^{\circ}$ are two constants. If the condition is not satisfied, the neighboring plane will not be merged.}

{Given a merging degree $i$, we repeatedly merge planes starting from a seed plane randomly selected from the plane list. A merged plane will be removed from the list to avoid duplicate merging. Such procedure produces a new list of planes whose size are at most $i\cdot L$ with $L=1$ or $2$ m being the root voxel size. Larger merging degree $i$ will lead to fewer number of planes but each with larger sizes.}

% We should note that the adaptive voxelization and merging method specified above are highly efficient and often produce planes with higher quality than existing plane extraction methods such as random Hough transform (RHT)\footnote{D. Borrmann, J. Elseberg, K. Lingemann, and A. Nuchter, "The 3d hough transform for plane detection in point clouds: A review and a new accumulator design," 3D Research, vol. 2, no. 2, pp. 1–13, 2011.}, random sample consensus (RANSAC)\footnote{https://pointclouds.org/} and region growing method \footnote{A. M. Araujo and M. M. Oliveira, "A robust statistics approach for plane detection in unorganized point clouds," Pattern Recognition, vol. 100, p. 107115, 2020.}. A detailed study of such method was presented in a preprint of our following-up work \hlr{[XXX]}. Evaluating our BA optimization method on this plane feature extraction and association method should be representative.

{Since the experimental results in the Sec. \ref{virtual_point_cloud} and \ref{real_dataset} have proved the PA with inner iteration outperforms the original PA, we use the PA with inner iteration by default.} The accuracy and computation time of all BA methods (including EF, BLAM, PA (inner), BAREG, and ours) at different merging degree $i$ are shown in Fig. \ref{fig:benchmark_merge}. The plane merging at different degrees leads to different computation time, so the time cost in the plot is the total time including adaptive voxelization, plane merging (if applicable), and BA optimization. As can be seen, our method consistently exhibits the highest accuracy and lowest time cost for all numbers of planes. Moreover, as the merging degree $i$ increases, the number of planes is decreased accordingly, leading to fewer planes that also reduce the optimization time of all BA methods. The reduction in optimization time is often larger than the time increment for merging, hence the total computation time still decreases with the merging degree. On the other hand, the pose RMSE of all methods all increase with the merging degree. This is because a larger merging degree introduces more bias to the optimization by merging planes not exactly on the same plane (e.g., slightly curved ground).

\begin{figure*} [!ht]
	\centering
	\includegraphics[width=0.85\linewidth]{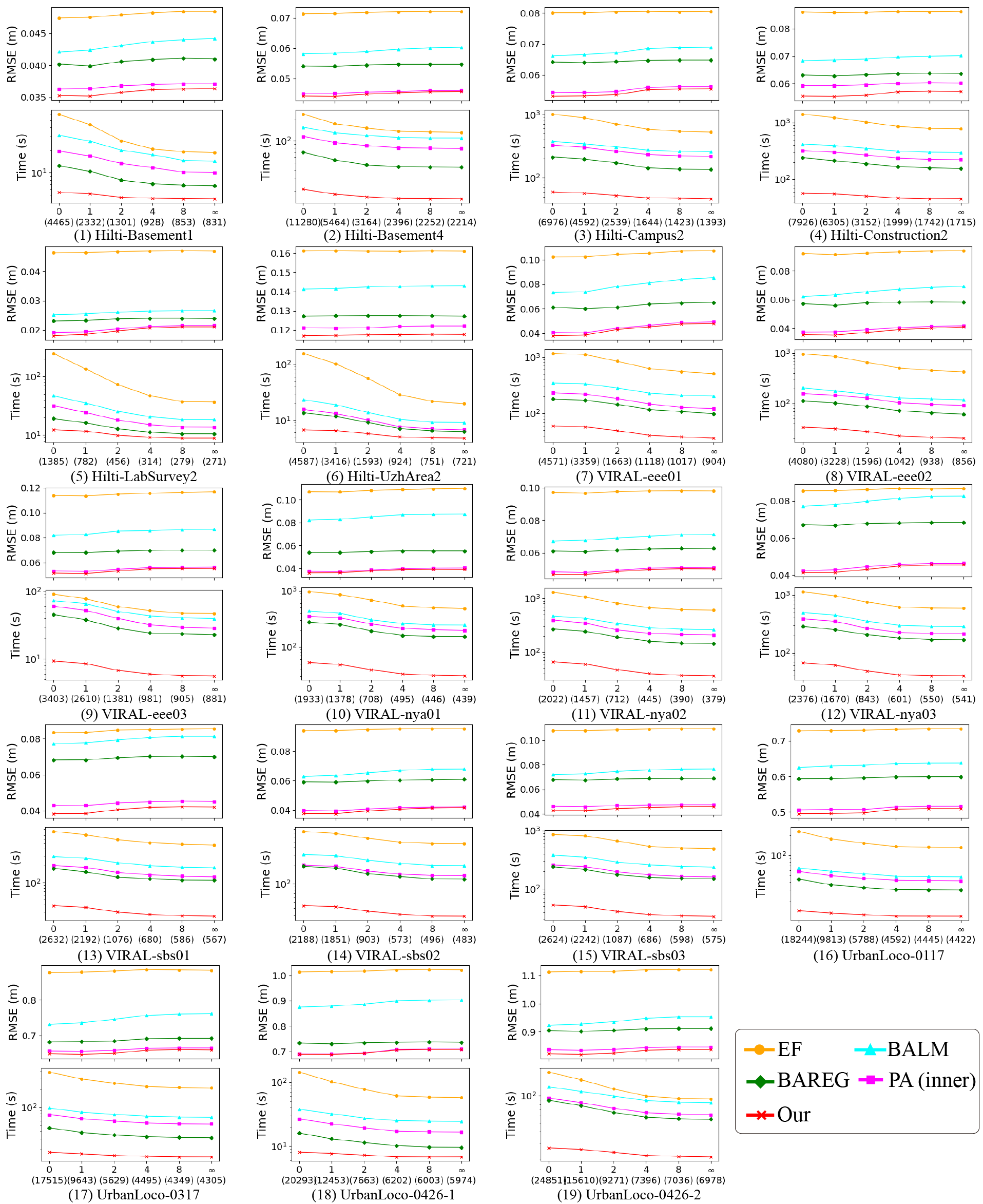}
	\caption{{The ATE and time cost of our method when merging planes at different degrees. The number $``0", ``1", ``2", ``4", ``\infty"$ on the X axis denotes the merging degree. The larger number in the parentheses below them are the number of planes corresponding to the merging degree.} }
	\label{fig:benchmark_merge}
    \vspace{-0.4cm}
\end{figure*}

% One limitation of adaptive voxelization is the fragmentation of planes into smaller voxels, which consequently incurs bring additional time costs when computing Hessian matrix. To address this issue, we have conducted a plane merging experiment, the details of which are described in the Section IV of the supplementary material. The results, presented in Fig. 19 of the supplementary material, demonstrate that our method maintains the best accuracy and the lowest time cost across varying numbers of planes. Furthermore, by setting a reasonable maximum for merging layers, the accuracy remains consistent, but the computation time reduces steadily.

% We can also notice that our method with merge outperforms that without merging. This is expected, as plane merging will reduce the number of planes (fewer cost item) and clusters (fewer point clusters $\mathbf C_{f_{ij}}$ in (\ref{BA-formulation-reduced-reduced})). 

%The Jacobin and Hessian matrices for different voxel are parallel, so we scheduled four threads to process relevant computation. Table \ref{benchmark2 time} shows that our method is the fastest. The accuracy of adaptive voxelization with merging and without merging is almost the same, but due to less plane correspondence to be optimized, the merging version can be more time-saving. 

\begin{table*}[t]
\caption{Optimization time for different methods}
\centering
\begin{threeparttable}
    {\begin{tabular}{llrrrrrrrrr}
    \toprule
    Datasets & Sequence & EF & BALM & {PA} & PA ({inner}) & BAREG & Ours (float) & Ours (edge) & Ours \\
    \midrule
    \multirow{6}{*}{Hilti} 
    & Basement1     &  297.68 & 145.72 & {129.39} & 106.08 &  52.99 & \textbf{ 7.20} & 12.07 & \textit{11.94} \\
    & Basement4     &  231.37 & 151.45 & {135.88} & 111.39 &  65.67 & \textbf{12.72} & 17.25 & \textit{17.01} \\
    & Campus2       &  989.37 & 352.72 & {290.78} & 261.39 & 191.87 & \textbf{27.09} & 40.02 & \textit{39.95} \\
    & Construction2 & 1415.18 & 412.00 & {335.70} & 313.23 & 231.48 & \textbf{33.04} & 47.34 & \textit{47.12} \\
    & LabSurvey2    &  244.86 &  42.47 & { 31.63} &  25.67 &  14.59 & \textbf{ 3.39} &  7.89 & \textit{ 7.64} \\
    & UzhArea2      &  153.43 &  20.25 & { 17.10} &  12.60 &  10.60 & \textbf{ 2.16} &  4.32 & \textit{ 4.08} \\
    \midrule
    \multirow{9}{*}{VIRAL}
    & eee01 & 1162.25 & 342.60 & {259.95} & 227.86 & 175.83 & \textbf{33.01} & 56.21 & \textit{55.22} \\
    & eee02 &  968.90 & 202.98 & {171.41} & 155.86 & 110.31 & \textbf{14.06} & 33.21 & \textit{32.34} \\
    & eee03 &   89.45 &  71.56 & { 66.11} &  59.02 &  44.10 & \textbf{ 3.27} &  8.17 & \textit{ 7.89} \\
    & nya01 &  972.81 & 438.09 & {364.18} & 351.14 & 276.37 & \textbf{31.19} & 51.78 & \textit{51.01} \\
    & nya02 & 1307.53 & 468.30 & {422.34} & 394.28 & 268.28 & \textbf{30.04} & 65.69 & \textit{65.19} \\
    & nya03 & 1134.21 & 493.29 & {479.64} & 385.79 & 287.19 & \textbf{39.26} & 67.73 & \textit{67.26} \\
    & sbs01 &  818.50 & 291.77 & {278.02} & 200.21 & 177.46 & \textbf{21.82} & 38.93 & \textit{37.80} \\
    & sbs02 &  738.91 & 304.65 & {268.42} & 201.61 & 193.68 & \textbf{27.45} & 42.54 & \textit{41.35} \\
    & sbs03 &  855.22 & 377.82 & {312.31} & 254.52 & 237.45 & \textbf{23.45} & 52.05 & \textit{51.55} \\
    \midrule
    \multirow{4}{*}{Urbanloco}
    & 0117   & 224.73 &  59.28 & { 58.18} & 52.80 &  39.08 & \textbf{ 8.73} &  9.92 & \textit{ 9.60} \\
    & 0317   & 380.98 &  92.11 & { 87.48} & 70.35 &  40.25 & \textbf{ 8.89} & 13.20 & \textit{12.38} \\
    & 0426-1 & 138.75 &  33.91 & { 32.92} & 22.90 &  12.29 & \textbf{ 3.70} &  5.77 & \textit{ 4.31} \\
    & 0426-2 & 174.40 & 117.36 & {108.79} & 85.26 &  80.44 & \textbf{14.47} & 17.92 & \textit{17.34} \\
    \midrule
    Average && 647.29 & 232.54 & {202.64} &171.11 & 132.10 & \textbf{18.15} & 31.16 & \textit{30.58} \\ 
    \bottomrule
    \end{tabular}}
\end{threeparttable}
\label{benchmark2 time}
\end{table*}

\section{Applications}
Bundle adjustment is the central technique of many lidar-based applications. In this section, we show how our bundle adjustment method can effectively improve the accuracy or computation efficiency of three vital applications: lidar-inertial odometry, multi-lidar calibration, and global mapping. Constrained by the page limit, details about the incorporation of bundle adjustment method in these applications and its effectiveness in real-world experiments are presented in Section I of the supplementary materials \cite{LiuZheng2022supplementary}.

\section{Discussion}

Here we discuss the efficiency, accuracy, and extendability of the proposed bundle adjustment method.

\subsection{Efficiency}

Our method achieved lower computation time than other state-of-the-art counterparts. The efficiency of our method are attributed to three inter-related and rigorously-proved techniques that make fully use of the problem nature and lidar point cloud property. The first technique is the solving of feature parameters in a closed-form before the BA optimization. It allows the feature parameters to be removed from the optimization, which fundamentally reduces the optimization dimension to the dimension of the pose only, a phenomenon that did not exist before in visual bundle adjustment problem. The second technique is a second-order solver which fits the quadratic cost function naturally and leads to fast convergence in the iterative optimization. This is enabled by the analytical derivation of the closed-form Jacobian and Hessian matrices of the cost function. The third technique is the point cluster, which enables the aggregation of all raw points without enumerating each individual point in neither of the cost evaluation, derivatives evaluation, or uncertainty evaluation. Collectively, these three techniques lead to an BA optimization with much lower dimension and time complexity. % of $O(M_f M_p + M_f M_p^2 + M_p^3)$ in each iteration, which achieves the lowest time consumption when compared with other second order methods (e.g., BALM \cite{liu2021balm} and plane adjustment \cite{zhou2020efficient}). 

%Being linear to the feature number and cubic to the pose number, our method seems to reach the theoretical limit of an exact second-order method. To overcome this limit, especially the cubic growth with pose number, in the future, we could explore the connectivity of the poses and divide the BA optimization into multiple smaller problems. 

%exploit the sparsity in the Hessian matrix $\mathbf{H}$, which is completely ignored in this paper, and the incremental nature of the data reception. Since the problem is not in a least square form (the cost function in (\ref{BA-formulation-reduced-reduced}) is not a sum of squares), incremental smoothing and mapping technique such as \cite{kaess2008isam} commonly used in visual SLAM cannot be applied directly and a holistic new derivation is required. Another possible approach is parallelizing the decomposition of the Hessian matrix. Noticing that the overall Hessian matrix is the summation of many smaller (and sparser) Hessian matrices, each corresponds to one cost item, their decomposition could be parallelized and then combined to produce the overall decomposition. We would explore these interesting topics in the future. 

\vspace{-0.4cm}
\subsection{Accuracy}

Benefiting from the point cluster technique, our proposed method is able to exploit the information of all raw point measurements, achieving high pose estimation accuracy (a few centimeters) at the level of lidar measurement noise. Optimization from the raw lidar points also enables the developed method to estimate the uncertainty level of the estimated pose, which may be useful when this information is further fused with measurements from other sensors (e.g., IMU sensors). Moreover,  by minimizing the Euclidean distance from each raw point to the corresponding feature, our method can reinforce the map consistency in a more direct manner than conventional pose graph optimization.  While at a higher computation cost (due to the more complete consideration of features co-visible in multiple scans), it considerably improves the mapping accuracy which is important for mapping applications. Due to this reason, our method is particularly useful for accuracy refinements from a baseline pose trajectory that can be obtained by an odometry or a pose graph optimization module. The second order optimization provides very fast convergence when the solution is near to the optimal value, preventing premature solutions. 

\vspace{-0.4cm}
\subsection{Extendability}

As a basic technique for multiple scan registration, our proposed method can be easily be integrated with other formality of data, such as images and IMU measurements, by incorporating visual bundle adjustment factors and IMU pre-integration factors \cite{forster2016manifold} in the optimization. Moreover, besides the frame-based pose trajectory, which attaches each frame an independent pose to estimate, our method can also work with other forms of pose trajectories, such as continuous-time trajectories based on Splines \cite{bosse2012zebedee, droeschel2018efficient} or Gaussian Process models \cite{tong2013gaussian, le2020in2laama}, which have the capability to compensate the in-frame motion distortion.  According to the chain rules, the derivatives of the BA cost with respect to the trajectory parameters will consist of two parts: the first is the derivative of the BA cost with respect to the pose of each point cluster as derived in this paper, and the second part is the derivatives of the pose with respect to the trajectory parameters, which depends on the specific trajectories being used.

\section{Conclusion}

This paper proposed a novel bundle adjustment method for lidar point cloud. The central of the proposed method is a point cluster technique, which aggregates all raw points into a compact set of parameters without enumerating each individual point. The paper showed how the bundle adjustment problem can be represented by the point cluster and also derived the analytical form of the Jacobian and Hessian matrices based on the point cluster. Based on these derivations, the paper developed a second-order solver, which estimates both the pose and the pose uncertainty. %Besides the optimization solvers, the paper also considered the problem of feature association and proposed an adaptive voxelization and merging technique. 
The developed BA method is open sourced to benefit the community.

\vspace{-0.09cm}
Besides the technical developments, this paper also made some theoretical contributions, including the formalization of the point cluster and its operations, revealing of the invariance property of the formulated BA optimization, the proof of null space and sparsity of the derived Jacobian and Hessan matrices, and the time complexity analysis of the proposed BA method and its comparison with others. These theoretical results serve the foundation of our developed BA techniques. 

\vspace{-0.09cm}
The proposed methods and implementations were extensively verified in both simulation and real-world experiments, in terms of consistency, efficiency, accuracy, and robustness. In all evaluations, the proposed method achieved consistently higher accuracy while consuming significantly lower computation time. This paper further demonstrated three applications of the BA techniques, including lidar-inertial odometry, multi-lidar calibration, and high-accuracy mapping. In all applications, the adoption of BA method could effectively improve the accuracy or the efficiency. 

In the future, we would like to incorporate the BA method more tightly to the above applications and beyond. This would require more thorough considerations of many practical issues, such as point cloud motion compensation, removal of dynamic objects, tightly-fusion with other formality of sensor data (e.g., IMU, camera) and module (e.g., loop closure).  

\bibliography{bare_jrnl}

\clearpage
\pagestyle{empty}
\thispagestyle{empty}


\section*{Supplementary material}

# Efficient and Consistent Bundle Adjustment on Lidar Point Clouds (Supplementary)

Zheng Liu, Xiyuan Liu and Fu Zhang

Please note that equation numbers and section numbers from the main manuscript are labelled in this letter in red.

## I. LEMMAS

**Lemma 1.** *For a scalar $x \in \mathbb{R}$ and a matrix $\mathbf{A} \in \mathbb{S}^{3 \times 3}$ which depends on $x$, we have the two following conclusions.*

$$\frac{\partial \lambda_l(x)}{\partial x} = \mathbf{u}_l(x)^T \frac{\partial \mathbf{A}(x)}{\partial x} \mathbf{u}_l(x) \tag{1}$$

$$\frac{\partial \mathbf{u}_l(x)}{\partial x} = \sum_{k=1, k \neq l}^{3} \frac{1}{\lambda_l - \lambda_k} \mathbf{u}_k(x) \mathbf{u}_k(x)^T \frac{\partial \mathbf{A}(x)}{\partial x} \mathbf{u}_l(x) \tag{2}$$

*where $\lambda_l$ ($l = 1, 2, 3$) denotes the $l$-th largest eigenvalue and $\mathbf{u}_l$ is the corresponding eigenvector.*

*Proof.* Since the matrix $\mathbf{A}(x)$ is symmetric, its singular value decomposition is,

$$\mathbf{A}(x) = \mathbf{U}(x)\mathbf{\Lambda}(x)\mathbf{U}(x)^T \tag{3}$$

where $\mathbf{\Lambda}(x) = \text{diag}\left(\lambda_1(x), \lambda_2(x), \lambda_3(x)\right)$ consists of all the eigenvalues and $\mathbf{U}(x) = \begin{bmatrix} \mathbf{u}_1(x) & \mathbf{u}_2(x) & \mathbf{u}_3(x) \end{bmatrix}$ is an orthonormal matrix consisting of the eigenvectors. Therefore,

$$\mathbf{\Lambda}(x) = \mathbf{U}(x)^T \mathbf{A}(x) \mathbf{U}(x) \tag{4}$$

Both sides take the derivative of $x$,

$$\frac{\partial \mathbf{\Lambda}(x)}{\partial x} = \mathbf{U}(x)^T \frac{\partial \mathbf{A}(x)}{\partial x} \mathbf{U}(x) + \mathbf{U}(x)^T \mathbf{A}(x) \frac{\partial \mathbf{U}(x)}{\partial x}$$
$$+ \left(\frac{\partial \mathbf{U}(x)}{\partial x}\right)^T \mathbf{A}(x) \mathbf{U}(x) \tag{5}$$

Since $\mathbf{U}(x)^T \mathbf{A}(x) = \mathbf{\Lambda}(x) \mathbf{U}(x)^T$ and $\mathbf{A}(x) \mathbf{U}(x) = \mathbf{U}(x) \mathbf{\Lambda}(x)$, the equation is

$$\frac{\partial \mathbf{\Lambda}(x)}{\partial x} = \mathbf{U}(x)^T \frac{\partial \mathbf{A}(x)}{\partial x} \mathbf{U}(x) + \mathbf{\Lambda}(x) \underbrace{\mathbf{U}(x)^T \frac{\partial \mathbf{U}(x)}{\partial x}}_{\mathbf{D}(x)}$$
$$+ \underbrace{\left(\frac{\partial \mathbf{U}(x)}{\partial x}\right)^T \mathbf{U}(x)}_{\mathbf{D}^T(x)} \mathbf{\Lambda}(x) \tag{6}$$

Denote $\mathbf{D}(x) \triangleq \mathbf{U}(x)^T \frac{\partial \mathbf{U}(x)}{\partial x}$. Since $\mathbf{U}(x)\mathbf{U}(x)^T = \mathbf{I}$, differentiating both sides with respect to $x$ leads to,

$$\mathbf{U}(x)^T \frac{\partial \mathbf{U}(x)}{\partial x} + \left(\frac{\partial \mathbf{U}(x)}{\partial x}\right)^T \mathbf{U}(x) = \mathbf{0}$$
$$\Rightarrow \mathbf{D}(x) + \mathbf{D}^T(x) = \mathbf{0}$$

It is seen that $\mathbf{D}(x)$ is a skew symmetric matrix whose diagonal elements are zeros. Moreover, since $\mathbf{\Lambda}(x)$ is diagonal,

the last two items of the right side of (6) sum to zero on diagonal positions. Only considering the diagonal elements in (6) leads to

$$\frac{\partial \lambda_l(x)}{\partial x} = \mathbf{u}_l(x)^T \frac{\partial \mathbf{A}(x)}{\partial x} \mathbf{u}_l(x), l \in \{1, 2, 3\} \tag{7}$$

which yields the first conclusion. Now we aims to prove the second one. In (6), $\frac{\partial \mathbf{\Lambda}(x)}{\partial x}$ is diagonal matrix and thus for the off-diagonal, $k$-th row, $l$-th column, element ($k \neq l$),

$$0 = \mathbf{u}_k(x)^T \frac{\partial \mathbf{A}(x)}{\partial x} \mathbf{u}_l(x) + \lambda_k D_x^{k,l} - D_x^{k,l} \lambda_l \tag{8}$$

where $D_x^{k,l}$ is the $k$-th row, $l$-th column element in the skew symmetric $\mathbf{D}(x)$ and satisfy $D_x^{k,l} = -D_x^{l,k}$. From (8), we can solve $D_x^{k,l}$

$$D_x^{k,l} = \begin{cases} \dfrac{1}{\lambda_l - \lambda_k} \mathbf{u}_k(x)^T \dfrac{\partial \mathbf{A}(x)}{\partial x} \mathbf{u}_l(x), k \neq l \\ 0 \qquad\qquad\qquad\qquad\quad , k = l \end{cases} \tag{9}$$

Since $\mathbf{D}(x) \triangleq \mathbf{U}(x)^T \frac{\partial \mathbf{U}(x)}{\partial x}$, we have $\frac{\partial \mathbf{U}(x)}{\partial x} = \mathbf{U}(x)\mathbf{D}(x)$. Taking the $l$-th column on both sides leads to

$$\frac{\partial \mathbf{u}_l(x)}{\partial x} = \mathbf{U}(x)\mathbf{D}_x^{:,l}, \tag{10}$$

where $\mathbf{D}_x^{:,l} \in \mathbb{R}^3$ represents the $l$-th column of $\mathbf{D}(x)$. Finally, substituting $\mathbf{U}(x) = \begin{bmatrix} \mathbf{u}_1(x) & \mathbf{u}_2(x) & \mathbf{u}_3(x) \end{bmatrix}$ and (9) into (10), we obtain

$$\frac{\partial \mathbf{u}_l(x)}{\partial x} = \sum_{k=1, k \neq l}^{3} \frac{1}{\lambda_l - \lambda_k} \mathbf{u}_k(x) \mathbf{u}_k(x)^T \frac{\partial \mathbf{A}(x)}{\partial x} \mathbf{u}_l(x), \tag{11}$$

which yields the second conclusion.

$\square$

**Lemma 2.** *Given*

(1) *Matrices $\mathbf{C}_j = \begin{bmatrix} \mathbf{P}_j & \mathbf{v}_j \\ \mathbf{v}_j^T & N_j \end{bmatrix} \in \mathbb{S}^{4 \times 4}, j = 1, \cdots, M_p$;*

(2) *Poses $\mathbf{T}_j \in SE(3), j = 1, \cdots, M_p$;*

(3) *A matrix $\mathbf{C} = \begin{bmatrix} \mathbf{P} & \mathbf{v} \\ \mathbf{v}^T & N \end{bmatrix} \triangleq \sum_{j=1}^{M_p} \mathbf{T}_j \mathbf{C}_j \mathbf{T}_j^T \in \mathbb{S}^{4 \times 4}$ and a matrix function $\mathbf{A}(\mathbf{C}) \triangleq \frac{1}{N}\mathbf{P} - \frac{1}{N^2}\mathbf{v}\mathbf{v}^T$;*

(4) *Two constant eigenvectors $\mathbf{u}_k, \mathbf{u}_l \in \mathbb{R}^3$,*

*then the first and second order derivatives of $\mathbf{u}_k^T \mathbf{A}(\mathbf{T})\mathbf{u}_l$ w.r.t. $\mathbf{T}$ are:*

$$\mathbf{g}_{kl} \triangleq \frac{\partial \mathbf{u}_k^T \mathbf{A}(\delta\mathbf{T})\mathbf{u}_l}{\partial \delta\mathbf{T}} = \begin{bmatrix} \cdots & \mathbf{g}_{kl}^{\mathrm{q}} & \cdots \end{bmatrix} \in \mathbb{R}^{1 \times 6M_p}, \quad (12)$$

$$\mathbf{H}_{kl} \triangleq \frac{\partial^2 \mathbf{u}_k^T \mathbf{A}(\delta\mathbf{T})\mathbf{u}_l}{\partial (\delta\mathbf{T})^2} = \begin{bmatrix} & \vdots & \\ \cdots & \mathbf{H}_{kl}^{\mathrm{p,q}} & \cdots \\ & \vdots & \end{bmatrix} \in \mathbb{R}^{6M_p \times 6M_p}, \tag{13}$$

*where* $\mathbf{g}_{kl}^{\mathrm{q}} \in \mathbb{R}^{1\times6}$, $\mathbf{H}_{kl}^{\mathrm{p,q}} \in \mathbb{R}^{6\times6}$, $\forall \mathrm{p,q} \in \{1, \cdots, M_p\}$, *are*

$$\mathbf{g}_{kl}^{\mathrm{q}} = \frac{1}{N}\mathbf{u}_l^T\mathbf{S}_{\mathbf{P}}(\mathbf{T}_{\mathrm{q}} - \frac{1}{N}\mathbf{CF})\mathbf{C}_{\mathrm{q}}\mathbf{T}_{\mathrm{q}}^T\mathbf{U}_k^T$$
$$+ \frac{1}{N}\mathbf{u}_k^T\mathbf{S}_{\mathbf{P}}(\mathbf{T}_{\mathrm{q}} - \frac{1}{N}\mathbf{CF})\mathbf{C}_{\mathrm{q}}\mathbf{T}_{\mathrm{q}}^T\mathbf{U}_l^T \tag{14}$$

$$\mathbf{H}_{kl}^{\mathrm{p,q}} = -\frac{2}{N^2}\mathbf{U}_k\mathbf{T}_{\mathrm{p}}\mathbf{D}_{\mathrm{p,q}}\mathbf{T}_{\mathrm{q}}^T\mathbf{U}_l^T + \mathbb{1}_{\mathrm{p=q}}\Big(\frac{2}{N}\mathbf{U}_k\mathbf{T}_{\mathrm{q}}\mathbf{C}_{\mathrm{q}}\mathbf{T}_{\mathrm{q}}^T\mathbf{U}_l^T$$
$$+ \begin{bmatrix} \mathbf{K}_{kl}^{\mathrm{q}} & \mathbf{0}_{3\times3} \\ \mathbf{0}_{3\times3} & \mathbf{0}_{3\times3} \end{bmatrix}\Big) \tag{15}$$

$$\mathbf{K}_{kl}^{\mathrm{q}} = \frac{1}{N}\lfloor\mathbf{S}_{\mathbf{P}}\mathbf{T}_{\mathrm{q}}\mathbf{C}_{\mathrm{q}}(\mathbf{T}_{\mathrm{q}} - \frac{1}{N}\mathbf{CF})^T\mathbf{S}_{\mathbf{P}}^T\mathbf{u}_k\rfloor\lfloor\mathbf{u}_l\rfloor$$
$$+ \frac{1}{N}\lfloor\mathbf{u}_k\rfloor\lfloor\mathbf{S}_{\mathbf{P}}\mathbf{T}_{\mathrm{q}}\mathbf{C}_{\mathrm{q}}(\mathbf{T}_{\mathrm{q}} - \frac{1}{N}\mathbf{CF})^T\mathbf{S}_{\mathbf{P}}^T\mathbf{u}_l\rfloor \tag{16}$$

*where*

$$\mathbf{U}_l = \begin{bmatrix} -\lfloor\mathbf{u}_l\rfloor & \mathbf{0}_{3\times1} \\ \mathbf{0}_{3\times3} & \mathbf{u}_l \end{bmatrix} \quad \mathbf{F} = \begin{bmatrix} \mathbf{0}_{3\times3} & \mathbf{0}_{3\times1} \\ \mathbf{0}_{1\times3} & 1 \end{bmatrix} \tag{17}$$

$$\mathbf{S}_{\mathbf{P}} = \begin{bmatrix} \mathbf{I}_{3\times3} & \mathbf{0}_{3\times1} \end{bmatrix} \quad \mathbf{D}_{\mathrm{p,q}} = \mathbf{C}_{\mathrm{p}}\mathbf{F}\mathbf{C}_{\mathrm{q}}, \tag{18}$$

*and*

$$\mathbf{g}_{kl}^{\mathrm{q}} = \mathbf{0}_{1\times6}, \quad \mathbf{H}_{kl}^{\mathrm{p,q}} = \mathbf{0}_{6\times6}, \forall k, l,$$
$$\text{if p or q} \in \mathcal{I} \triangleq \{j|1 \leq j \leq M_p, \mathbf{C}_j = \mathbf{0}\}. \tag{19}$$

*Proof.* Partition the matrix $\mathbf{C}$ as

$$\mathbf{C} = \begin{bmatrix} \mathbf{P} & \mathbf{v} \\ \mathbf{v}^T & N \end{bmatrix} = \sum_{j=1}^{M_p} \mathbf{T}_j\mathbf{C}_j\mathbf{T}_j^T, \tag{20}$$

then

$$\mathbf{P} = \mathbf{S}_{\mathbf{P}}\mathbf{C}\mathbf{S}_{\mathbf{P}}^T = \sum_{j=1}^{M_p}\left(\mathbf{R}_j\mathbf{P}_j\mathbf{R}_j^T + \mathbf{R}_j\mathbf{v}_j\mathbf{t}_j^T\right.$$
$$\left. + \mathbf{t}_j\mathbf{v}_j^T\mathbf{R}_j^T + N_j\mathbf{t}_j\mathbf{t}_j^T\right), \tag{21}$$

$$\mathbf{v} = \mathbf{S}_{\mathbf{P}}\mathbf{C}\mathbf{S}_{\mathbf{v}}^T = \sum_{j=1}^{M_p}(\mathbf{R}_j\mathbf{v}_j + N_j\mathbf{t}_j), N = \sum_{j=1}^{M_p}N_j, \tag{22}$$

$$\mathbf{S}_{\mathbf{P}} = \begin{bmatrix} \mathbf{I}_{3\times3} & \mathbf{0}_{3\times1} \end{bmatrix} \in \mathbb{R}^{3\times4}, \tag{23}$$

$$\mathbf{S}_{\mathbf{v}} = \begin{bmatrix} \mathbf{0}_{1\times3} & 1 \end{bmatrix} \in \mathbb{R}^{1\times4}. \tag{24}$$

Therefore,

$$\mathbf{A}\left(\sum_{j=1}^{M_p}\mathbf{T}_j\mathbf{C}_j\mathbf{T}_j^T\right) = \frac{1}{N}\mathbf{P} - \frac{1}{N^2}\mathbf{v}\mathbf{v}^T \tag{25}$$

$$= \frac{1}{N}\mathbf{S}_{\mathbf{P}}\left(\mathbf{C} - \frac{1}{N}\mathbf{C}\mathbf{S}_{\mathbf{v}}^T\mathbf{S}_{\mathbf{v}}\mathbf{C}^T\right)\mathbf{S}_{\mathbf{P}}^T \tag{26}$$

$$= \frac{1}{N}\mathbf{S}_{\mathbf{P}}\left(\sum_{\mathrm{q}=1}^{M_p}\mathbf{T}_{\mathrm{q}}\mathbf{C}_{\mathrm{q}}\mathbf{T}_{\mathrm{q}}^T\right.$$
$$\left. - \frac{1}{N}\sum_{\mathrm{p}=1}^{M_p}\sum_{\mathrm{q}=1}^{M_p}\mathbf{T}_{\mathrm{p}}\mathbf{C}_{\mathrm{p}}\mathbf{T}_{\mathrm{p}}^T\mathbf{S}_{\mathbf{v}}^T\mathbf{S}_{\mathbf{v}}\mathbf{T}_{\mathrm{q}}\mathbf{C}_{\mathrm{q}}\mathbf{T}_{\mathrm{q}}^T\right)\mathbf{S}_{\mathbf{P}}^T. \tag{27}$$

Since $\mathbf{T}_{\mathrm{p}}^T\mathbf{S}_{\mathbf{v}}^T\mathbf{S}_{\mathbf{v}}\mathbf{T}_{\mathrm{q}} = \begin{bmatrix} \mathbf{0}_{3\times3} & \mathbf{0}_{3\times1} \\ \mathbf{0}_{1\times3} & 1 \end{bmatrix}$, we obtain

$$\mathbf{A} = \frac{1}{N}\mathbf{S}_{\mathbf{P}}\left(\sum_{\mathrm{q}=1}^{M_p}\mathbf{T}_{\mathrm{q}}\mathbf{C}_{\mathrm{q}}\mathbf{T}_{\mathrm{q}}^T - \frac{1}{N}\sum_{\mathrm{p}=1}^{M_p}\sum_{\mathrm{q}=1}^{M_p}\mathbf{T}_{\mathrm{p}}\mathbf{D}_{\mathrm{p,q}}\mathbf{T}_{\mathrm{q}}^T\right)\mathbf{S}_{\mathbf{P}}^T, \tag{28}$$

$$\mathbf{D}_{\mathrm{p,q}} = \mathbf{C}_{\mathrm{p}}\mathbf{F}\mathbf{C}_{\mathrm{q}}, \quad \mathbf{F} = \begin{bmatrix} \mathbf{0}_{3\times3} & \mathbf{0}_{3\times1} \\ \mathbf{0}_{1\times3} & 1 \end{bmatrix} \tag{29}$$

where we omitted the input argument of $\mathbf{A}$ for the sake of notation simplicity. Since $N = \sum_{j=1}^{M_p} N_j$ is a constant number that is irrelevant to the pose $\mathbf{T}$, perturbing the input of $\mathbf{A}$, which is the pose vector $\mathbf{T}$, by $\delta\mathbf{T}$ yields

$$\mathbf{u}_k^T\mathbf{A}(\delta\mathbf{T})\mathbf{u}_l = \frac{1}{N}\mathbf{u}_k^T\mathbf{S}_{\mathbf{P}}\left(\sum_{\mathrm{q}=1}^{M_p}\left(\mathbf{T}_{\mathrm{q}} \boxplus \delta\mathbf{T}_{\mathrm{q}}\right)\mathbf{C}_{\mathrm{q}}\left(\mathbf{T}_{\mathrm{q}} \boxplus \delta\mathbf{T}_{\mathrm{q}}\right)^T\right.$$
$$\left. - \frac{1}{N}\sum_{\mathrm{p}=1}^{M_p}\sum_{\mathrm{q}=1}^{M_p}\left(\mathbf{T}_{\mathrm{p}} \boxplus \delta\mathbf{T}_{\mathrm{p}}\right)\mathbf{D}_{\mathrm{p,q}}\left(\mathbf{T}_{\mathrm{q}} \boxplus \delta\mathbf{T}_{\mathrm{q}}\right)^T\right)\mathbf{S}_{\mathbf{P}}^T\mathbf{u}_l. \tag{30}$$

Based on the definition of $\boxplus$ on $SE(3)$ in (27), we define,

$$\mathbf{w}_{\mathrm{q}l}(\delta\mathbf{T}_{\mathrm{q}}) \triangleq (\mathbf{T}_{\mathrm{q}} \boxplus \delta\mathbf{T}_{\mathrm{q}})^T\mathbf{S}_{\mathbf{P}}^T\mathbf{u}_l$$
$$= \begin{bmatrix} \mathbf{R}_{\mathrm{q}}^T\exp^T(\lfloor\delta\boldsymbol{\phi}_{\mathrm{q}}\rfloor)\mathbf{u}_l \\ \mathbf{t}_{\mathrm{q}}^T\exp^T(\lfloor\delta\boldsymbol{\phi}_{\mathrm{q}}\rfloor)\mathbf{u}_l + \mathbf{u}_l^T\delta\mathbf{t}_{\mathrm{q}} \end{bmatrix} \tag{31}$$

When $\delta\boldsymbol{\phi}_{\mathrm{q}}$ is small, which is indeed the case for the purpose of derivative computation, we have

$$\exp(\lfloor\delta\boldsymbol{\phi}_{\mathrm{q}}\rfloor) \approx \mathbf{I} + \lfloor\delta\boldsymbol{\phi}_{\mathrm{q}}\rfloor + \frac{1}{2}\lfloor\delta\boldsymbol{\phi}_{\mathrm{q}}\rfloor^2 \tag{32}$$

Substituting (32) into $\mathbf{w}_{\mathrm{q}l}(\delta\mathbf{T}_{\mathrm{q}})$, we obtain

$$\mathbf{w}_{\mathrm{q}l}(\delta\mathbf{T}_{\mathrm{q}}) \approx \begin{bmatrix} \mathbf{R}_{\mathrm{q}}^T(\mathbf{I} - \lfloor\delta\boldsymbol{\phi}_{\mathrm{q}}\rfloor + \frac{1}{2}\lfloor\delta\boldsymbol{\phi}_{\mathrm{q}}\rfloor^2)\mathbf{u}_l \\ \mathbf{t}_{\mathrm{q}}^T(\mathbf{I} - \lfloor\delta\boldsymbol{\phi}_{\mathrm{q}}\rfloor + \frac{1}{2}\lfloor\delta\boldsymbol{\phi}_{\mathrm{q}}\rfloor^2)\mathbf{u}_l + \mathbf{u}_l^T\delta\mathbf{t}_{\mathrm{q}} \end{bmatrix}$$
$$\approx \underbrace{\begin{bmatrix} \mathbf{R}_{\mathrm{q}}^T\mathbf{u}_l \\ \mathbf{t}_{\mathrm{q}}^T\mathbf{u}_l \end{bmatrix}}_{\bar{\mathbf{w}}_{\mathrm{q}l}} + \underbrace{\begin{bmatrix} \mathbf{R}_{\mathrm{q}}^T\lfloor\mathbf{u}_l\rfloor & \mathbf{0}_{3\times3} \\ \mathbf{t}_{\mathrm{q}}^T\lfloor\mathbf{u}_l\rfloor & \mathbf{u}_l^T \end{bmatrix}}_{\mathbf{J}_{\mathbf{w}_{\mathrm{q}l}} \in \mathbb{R}^{4\times6}}\underbrace{\begin{bmatrix} \delta\boldsymbol{\phi}_{\mathrm{q}} \\ \delta\mathbf{t}_{\mathrm{q}} \end{bmatrix}}_{\delta\mathbf{T}_{\mathrm{q}}} + \underbrace{\begin{bmatrix} \frac{1}{2}\mathbf{R}_{\mathrm{q}}^T\lfloor\delta\boldsymbol{\phi}_{\mathrm{q}}\rfloor^2\mathbf{u}_l \\ \frac{1}{2}\mathbf{t}_{\mathrm{q}}^T\lfloor\delta\boldsymbol{\phi}_{\mathrm{q}}\rfloor^2\mathbf{u}_l \end{bmatrix}}_{\delta\boldsymbol{\Phi}_{\mathrm{q}l}} \tag{33}$$

where $\bar{\mathbf{w}}_{\mathrm{q}l}$, $\delta\boldsymbol{\Phi}_{\mathrm{q}l}$ and $\mathbf{J}_{\mathbf{w}_{\mathrm{q}l}}$ can be simplified

$$\bar{\mathbf{w}}_{\mathrm{q}l} = (\mathbf{S}_{\mathbf{P}}\mathbf{T}_{\mathrm{q}})^T\mathbf{u}_l \quad \delta\boldsymbol{\Phi}_{\mathrm{q}l} = \frac{1}{2}(\mathbf{S}_{\mathbf{P}}\mathbf{T}_{\mathrm{q}})^T\lfloor\delta\boldsymbol{\phi}_{\mathrm{q}}\rfloor^2\mathbf{u}_l$$

$$\mathbf{J}_{\mathbf{w}_{\mathrm{q}l}} = \mathbf{T}_{\mathrm{q}}^T\mathbf{U}_l^T \quad \mathbf{U}_l = \begin{bmatrix} -\lfloor\mathbf{u}_l\rfloor & \mathbf{0}_{3\times1} \\ \mathbf{0}_{3\times3} & \mathbf{u}_l \end{bmatrix} \tag{34}$$

Substituting (33) into (30) and keeping terms up to the second order lead to

$$\mathbf{u}_k^T \mathbf{A}(\delta\mathbf{T})\mathbf{u}_l = \frac{1}{N}\sum_{q=1}^{M_p}\mathbf{w}_{qk}^T(\delta\mathbf{T}_q)\mathbf{C}_q\mathbf{w}_{ql}(\delta\mathbf{T}_q)$$

$$- \frac{1}{N^2}\sum_{p=1}^{M_p}\sum_{q=1}^{M_p}\mathbf{w}_{pk}^T(\delta\mathbf{T}_p)\mathbf{D}_{p,q}\mathbf{w}_{ql}(\delta\mathbf{T}_q) \quad (35)$$

$$= \frac{1}{N}\sum_{q=1}^{M_p}\Big(\big(\bar{\mathbf{w}}_{qk}^T\mathbf{C}_q\mathbf{J}_{\mathbf{w}_{ql}} + \bar{\mathbf{w}}_{ql}^T\mathbf{C}_q\mathbf{J}_{\mathbf{w}_{qk}}\big)\delta\mathbf{T}_q + \bar{\mathbf{w}}_{qk}^T\mathbf{C}_q\bar{\mathbf{w}}_{ql}$$

$$+ \delta\mathbf{T}_q^T\mathbf{J}_{\mathbf{w}_{qk}}^T\mathbf{C}_q\mathbf{J}_{\mathbf{w}_{ql}}\delta\mathbf{T}_q + \underbrace{\bar{\mathbf{w}}_{qk}^T\mathbf{C}_q\delta\mathbf{\Phi}_{ql} + \delta\mathbf{\Phi}_{qk}^T\mathbf{C}_q\bar{\mathbf{w}}_{ql}}_{\frac{1}{2}\delta\mathbf{T}_q^T\mathbf{Q}_{kl}^q\delta\mathbf{T}_q}\Big)$$

$$- \frac{1}{N^2}\sum_{p=1}^{M_p}\sum_{q=1}^{M_p}\Big(\delta\mathbf{T}_p^T\mathbf{J}_{\mathbf{w}_{pk}}^T\mathbf{D}_{p,q}\mathbf{J}_{\mathbf{w}_{ql}}\delta\mathbf{T}_q + \bar{\mathbf{w}}_{pk}^T\mathbf{D}_{p,q}\mathbf{J}_{\mathbf{w}_{ql}}\delta\mathbf{T}_q$$

$$+ \delta\mathbf{T}_p^T\mathbf{J}_{\mathbf{w}_{pk}}^T\mathbf{D}_{p,q}\bar{\mathbf{w}}_{ql} + \bar{\mathbf{w}}_{pk}^T\mathbf{D}_{p,q}\bar{\mathbf{w}}_{ql}$$

$$+ \underbrace{\bar{\mathbf{w}}_{pk}^T\mathbf{D}_{p,q}\delta\mathbf{\Phi}_{ql}}_{\frac{1}{2}\delta\mathbf{T}_q^T\mathbf{N}_{kl}^{p,q}\delta\mathbf{T}_q} + \underbrace{\delta\mathbf{\Phi}_{pk}^T\mathbf{D}_{p,q}\bar{\mathbf{w}}_{ql}}_{\frac{1}{2}\delta\mathbf{T}_p^T\mathbf{M}_{kl}^{p,q}\delta\mathbf{T}_p}\Big). \quad (36)$$

To compute $\mathbf{Q}_{kl}^q$ in $\frac{1}{2}\delta\mathbf{T}_q^T\mathbf{Q}_{kl}^q\delta\mathbf{T}_q$, note that $\mathbf{a}^T\lfloor\delta\phi\rfloor^2\mathbf{b} = \delta\phi^T\lfloor\mathbf{a}\rfloor\lfloor\mathbf{b}\rfloor\delta\phi$, $\forall\mathbf{a},\mathbf{b},\delta\phi\in\mathbb{R}^3$, we have

$$\bar{\mathbf{w}}_{qk}^T\mathbf{C}_q\delta\mathbf{\Phi}_{ql} = \frac{1}{2}\mathbf{u}_k^T\mathbf{S}_{\mathbf{P}}\mathbf{T}_q\mathbf{C}_q\mathbf{T}_q^T\mathbf{S}_{\mathbf{P}}^T\lfloor\delta\phi_q\rfloor^2\mathbf{u}_l$$
$$= \frac{1}{2}\delta\phi_q^T\lfloor\mathbf{S}_{\mathbf{P}}\mathbf{T}_q\mathbf{C}_q\mathbf{T}_q^T\mathbf{S}_{\mathbf{P}}^T\mathbf{u}_k\rfloor\lfloor\mathbf{u}_l\rfloor\delta\phi_q. \quad (37)$$

Similarly,

$$\delta\mathbf{\Phi}_{qk}^T\mathbf{C}_q\bar{\mathbf{w}}_{ql} = \frac{1}{2}\delta\phi_q^T\lfloor\mathbf{u}_k\rfloor\lfloor\mathbf{S}_{\mathbf{P}}\mathbf{T}_q\mathbf{C}_q\mathbf{T}_q^T\mathbf{S}_{\mathbf{P}}^T\mathbf{u}_l\rfloor\delta\phi_q. \quad (38)$$

Summing (37) and (38) and extending the $\delta\phi$ into $\delta\mathbf{T}$:

$$\mathbf{Q}_{kl}^q =$$
$$\begin{bmatrix}\lfloor\mathbf{S}_{\mathbf{P}}\mathbf{T}_q\mathbf{C}_q\mathbf{T}_q^T\mathbf{S}_{\mathbf{P}}^T\mathbf{u}_k\rfloor\lfloor\mathbf{u}_l\rfloor + \lfloor\mathbf{u}_k\rfloor\lfloor\mathbf{S}_{\mathbf{P}}\mathbf{T}_q\mathbf{C}_q\mathbf{T}_q^T\mathbf{S}_{\mathbf{P}}^T\mathbf{u}_l\rfloor & \mathbf{0} \\ \mathbf{0} & \mathbf{0}\end{bmatrix} \quad (39)$$

For $\frac{1}{2}\delta\mathbf{T}_q^T\mathbf{N}_{kl}^{p,q}\delta\mathbf{T}_q$ and $\frac{1}{2}\delta\mathbf{T}_p^T\mathbf{M}_{kl}^{p,q}\delta\mathbf{T}_p$,

$$\bar{\mathbf{w}}_{pk}^T\mathbf{D}_{p,q}\delta\mathbf{\Phi}_{ql} = \frac{1}{2}\mathbf{u}_k^T\mathbf{S}_{\mathbf{P}}\mathbf{T}_p\mathbf{D}_{p,q}\mathbf{T}_q^T\mathbf{S}_{\mathbf{P}}^T\lfloor\delta\phi_q\rfloor^2\mathbf{u}_l$$
$$= \frac{1}{2}\delta\phi_q^T\lfloor\mathbf{S}_{\mathbf{P}}\mathbf{T}_q\mathbf{D}_{qp}^T\mathbf{T}_p^T\mathbf{S}_{\mathbf{P}}^T\mathbf{u}_k\rfloor\lfloor\mathbf{u}_l\rfloor\delta\phi_q \quad (40)$$

$$\delta\mathbf{\Phi}_{pk}^T\mathbf{D}_{p,q}\bar{\mathbf{w}}_{ql} = \frac{1}{2}\mathbf{u}_k^T\lfloor\delta\phi_p\rfloor^2\mathbf{S}_{\mathbf{P}}\mathbf{T}_p\mathbf{D}_{p,q}\mathbf{T}_q^T\mathbf{S}_{\mathbf{P}}^T\mathbf{u}_l$$
$$= \frac{1}{2}\delta\phi_p^T\lfloor\mathbf{u}_k\rfloor\lfloor\mathbf{S}_{\mathbf{P}}\mathbf{T}_p\mathbf{D}_{p,q}\mathbf{T}_q^T\mathbf{S}_{\mathbf{P}}^T\mathbf{u}_l\rfloor\delta\phi_p \quad (41)$$

where $\mathbf{D}_{q,p} = \mathbf{D}_{p,q}^T$. Thus, extending the $\delta\phi$ into $\delta\mathbf{T}$, we obtain

$$\mathbf{N}_{kl}^{p,q} = \begin{bmatrix}\lfloor\mathbf{S}_{\mathbf{P}}\mathbf{T}_q\mathbf{D}_{qp}\mathbf{T}_p^T\mathbf{S}_{\mathbf{P}}^T\mathbf{u}_k\rfloor\lfloor\mathbf{u}_l\rfloor & \mathbf{0}_{3\times3} \\ \mathbf{0}_{3\times3} & \mathbf{0}_{3\times3}\end{bmatrix} \quad (42)$$

$$\mathbf{M}_{kl}^{p,q} = \begin{bmatrix}\lfloor\mathbf{u}_k\rfloor\lfloor\mathbf{S}_{\mathbf{P}}\mathbf{T}_p\mathbf{D}_{p,q}\mathbf{T}_q^T\mathbf{S}_{\mathbf{P}}^T\mathbf{u}_l\rfloor & \mathbf{0}_{3\times3} \\ \mathbf{0}_{3\times3} & \mathbf{0}_{3\times3}\end{bmatrix} \quad (43)$$

It can be seen that (36) is quadratic w.r.t. $\delta\mathbf{T}$, so we cast it into the following standard form

$$\mathbf{u}_k^T\mathbf{A}(\delta\mathbf{T})\mathbf{u}_l = \frac{1}{2}\delta\mathbf{T}^T\cdot\mathbf{H}_{kl}\cdot\delta\mathbf{T} + \mathbf{g}_{kl}\cdot\delta\mathbf{T} + r_{kl}, \quad (44)$$

where $\mathbf{g}_{kl}$ is

$$\mathbf{g}_{kl} = \begin{bmatrix}\cdots & \mathbf{g}_{kl}^q & \cdots\end{bmatrix}\in\mathbb{R}^{1\times6M_p} \quad (45)$$

$$\mathbf{g}_{kl}^q = \frac{1}{N}(\bar{\mathbf{w}}_{qk}^T\mathbf{C}_q\mathbf{J}_{\mathbf{w}_{ql}} + \bar{\mathbf{w}}_{ql}^T\mathbf{C}_q\mathbf{J}_{\mathbf{w}_{qk}})$$
$$- \frac{1}{N^2}\sum_{p=1}^{M_p}(\bar{\mathbf{w}}_{pk}^T\mathbf{D}_{p,q}\mathbf{J}_{\mathbf{w}_{ql}} + \bar{\mathbf{w}}_{pl}^T\mathbf{D}_{p,q}\mathbf{J}_{\mathbf{w}_{qk}}) \quad (46)$$

$$= \frac{1}{N}(\mathbf{u}_k^T\mathbf{S}_{\mathbf{P}}\mathbf{T}_q\mathbf{C}_q\mathbf{T}_q^T\mathbf{U}_l^T + \mathbf{u}_l^T\mathbf{S}_{\mathbf{P}}\mathbf{T}_q\mathbf{C}_q\mathbf{T}_q^T\mathbf{U}_k^T)$$
$$- \frac{1}{N^2}\sum_{p=1}^{M_p}(\mathbf{u}_k^T\mathbf{S}_{\mathbf{P}}\mathbf{T}_p\mathbf{C}_{p,q}\mathbf{T}_q^T\mathbf{U}_l^T + \mathbf{u}_l^T\mathbf{S}_{\mathbf{P}}\mathbf{T}_p\mathbf{D}_{p,q}\mathbf{T}_q^T\mathbf{U}_k^T)$$

$$= \frac{1}{N}\mathbf{u}_l^T\mathbf{S}_{\mathbf{P}}(\mathbf{T}_q\mathbf{C}_q - \frac{1}{N}\sum_{p=1}^{M_p}\mathbf{T}_p\mathbf{D}_{p,q})\mathbf{T}_q^T\mathbf{U}_k^T$$
$$+ \frac{1}{N}\mathbf{u}_k^T\mathbf{S}_{\mathbf{P}}(\mathbf{T}_q\mathbf{C}_q - \frac{1}{N}\sum_{p=1}^{M_p}\mathbf{T}_p\mathbf{D}_{p,q})\mathbf{T}_q^T\mathbf{U}_l^T \quad (47)$$

where $\mathbf{D}_{p,q} = \mathbf{C}_p\mathbf{F}\mathbf{C}_q$ and $\mathbf{F} = \begin{bmatrix}\mathbf{0}_{3\times1} & \mathbf{0}_{3\times1} \\ \mathbf{0}_{1\times3} & 1\end{bmatrix}$.

Since $\mathbf{v} = \sum_{j=1}^{M_p}(\mathbf{R}_j\mathbf{v}_j + N_j\mathbf{t}_j)$,

$$\sum_{p=1}^{M_p}\mathbf{T}_p\mathbf{C}_p\mathbf{F} = \sum_{p=1}^{M_p}\begin{bmatrix}\mathbf{R}_p & \mathbf{t}_p \\ \mathbf{0} & 1\end{bmatrix}\begin{bmatrix}\mathbf{P}_p & \mathbf{v}_p \\ \mathbf{v}_p^T & N_p\end{bmatrix}\begin{bmatrix}\mathbf{0}_{3\times1} & \mathbf{0}_{3\times1} \\ \mathbf{0}_{1\times3} & 1\end{bmatrix}$$
$$= \sum_{p=1}^{M_p}\begin{bmatrix}\mathbf{0}_{3\times1} & \mathbf{R}_p\mathbf{v}_p + N_p\mathbf{t}_p \\ \mathbf{0}_{1\times3} & N_p\end{bmatrix} = \begin{bmatrix}\mathbf{0} & \mathbf{v} \\ \mathbf{0} & N_p\end{bmatrix} = \mathbf{C}\mathbf{F}. \quad (48)$$

The common part in (47) is

$$\mathbf{T}_q\mathbf{C}_q - \frac{1}{N}\sum_{p=1}^{M_p}\mathbf{T}_p\mathbf{D}_{p,q} = \mathbf{T}_q\mathbf{C}_q - \frac{1}{N}\underbrace{\sum_{p=1}^{M_p}\mathbf{T}_p\mathbf{C}_p\mathbf{F}}_{\mathbf{C}\mathbf{F}}\mathbf{C}_q$$
$$= (\mathbf{T}_q - \frac{1}{N}\mathbf{C}\mathbf{F})\mathbf{C}_q. \quad (49)$$

Therefore,

$$\mathbf{g}_{kl}^q = \frac{1}{N}\mathbf{u}_l^T\mathbf{S}_{\mathbf{P}}(\mathbf{T}_q - \frac{1}{N}\mathbf{C}\mathbf{F})\mathbf{C}_q\mathbf{T}_q^T\mathbf{U}_k^T$$
$$+ \frac{1}{N}\mathbf{u}_k^T\mathbf{S}_{\mathbf{P}}(\mathbf{T}_q - \frac{1}{N}\mathbf{C}\mathbf{F})\mathbf{C}_q\mathbf{T}_q^T\mathbf{U}_l^T. \quad (50)$$

Additionally, it is seen that if $\mathbf{C}_q = \mathbf{0}_{4\times4}$, $\mathbf{g}_{kl}^q = \mathbf{0}_{1\times6}$, $\forall k,l$.

For $\mathbf{H}_{kl}$ in (44), partition it as

$$\mathbf{H}_{kl} = \begin{bmatrix}\cdots & \vdots & \\ \cdots & \mathbf{H}_{kl}^{p,q} & \cdots \\ & \vdots & \end{bmatrix}\in\mathbb{R}^{6M_p\times6M_p} \quad (51)$$

$$\mathbf{H}_{kl}^{p,q} = -\frac{2}{N^2}\mathbf{J}_{\mathbf{w}_{pk}}^T\mathbf{D}_{p,q}\mathbf{J}_{\mathbf{w}_{ql}} + \mathbb{1}_{p=q}\Big(\frac{2}{N}\mathbf{J}_{\mathbf{w}_{qk}}^T\mathbf{C}_q\mathbf{J}_{\mathbf{w}_{ql}}$$
$$+ \frac{1}{N}\mathbf{Q}_{kl}^q - \frac{1}{N^2}\sum_{\nu=1}^{M_p}(\mathbf{N}_{kl}^{\nu,q} + \mathbf{M}_{kl}^{q,\nu})\Big) \quad (52)$$

where

$$\mathbf{J}^T_{\mathbf{w}_{pk}} \mathbf{D}_{p,q} \mathbf{J}_{\mathbf{w}_{ql}} = \mathbf{U}_k \mathbf{T}_p \mathbf{D}_{p,q} \mathbf{T}_q^T \mathbf{U}_l^T \tag{53}$$

$$\mathbf{J}^T_{\mathbf{w}_{qk}} \mathbf{C}_q \mathbf{J}_{\mathbf{w}_{ql}} = \mathbf{U}_k \mathbf{T}_q \mathbf{C}_q \mathbf{T}_q^T \mathbf{U}_l^T \tag{54}$$

$$\sum_{\nu=1}^{M_p} (\mathbf{N}_{kl}^{\nu,q} + \mathbf{M}_{kl}^{q,\nu}) = \sum_{\nu=1}^{M_p}$$
$$\begin{bmatrix} \lfloor \mathbf{S_P T_q D_{q,\nu} T_\nu^T S_P^T u}_k \rfloor \lfloor \mathbf{u}_l \rfloor + \lfloor \mathbf{u}_k \rfloor \lfloor \mathbf{S_P T_q D_{q,\nu} T_\nu^T S_P^T u}_l \rfloor & \mathbf{0} \\ \mathbf{0} & \mathbf{0} \end{bmatrix} \tag{55}$$

By using similar method in (48), we have

$$\sum_{\nu=1}^{M_p} \mathbf{D}_{q,\nu} \mathbf{T}_\nu^T = \sum_{\nu=1}^{M_p} \mathbf{C}_q \mathbf{F} \mathbf{C}_\nu \mathbf{T}_\nu^T = \mathbf{C}_q \mathbf{F} \mathbf{C} \tag{56}$$

The value of $\mathbf{Q}_{kl}^q$ is in (39). Thus, the matrix $\mathbf{H}_{kl}^{p,q}$ is

$$\mathbf{H}_{kl}^{p,q} = -\frac{2}{N^2} \mathbf{U}_k \mathbf{T}_p \mathbf{D}_{p,q} \mathbf{T}_q^T \mathbf{U}_l^T + \mathbb{1}_{p=q} \Big( \frac{2}{N} \mathbf{U}_k \mathbf{T}_q \mathbf{C}_q \mathbf{T}_q^T \mathbf{U}_l^T$$
$$+ \begin{bmatrix} \mathbf{K}_{kl}^q & \mathbf{0}_{3\times 3} \\ \mathbf{0}_{3\times 3} & \mathbf{0}_{3\times 3} \end{bmatrix} \Big) \tag{57}$$

$$\mathbf{K}_{kl}^q = \frac{1}{N} \mathbf{Q}_{kl}^q - \frac{1}{N^2} \sum_{\nu=1}^{M_p} (\mathbf{N}_{kl}^{\nu,q} + \mathbf{M}_{kl}^{q,\nu})$$
$$= \frac{1}{N} \lfloor \mathbf{S_P T_q C_q (T_q} - \frac{1}{N} \mathbf{CF})^T \mathbf{S_P^T u}_k \rfloor \lfloor \mathbf{u}_l \rfloor$$
$$+ \frac{1}{N} \lfloor \mathbf{u}_k \rfloor \lfloor \mathbf{S_P T_q C_q (T_q} - \frac{1}{N} \mathbf{CF})^T \mathbf{S_P^T u}_l \rfloor, \tag{58}$$

which yields the solution. Additionally, it can be seen that if $\mathbf{C}_q = \mathbf{0}_{4\times 4}$ or $\mathbf{C}_p = \mathbf{0}_{4\times 4}$, we have $\mathbf{D}_{p,q} = \mathbf{C}_p \mathbf{F} \mathbf{C}_q = \mathbf{0}_{4\times 4}$, $\mathbf{K}_{kl}^q = \mathbf{0}_{3\times 3}$ and then $\mathbf{H}_{kl}^{p,q} = \mathbf{0}_{6\times 6}, \forall k, l$. □

## II. Proof of Theorems

### A. Proof of formula (6) and (8)

*Proof.* The variable to be optimized is $\boldsymbol{\pi}_i = (\mathbf{n}_i, \mathbf{q}_i)$ and the cost function is

$$c_i = \min_{\boldsymbol{\pi}_i = (\mathbf{n}_i, \mathbf{q}_i)} \Big( \frac{1}{N_i} \sum_{j=1}^{M_p} \sum_{k=1}^{N_{ij}} \big\| \mathbf{h}_i^T (\mathbf{p}_{ijk} - \mathbf{q}_i) \big\|_2^2 \Big) \tag{59}$$

where $\mathbf{h}_i = \mathbf{n}_i$ for plane feature and $\mathbf{h}_i = (\mathbf{I} - \mathbf{n}_i \mathbf{n}_i^T)$ for edge feature. The dimensions of $\mathbf{h}_i$ may be different for these two features but it has no influence on following derivation.

$$c_i = \min_{\mathbf{n}_i} \min_{\mathbf{q}_i} \Big( \frac{1}{N_i} \sum_{j=1}^{M_p} \sum_{k=1}^{N_{ij}} \big\| \mathbf{h}_i^T (\mathbf{p}_{ijk} - \mathbf{q}_i) \big\|_2^2 \Big) \tag{60}$$

$$= \min_{\mathbf{n}_i} \Big( \min_{\mathbf{q}_i} \Big( \frac{1}{N_i} \sum_{j=1}^{M_p} \sum_{k=1}^{N_{ij}} (\mathbf{p}_{ijk} - \mathbf{q}_i)^T \mathbf{h}_i \mathbf{h}_i^T (\mathbf{p}_{ijk} - \mathbf{q}_i) \Big) \Big).$$

As can be seen, the inner optimization on $\mathbf{q}_i$ is a standard quadratic optimization problem. So, the optimum $\mathbf{q}_i^\star$ can be solved by setting the derivative to zero:

$$2 \sum_{j=1}^{M_p} \sum_{k=1}^{N_{ij}} \mathbf{h}_i \mathbf{h}_i^T (\mathbf{p}_{ijk} - \mathbf{q}_i) = \mathbf{0} \implies$$

$$2 \mathbf{h}_i \mathbf{h}_i^T \Big( \sum_{j=1}^{M_p} \sum_{k=1}^{N_{ij}} \mathbf{p}_{ijk} - N_i \mathbf{q}_i \Big) = \mathbf{0} \tag{61}$$

where $N_i = \sum_{j=1}^{M_p} \sum_{k=1}^{N_{ij}} 1 = \sum_{j=1}^{M_p} N_{ij}$. This equation does not lead to a unique solution of $\mathbf{q}_i$, one particular optimum solution is $\mathbf{q}_i^\star = \frac{1}{N_i} \sum_{j=1}^{M_p} \sum_{k=1}^{N_{ij}} \mathbf{p}_{ijk} = \bar{\mathbf{p}}_i$ as is defined in (7).

Now, substituting the optimum solution $\mathbf{q}_i^\star = \bar{\mathbf{p}}_i$ into (60) leads to:

$$c_i = \min_{\mathbf{n}_i} \Big( \frac{1}{N_i} \sum_{j=1}^{M_p} \sum_{k=1}^{N_{ij}} \big\| \mathbf{h}_i^T (\mathbf{p}_{ijk} - \bar{\mathbf{p}}_i) \big\|_2^2 \Big) \tag{62}$$

To solve for the optimal parameter $\mathbf{n}_i$ in the above optimization problem, we discuss the case of plane and edge features separately, as follows.

*1) Plane feature:* $\mathbf{h}_i = \mathbf{n}_i$.

$$c_i = \min_{\|\mathbf{n}_i\|_2 = 1} \Big( \frac{1}{N_i} \sum_{j=1}^{M_p} \sum_{k=1}^{N_{ij}} \big\| \mathbf{n}_i^T (\mathbf{p}_{ijk} - \bar{\mathbf{p}}_i) \big\|_2^2 \Big)$$

$$= \min_{\|\mathbf{n}_i\|_2 = 1} \Big( \frac{1}{N_i} \sum_{j=1}^{M_p} \sum_{k=1}^{N_{ij}} \big( \mathbf{n}_i^T (\mathbf{p}_{ijk} - \bar{\mathbf{p}}_i)(\mathbf{p}_{ijk} - \bar{\mathbf{p}}_i)^T \mathbf{n}_i \big) \Big)$$

$$= \min_{\mathbf{n}_i} \mathbf{n}_i^T \mathbf{A}_i \mathbf{n}_i, \tag{63}$$

where $\mathbf{A}_i$ is defined in (7) and is a symmetric matrix. Performing Singular Value Decomposition (SVD) of $\mathbf{A}_i$

$$\mathbf{A}_i = \mathbf{U}_i \boldsymbol{\Lambda}_i \mathbf{U}_i^T \tag{64}$$

where

$$\mathbf{U}_i = \begin{bmatrix} \mathbf{u}_1 & \mathbf{u}_2 & \mathbf{u}_3 \end{bmatrix} \quad \boldsymbol{\Lambda}_i = \mathrm{diag}(\lambda_1 \quad \lambda_2 \quad \lambda_3) \tag{65}$$

with $\lambda_1 \geq \lambda_2 \geq \lambda_3$ and $\mathbf{U}_i^T \mathbf{U}_i = \mathbf{I}$.

Denote $\mathbf{m} = \mathbf{U}_i \mathbf{n}_i = \begin{bmatrix} m_1 & m_2 & m_3 \end{bmatrix}^T$, $\|\mathbf{m}\|_2 = \sqrt{\mathbf{n}_i^T \mathbf{U}_i^T \mathbf{U}_i \mathbf{n}_i} = 1$, then (63) reduces to

$$c_i = \min_{\|\mathbf{n}_i\|_2 = 1} (\mathbf{n}_i^T \mathbf{U}_i \boldsymbol{\Lambda}_i \mathbf{U}_i^T \mathbf{n}_i) = \min_{\|\mathbf{m}\|_2 = 1} (\mathbf{m}^T \boldsymbol{\Lambda}_i \mathbf{m})$$

$$= \min_{\|\mathbf{m}\|_2 = 1} (\lambda_1 m_1^2 + \lambda_2 m_2^2 + \lambda_3 m_3^2)$$

$$\geq \min_{\|\mathbf{m}\|_2 = 1} (\lambda_3 m_1^2 + \lambda_3 m_2^2 + \lambda_3 m_3^2) = \lambda_3, \tag{66}$$

where the minimum value $\lambda_3$ is reached when $m_3 = 1$, i.e., $\mathbf{m}^\star = \begin{bmatrix} 0 & 0 & 1 \end{bmatrix}^T$ and $\mathbf{n}_i^\star = \mathbf{U}_i \mathbf{m}^\star = \mathbf{u}_3$.

Therefore, the optimal cost is $\lambda_3(\mathbf{A}_i)$ and the optimum solution is $\mathbf{n}^\star = \mathbf{u}_3(\mathbf{A}_i)$ and $\mathbf{q}^\star = \bar{\mathbf{p}}_i$.

*2) Edge feature:* $\mathbf{h}_i = \mathbf{I} - \mathbf{n}_i\mathbf{n}_i^T$, then $\mathbf{h}_i\mathbf{h}_i^T = \mathbf{h}_i$ and

$$
\begin{aligned}
c_i &= \min_{\mathbf{n}_i} \left( \frac{1}{N_i} \sum_{j=1}^{M_p} \sum_{k=1}^{N_{ij}} \left\| (\mathbf{I} - \mathbf{n}_i\mathbf{n}_i^T)(\mathbf{p}_{ijk} - \bar{\mathbf{p}}_i) \right\|_2^2 \right) \\
&= \min_{\mathbf{n}_i} \left( \frac{1}{N_i} \sum_{j=1}^{M_p} \sum_{k=1}^{N_{ij}} (\mathbf{p}_{ijk} - \bar{\mathbf{p}}_i)^T (\mathbf{I} - \mathbf{n}_i\mathbf{n}_i^T)(\mathbf{p}_{ijk} - \bar{\mathbf{p}}_i) \right) \\
&= \min_{\mathbf{n}_i} \left( \frac{1}{N_i} \sum_{j=1}^{M_p} \sum_{k=1}^{N_{ij}} (\mathbf{p}_{ijk} - \bar{\mathbf{p}}_i)^T (\mathbf{p}_{ijk} - \bar{\mathbf{p}}_i) \right. \\
&\qquad \left. - \mathbf{n}_i^T \left( \frac{1}{N_i} \sum_{j=1}^{M_p} \sum_{k=1}^{N_{ij}} (\mathbf{p}_{ijk} - \bar{\mathbf{p}}_i)^T (\mathbf{p}_{ijk} - \bar{\mathbf{p}}_i) \right) \mathbf{n}_i \right) \\
&= \min_{\mathbf{n}_i} (\text{trace}(\mathbf{A}_i) - \mathbf{n}_i^T \mathbf{A}_i \mathbf{n}_i) \\
&= \lambda_1 + \lambda_2 + \lambda_3 - \underbrace{\max_{\mathbf{n}_i} \mathbf{n}_i^T \mathbf{A}_i \mathbf{n}_i}_{=\lambda_1, \text{ when } , \mathbf{n}_i^\star = \mathbf{u}_1} \\
&= \lambda_2 + \lambda_3.
\end{aligned}
\tag{67}
$$

Therefore, the optimal cost is $\lambda_2(\mathbf{A}_i) + \lambda_3(\mathbf{A}_i)$ and the optimum solution is $\mathbf{n}^\star = \mathbf{u}_1(\mathbf{A}_i)$ and $\mathbf{q}^\star = \bar{\mathbf{p}}_i$. $\quad\square$

### B. Proof of Theorem 1

For the point collections $\mathcal{C} = \{\mathbf{p}_k \in \mathbb{R}^3 | k = 1, \cdots, n\}$, its point cluster is

$$
\Re(\mathcal{C}) = \sum_{k=1}^{n} \begin{bmatrix} \mathbf{p}_k \\ 1 \end{bmatrix} \begin{bmatrix} \mathbf{p}_k^T & 1 \end{bmatrix}
\tag{68}
$$

The rigid transformation of $\mathcal{C}$ by pose $\mathbf{T} = \begin{bmatrix} \mathbf{R} & \mathbf{t} \\ \mathbf{0}_{3\times 1} & 1 \end{bmatrix}$ is

$$
\mathbf{T} \circ \mathcal{C} = \{\mathbf{R}\mathbf{p}_k + \mathbf{t} \in \mathbb{R}^3 | k = 1, \cdots, n\}
\tag{69}
$$

whose point cluster is

$$
\begin{aligned}
\Re(\mathbf{T} \circ \mathcal{C}) &= \sum_{k=1}^{n} \begin{bmatrix} \mathbf{R}\mathbf{p}_k + \mathbf{t} \\ 1 \end{bmatrix} \begin{bmatrix} (\mathbf{R}\mathbf{p}_k + \mathbf{t})^T & 1 \end{bmatrix} \\
&= \sum_{k=1}^{n} \mathbf{T} \begin{bmatrix} \mathbf{p}_k \\ 1 \end{bmatrix} \begin{bmatrix} \mathbf{p}_k^T & 1 \end{bmatrix} \mathbf{T}^T = \mathbf{T} \Re(\mathcal{C}) \mathbf{T}^T
\end{aligned}
\tag{70}
$$

which yields the solution. $\quad\square$

### C. Proof of Theorem 2

For two point collections $\mathcal{C}_1 = \{\mathbf{p}_k^1 \in \mathbb{R}^3 | k = 1, \cdots, n_1\}$ and $\mathcal{C}_2 = \{\mathbf{p}_k^2 \in \mathbb{R}^3 | k = 1, \cdots, n_2\}$, their point clusters are

$$
\Re(\mathcal{C}_i) = \sum_{k=1}^{n_l} \begin{bmatrix} \mathbf{p}_k^l \\ 1 \end{bmatrix} \begin{bmatrix} (\mathbf{p}_k^l)^T & 1 \end{bmatrix} \quad l = 1, 2
\tag{71}
$$

The merge of $\mathcal{C}_1$ and $\mathcal{C}_2$ is

$$
\mathcal{C}_1 \oplus \mathcal{C}_2 = \{\mathbf{p}_k^l \in \mathbb{R}^3 | l = 1, 2; k = 1, \cdots, n_l\}
\tag{72}
$$

whose point cluster is

$$
\begin{aligned}
\Re(\mathcal{C}_1 \oplus \mathcal{C}_2) &= \sum_{k=1}^{n_1} \begin{bmatrix} \mathbf{p}_k^1 \\ 1 \end{bmatrix} \begin{bmatrix} (\mathbf{p}_k^1)^T & 1 \end{bmatrix} + \sum_{k=1}^{n_2} \begin{bmatrix} \mathbf{p}_k^2 \\ 1 \end{bmatrix} \begin{bmatrix} (\mathbf{p}_k^2)^T & 1 \end{bmatrix} \\
&= \Re(\mathcal{C}_1) + \Re(\mathcal{C}_2)
\end{aligned}
\tag{73}
$$

which yields the solution. $\quad\square$

### D. Proof of Theorem 3

Let $\mathbf{T}_0 = \begin{bmatrix} \mathbf{R}_0 & \mathbf{t}_0 \\ 0 & 1 \end{bmatrix}$ and $\bar{\mathbf{C}} = \mathbf{T}_0 \mathbf{C} \mathbf{T}_0^T = \begin{bmatrix} \bar{\mathbf{P}} & \bar{\mathbf{v}} \\ \bar{\mathbf{v}}^T & N \end{bmatrix}$, then

$$
\bar{\mathbf{P}} = \mathbf{R}_0\mathbf{P}_0\mathbf{R}_0^T + \mathbf{R}_0\mathbf{v}_0\mathbf{t}_0^T + \mathbf{t}_0\mathbf{v}_0^T\mathbf{R}_0^T + N\mathbf{t}_0\mathbf{t}_0^T,
\tag{74}
$$

$$
\bar{\mathbf{v}} = \mathbf{R}_0\mathbf{v}_0 + N\mathbf{t}_0,
\tag{75}
$$

$$
\mathbf{A}\left( \mathbf{T}_0\mathbf{C}\mathbf{T}_0^T \right) = \frac{1}{N}\bar{\mathbf{P}} - \frac{1}{N^2}\bar{\mathbf{v}}\bar{\mathbf{v}}^T = \mathbf{R}_0\mathbf{A}(\mathbf{C})\mathbf{R}_0^T.
\tag{76}
$$

Since $\mathbf{A}\left( \mathbf{T}_0\mathbf{C}\mathbf{T}_0^T \right)$ and $\mathbf{A}(\mathbf{C})$ are similar by transformation $\mathbf{R}_0$, they have the same eigenvalue. $\quad\square$

### E. Proof of Theorem 4

Denote $\lambda_l$ the $l$-th largest eigenvalue of $\mathbf{A}$ and $\mathbf{u}_l$ the corresponding vector, i.e., $\lambda_l\mathbf{u}_l = \mathbf{A}\mathbf{u}_l$. Since $\mathbf{A}$ is symmetric, $\mathbf{u}_l$ is an orthonormal vector. Multiplying both sides of $\lambda_l\mathbf{u}_l = \mathbf{A}\mathbf{u}_l$ by $\mathbf{u}_l^T$ leads to

$$
\lambda_l = \mathbf{u}_l^T \mathbf{A} \mathbf{u}_l.
\tag{77}
$$

Note that in the above equation, $\lambda_l$, $\mathbf{u}_l$ and $\mathbf{A}_l$ all depend on the pose $\mathbf{T}$. To avoid any confusion, we write them as explicit functions of $\mathbf{T}$:

$$
\lambda_l(\mathbf{T}) = \mathbf{u}_l^T(\mathbf{T})\mathbf{A}(\mathbf{T})\mathbf{u}_l(\mathbf{T}).
\tag{78}
$$

Parameterizing the pose $\mathbf{T}$ by $\delta\mathbf{T}$ leads to

$$
\lambda_l(\delta\mathbf{T}) = \mathbf{u}_l^T(\delta\mathbf{T})\mathbf{A}(\delta\mathbf{T})\mathbf{u}_l(\delta\mathbf{T}).
\tag{79}
$$

From the first conclusion in Lemma 1, we know that for a vector $\mathbf{x} = \begin{bmatrix} x_1 & \cdots & x_m \end{bmatrix} \in \mathbb{R}^m$ that the matrix $\mathbf{A}$ depends on, we have

$$
\frac{\partial\lambda_l(\mathbf{x})}{\partial x_i} = \frac{\partial \left( \mathbf{u}_l^T(\mathbf{x})\mathbf{A}(\mathbf{x})\mathbf{u}_l(\mathbf{x}) \right)}{\partial x_i} = \mathbf{u}_l^T(\mathbf{x})\frac{\partial\mathbf{A}(\mathbf{x})}{\partial x_i}\mathbf{u}_l(\mathbf{x}),
$$
$$
\forall i = 1, \cdots, m, x_i \in \mathbb{R}
\tag{80}
$$

To avoid the use of a tensor when applying this derivative to the complete parameter vector $\mathbf{x}$, we fix the vector $\mathbf{u}_l(\mathbf{x})$ at its current value and lump it with the matrix $\mathbf{A}(\mathbf{x})$ within the derivative, i.e.,

$$
\mathbf{u}_l^T(\mathbf{x})\frac{\partial\mathbf{A}(\mathbf{x})}{\partial x_i}\mathbf{u}_l(\mathbf{x}) = \frac{\partial\mathbf{u}_l^T\mathbf{A}(\mathbf{x})\mathbf{u}_l}{\partial x_i}
\tag{81}
$$

where on the right hand side, $\mathbf{u}_l$ is fixed (so we remove its argument $\mathbf{x}$) and the derivative is only applied on the component $\mathbf{A}(\mathbf{x})$ (so we keep its argument $\mathbf{x}$).

Now, the input parameter is the poses parameterized by $\delta\mathbf{T}$, setting $\mathbf{x}$ to $\delta\mathbf{T}$ in (81) leads to

$$
\frac{\partial\lambda_l(\delta\mathbf{T})}{\partial\delta\mathbf{T}} = \frac{\partial\mathbf{u}_l^T\mathbf{A}(\delta\mathbf{T})\mathbf{u}_l}{\partial\delta\mathbf{T}}.
\tag{82}
$$

Recalling the results from Lemma 2 with $k = l$, we obtain:

$$
\mathbf{u}_l^T\mathbf{A}(\delta\mathbf{T})\mathbf{u}_l = \frac{1}{2}\delta\mathbf{T}^T \cdot \mathbf{H}_{ll} \cdot \delta\mathbf{T} + \mathbf{g}_{ll} \cdot \delta\mathbf{T} + r_{ll}.
\tag{83}
$$

Therefore, the first order derivative of $\lambda_l(\mathbf{T})$ w.r.t. $\mathbf{T}$ is

$$
\mathbf{J} \triangleq \frac{\partial\lambda_l(\delta\mathbf{T})}{\partial\delta\mathbf{T}} = \frac{\partial\mathbf{u}_l^T\mathbf{A}(\delta\mathbf{T})\mathbf{u}_l}{\partial\delta\mathbf{T}} = \mathbf{g}_{ll}.
\tag{84}
$$

Now, we derive the second order derivative. From (80), we have,

$$\frac{\partial \lambda_l(\mathbf{x})}{\partial x_i} = \mathbf{u}_l^T(\mathbf{x})\frac{\partial \mathbf{A}(\mathbf{x})}{\partial x_i}\mathbf{u}_l(\mathbf{x}), \forall x_i \in \mathbb{R}, \quad (85)$$

Differentiating it w.r.t. the second parameter $x_j \in \mathbb{R}$ leads to

$$\frac{\partial^2 \lambda_l(\mathbf{x})}{\partial x_j \partial x_i} = \frac{\partial}{\partial x_j}\left(\mathbf{u}_l^T(\mathbf{x})\frac{\partial \mathbf{A}(\mathbf{x})}{\partial x_i}\mathbf{u}_l(\mathbf{x})\right)$$
$$= \left(\frac{\partial \mathbf{u}_l(\mathbf{x})}{\partial x_j}\right)^T\left(\frac{\partial \mathbf{A}(\mathbf{x})\mathbf{u}_l}{\partial x_i}\right) + \frac{\partial}{\partial x_j}\left(\frac{\partial \mathbf{u}_l^T \mathbf{A}(\mathbf{x})\mathbf{u}_l}{\partial x_i}\right)$$
$$+ \left(\frac{\partial \mathbf{A}(\mathbf{x})\mathbf{u}_l}{\partial x_i}\right)^T\left(\frac{\partial \mathbf{u}_l(\mathbf{x})}{\partial x_j}\right) \quad (86)$$

Applying the above results to each elements $x_i, x_j$ leads to

$$\frac{\partial^2 \lambda_l(\mathbf{x})}{\partial \mathbf{x}^2} = \left(\frac{\partial \mathbf{u}_l(\mathbf{x})}{\partial \mathbf{x}}\right)^T\left(\frac{\partial \mathbf{A}(\mathbf{x})\mathbf{u}_l}{\partial \mathbf{x}}\right)$$
$$+ \frac{\partial}{\partial \mathbf{x}}\left(\frac{\partial \mathbf{u}_l^T \mathbf{A}(\mathbf{x})\mathbf{u}_l}{\partial \mathbf{x}}\right) + \left(\frac{\partial \mathbf{A}(\mathbf{x})\mathbf{u}_l}{\partial \mathbf{x}}\right)^T\left(\frac{\partial \mathbf{u}_l(\mathbf{x})}{\partial \mathbf{x}}\right). \quad (87)$$

To compute $\frac{\partial \mathbf{u}_l(\mathbf{x})}{\partial \mathbf{x}}$, we apply the second conclusion in Lemma 1 to all components of $\mathbf{x}$ and use the notation trick similar to (81):

$$\frac{\partial \mathbf{u}_l(\mathbf{x})}{\partial \mathbf{x}} = \sum_{k=1, k \neq l}^{3}\frac{1}{\lambda_l - \lambda_k}\mathbf{u}_k\mathbf{u}_k^T\frac{\partial \mathbf{A}(\mathbf{x})\mathbf{u}_l}{\partial \mathbf{x}} \quad (88)$$

Now, the input parameter is the pose vector parameterized by $\delta \mathbf{T}$, substituting $\mathbf{x} = \delta \mathbf{T}$ into (87) leads to

$$\frac{\partial^2 \lambda_l(\delta \mathbf{T})}{\partial \delta \mathbf{T}^2} = \underbrace{\left(\frac{\partial \mathbf{u}_l(\delta \mathbf{T})}{\partial \delta \mathbf{T}}\right)^T\left(\frac{\partial \mathbf{A}(\delta \mathbf{T})\mathbf{u}_l}{\partial \delta \mathbf{T}}\right)}_{\mathbf{H}_l}$$
$$+ \underbrace{\frac{\partial^2 \mathbf{u}_l^T \mathbf{A}(\delta \mathbf{T})\mathbf{u}_l}{\partial \delta \mathbf{T}^2}}_{\mathbf{H}_{ll}} + \underbrace{\left(\frac{\partial \mathbf{A}(\delta \mathbf{T})\mathbf{u}_l}{\partial \delta \mathbf{T}}\right)^T\left(\frac{\partial \mathbf{u}_l(\delta \mathbf{T})}{\partial \delta \mathbf{T}}\right)}_{\mathbf{H}_l^T}, \quad (89)$$

where term $\mathbf{H}_{ll}$ is from (83). To obtain the term $\mathbf{H}_l$, we substitute $\mathbf{x} = \delta \mathbf{T}$ into (88):

$$\frac{\partial \mathbf{u}_l(\delta \mathbf{T})}{\partial \delta \mathbf{T}} = \sum_{k=1, k \neq l}^{3}\frac{1}{\lambda_l - \lambda_k}\mathbf{u}_k\mathbf{u}_k^T\frac{\partial \mathbf{A}(\delta \mathbf{T})\mathbf{u}_l}{\partial \delta \mathbf{T}} \quad (90)$$

$$= \sum_{k=1, k \neq l}^{3}\frac{1}{\lambda_l - \lambda_k}\mathbf{u}_k\frac{\partial \mathbf{u}_k^T \mathbf{A}(\delta \mathbf{T})\mathbf{u}_l}{\partial \delta \mathbf{T}}, \quad (91)$$

$$= \sum_{k=1, k \neq l}^{3}\frac{1}{\lambda_l - \lambda_k}\mathbf{u}_k\mathbf{g}_{kl}. \quad (92)$$

Then,

$$\mathbf{H}_l = \left(\frac{\partial \mathbf{A}(\delta \mathbf{T})\mathbf{u}_l}{\partial \delta \mathbf{T}}\right)^T\left(\frac{\partial \mathbf{u}_l(\delta \mathbf{T})}{\partial \delta \mathbf{T}}\right) \quad (93)$$

$$= \left(\frac{\partial \mathbf{A}(\delta \mathbf{T})\mathbf{u}_l}{\partial \delta \mathbf{T}}\right)^T\sum_{k=1, k \neq l}^{3}\frac{1}{\lambda_l - \lambda_k}\mathbf{u}_k\mathbf{g}_{kl} \quad (94)$$

$$= \sum_{k=1, k \neq l}^{3}\left(\frac{\partial \mathbf{u}_k^T \mathbf{A}(\delta \mathbf{T})\mathbf{u}_l}{\partial \delta \mathbf{T}}\right)^T\frac{1}{\lambda_l - \lambda_k}\mathbf{g}_{kl} \quad (95)$$

$$= \sum_{k=1, k \neq l}^{3}\frac{1}{\lambda_l - \lambda_k}\mathbf{g}_{kl}^T\mathbf{g}_{kl}. \quad (96)$$

Therefore, the second order derivative of $\lambda_l(\mathbf{T})$ w.r.t. $\mathbf{T}$ is

$$\mathbf{H} \triangleq \frac{\partial^2 \lambda_l(\delta \mathbf{T})}{\partial \delta \mathbf{T}^2} = \mathbf{H}_l + \mathbf{H}_{ll} + \mathbf{H}_l^T$$
$$= \mathbf{H}_{ll} + \sum_{k=1, k \neq l}^{3}\frac{2}{\lambda_l - \lambda_k}\mathbf{g}_{kl}^T\mathbf{g}_{kl}. \quad (97)$$

Next, we show that

$$\mathbf{J}\delta \mathbf{T} = 0, \delta \mathbf{T}^T\mathbf{H}\delta \mathbf{T} = 0, \delta \mathbf{T} = \begin{bmatrix}\mathbf{w} \\ \vdots \\ \mathbf{w}\end{bmatrix}, \forall \mathbf{w} \in \mathbb{R}^6 \quad (98)$$

From (84), we can obtain

$$\mathbf{J}\delta \mathbf{T} = \mathbf{g}_{ll}\delta \mathbf{T} = \sum_{q=1}^{M_p}\mathbf{g}_{ll}^q\mathbf{w} \quad (99)$$

where, from (14), we know

$$\mathbf{g}_{ll}^q = \frac{2}{N}\mathbf{u}_l^T\mathbf{S}_\mathbf{P}(\mathbf{T}_q - \frac{1}{N}\mathbf{C}\mathbf{F})\mathbf{C}_q\mathbf{T}_q^T\mathbf{U}_l^T \quad (100)$$

$$\mathbf{U}_l = \begin{bmatrix}-\lfloor\mathbf{u}_l\rfloor & \mathbf{0}_{3\times 1} \\ \mathbf{0}_{3\times 3} & \mathbf{u}_l\end{bmatrix} \quad \mathbf{F} = \begin{bmatrix}\mathbf{0}_{3\times 3} & \mathbf{0}_{3\times 1} \\ \mathbf{0}_{1\times 3} & 1\end{bmatrix} \quad (101)$$

$$\mathbf{S}_\mathbf{P} = \begin{bmatrix}\mathbf{I}_{3\times 3} & \mathbf{0}_{3\times 1}\end{bmatrix} \quad \mathbf{C} = \sum_{j=1}^{M_p}\mathbf{T}_j\mathbf{C}_j\mathbf{T}_j^T \quad (102)$$

Thus,

$$\sum_{q=1}^{M_p}\mathbf{g}_{ll}^q\mathbf{w} = \frac{2}{N}\sum_{q=1}^{M_p}\mathbf{u}_l^T\mathbf{S}_\mathbf{P}(\mathbf{T}_q - \frac{1}{N}\mathbf{C}\mathbf{F})\mathbf{C}_q\mathbf{T}_q^T\mathbf{U}_l^T\mathbf{w}$$
$$= \frac{2}{N}\mathbf{u}_l^T\mathbf{S}_\mathbf{P}\left(\sum_{q=1}^{M_p}(\mathbf{T}_q\mathbf{C}_q\mathbf{T}_q^T - \frac{1}{N}\mathbf{C}\mathbf{F}\mathbf{C}_q\mathbf{T}_q^T)\right)\mathbf{U}_l^T\mathbf{w}.$$

Since $\sum_{q=1}^{M_p}\mathbf{T}_q\mathbf{C}_q\mathbf{T}_q^T = \mathbf{C}$ and $\sum_{q=1}^{M_p}\mathbf{F}\mathbf{C}_q\mathbf{T}_q^T = \mathbf{F}\mathbf{C}$ from equation (48), we have

$$\sum_{q=1}^{M_p}\mathbf{g}_{ll}^q\mathbf{w} = \frac{2}{N}\mathbf{u}_l^T\mathbf{S}_\mathbf{P}(\mathbf{C} - \frac{1}{N}\mathbf{C}\mathbf{F}\mathbf{C})\mathbf{U}_l^T\mathbf{w}.$$

Recalling $\mathbf{A} = \frac{1}{N}\mathbf{P} - \frac{1}{N^2}\mathbf{v}\mathbf{v}^T$,

$$\mathbf{C} - \frac{1}{N}\mathbf{C}\mathbf{F}\mathbf{C}$$
$$= \begin{bmatrix}\mathbf{P} & \mathbf{v} \\ \mathbf{v}^T & N\end{bmatrix} - \frac{1}{N}\begin{bmatrix}\mathbf{P} & \mathbf{v} \\ \mathbf{v}^T & N\end{bmatrix}\begin{bmatrix}\mathbf{0}_{3\times 3} & \mathbf{0}_{3\times 1} \\ \mathbf{0}_{1\times 3} & 1\end{bmatrix}\begin{bmatrix}\mathbf{P} & \mathbf{v} \\ \mathbf{v}^T & N\end{bmatrix}$$
$$= \begin{bmatrix}\mathbf{P} & \mathbf{v} \\ \mathbf{v}^T & N\end{bmatrix} - \begin{bmatrix}\frac{1}{N}\mathbf{v}\mathbf{v}^T & \mathbf{v} \\ \mathbf{v}^T & N\end{bmatrix} = \begin{bmatrix}N\mathbf{A} & \mathbf{0}_{3\times 1} \\ \mathbf{0}_{1\times 3} & 0\end{bmatrix} \quad (103)$$

Thus,

$$\sum_{q=1}^{M_p} \mathbf{g}_{ll}^q \mathbf{w} = \frac{2}{N} \mathbf{u}_l^T \mathbf{S_P} \begin{bmatrix} N\mathbf{A} & \mathbf{0}_{3\times 1} \\ \mathbf{0}_{1\times 3} & 0 \end{bmatrix} \mathbf{U}_l^T \mathbf{w} \tag{104}$$

$$= 2 \begin{bmatrix} \mathbf{u}_l^T \mathbf{A} & \mathbf{0}_{3\times 1} \end{bmatrix} \begin{bmatrix} -\lfloor \mathbf{u}_l \rfloor & \mathbf{0}_{3\times 1} \\ \mathbf{0}_{3\times 3} & \mathbf{u}_l \end{bmatrix}^T \mathbf{w} \tag{105}$$

$$= 2 \begin{bmatrix} \mathbf{u}_l^T \mathbf{A} \lfloor \mathbf{u}_l \rfloor & \mathbf{0}_{3\times 1} \end{bmatrix} \mathbf{w} \tag{106}$$

Since $\mathbf{Au}_l = \lambda_l \mathbf{u}_l$,

$$\mathbf{u}_l^T \mathbf{A} \lfloor \mathbf{u}_l \rfloor = \lambda_l \mathbf{u}_l^T \lfloor \mathbf{u}_l \rfloor = \mathbf{0}_{1\times 3}. \tag{107}$$

Therefore,

$$\mathbf{J}\delta\mathbf{T} = \sum_{q=1}^{M_p} \mathbf{g}_{ll}^q \mathbf{w} = 0 \tag{108}$$

For the proof of $\delta\mathbf{T}^T \mathbf{H}\delta\mathbf{T} = 0$, from (97)

$$\delta\mathbf{T}^T \mathbf{H}\delta\mathbf{T} = \delta\mathbf{T}^T (\mathbf{H}_{ll} + \sum_{k=1, k\neq l}^{3} \frac{2}{\lambda_l - \lambda_k} \mathbf{g}_{kl}^T \mathbf{g}_{kl})\delta\mathbf{T}$$

$$= \mathbf{w}^T \sum_{p=1}^{M_p} \sum_{q=1}^{M_p} \left( \mathbf{H}_{ll}^{p,q} + \sum_{k=1, k\neq l}^{3} \frac{2}{\lambda_l - \lambda_k} (\mathbf{g}_{kl}^p)^T \mathbf{g}_{kl}^q \right) \mathbf{w} \tag{109}$$

where

$$\mathbf{H}_{ll}^{p,q} = -\frac{2}{N^2} \mathbf{U}_l \mathbf{T}_p \mathbf{D}_{p,q} \mathbf{T}_q^T \mathbf{U}_l^T + \mathbb{1}_{p=q} \left( \frac{2}{N} \mathbf{U}_l \mathbf{T}_q \mathbf{C}_q \mathbf{T}_q^T \mathbf{U}_l^T \right.$$
$$\left. + \begin{bmatrix} \mathbf{K}_{ll}^q & \mathbf{0}_{3\times 3} \\ \mathbf{0}_{3\times 3} & \mathbf{0}_{3\times 3} \end{bmatrix} \right) \tag{110}$$

$$\mathbf{K}_{ll}^q = \frac{1}{N} \lfloor \mathbf{S_P} \mathbf{T}_q \mathbf{C}_q (\mathbf{T}_q - \frac{1}{N}\mathbf{CF})^T \mathbf{S}_\mathbf{P}^T \mathbf{u}_l \rfloor \lfloor \mathbf{u}_l \rfloor$$
$$+ \frac{1}{N} \lfloor \mathbf{u}_l \rfloor \lfloor \mathbf{S_P} \mathbf{T}_q \mathbf{C}_q (\mathbf{T}_q - \frac{1}{N}\mathbf{CF})^T \mathbf{S}_\mathbf{P}^T \mathbf{u}_l \rfloor, \tag{111}$$

$$\mathbf{g}_{kl}^q = \frac{1}{N} \mathbf{u}_l^T \mathbf{S_P} (\mathbf{T}_q - \frac{1}{N}\mathbf{CF}) \mathbf{C}_q \mathbf{T}_q^T \mathbf{U}_l^T$$
$$+ \frac{1}{N} \mathbf{u}_k^T \mathbf{S_P} (\mathbf{T}_q - \frac{1}{N}\mathbf{CF}) \mathbf{C}_q \mathbf{T}_q^T \mathbf{U}_l^T \tag{112}$$

We will divide (109) into two parts to discuss. For the first part,

$$\sum_{p=1}^{M_p} \sum_{q=1}^{M_p} \mathbf{H}_{ll}^{p,q} = -\frac{2}{N^2} \sum_{p=1}^{M_p} \sum_{q=1}^{M_p} \mathbf{U}_l \mathbf{T}_p \mathbf{D}_{p,q} \mathbf{T}_q^T \mathbf{U}_l^T$$
$$+ \frac{2}{N} \sum_{q=1}^{M_p} \mathbf{U}_l \mathbf{T}_q \mathbf{C}_q \mathbf{T}_q^T \mathbf{U}_l^T + \begin{bmatrix} \sum_{q=1}^{M_p} \mathbf{K}_{ll}^q & \mathbf{0}_{3\times 3} \\ \mathbf{0}_{3\times 3} & \mathbf{0}_{3\times 3} \end{bmatrix} \tag{113}$$

Since $\sum_{q=1}^{M_p} \mathbf{T}_q \mathbf{C}_q \mathbf{T}_q^T = \mathbf{C}$, $\sum_{q=1}^{M_p} \mathbf{FC}_q \mathbf{T}_q^T = \mathbf{FC}$ and $\sum_{q=1}^{M_p} \mathbf{T}_q \mathbf{C}_q \mathbf{F} = \mathbf{CF}$ from (48), we have

$$\sum_{p=1}^{M_p} \sum_{q=1}^{M_p} \mathbf{T}_p \mathbf{D}_{p,q} \mathbf{T}_q^T = \sum_{p=1}^{M_p} \sum_{q=1}^{M_p} \mathbf{T}_p \mathbf{C}_p \mathbf{FC}_q \mathbf{T}_q^T$$
$$= \sum_{p=1}^{M_p} \mathbf{T}_p \mathbf{C}_p \sum_{q=1}^{M_p} \mathbf{FC}_q \mathbf{T}_q^T = \sum_{p=1}^{M_p} \mathbf{T}_p \mathbf{C}_p \mathbf{FC} = \mathbf{CFC}, \tag{114}$$

and

$$\sum_{q=1}^{M_p} \mathbf{K}_{ll}^q = \frac{1}{N} \sum_{q=1}^{M_p} \lfloor \mathbf{S_P} \mathbf{T}_q \mathbf{C}_q (\mathbf{T}_q - \frac{1}{N}\mathbf{CF})^T \mathbf{S}_\mathbf{P}^T \mathbf{u}_l \rfloor \lfloor \mathbf{u}_l \rfloor$$
$$+ \frac{1}{N} \sum_{q=1}^{M_p} \lfloor \mathbf{u}_l \rfloor \lfloor \mathbf{S_P} \mathbf{T}_q \mathbf{C}_q (\mathbf{T}_q - \frac{1}{N}\mathbf{CF})^T \mathbf{S}_\mathbf{P}^T \mathbf{u}_l \rfloor$$
$$= \frac{1}{N} \lfloor \mathbf{S_P} (\mathbf{C} - \frac{1}{N}\mathbf{CFC}) \mathbf{S}_\mathbf{P}^T \mathbf{u}_l \rfloor \lfloor \mathbf{u}_l \rfloor$$
$$+ \frac{1}{N} \lfloor \mathbf{u}_l \rfloor \lfloor \mathbf{S_P} (\mathbf{C} - \frac{1}{N}\mathbf{CFC}) \mathbf{S}_\mathbf{P}^T \mathbf{u}_l \rfloor. \tag{115}$$

Then, from (103) and $\mathbf{Au}_l = \lambda_l \mathbf{u}_l$,

$$\sum_{q=1}^{M_p} \mathbf{K}_{ll}^q = \lfloor \mathbf{Au}_l \rfloor \lfloor \mathbf{u}_l \rfloor + \lfloor \mathbf{u}_l \rfloor \lfloor \mathbf{Au}_l \rfloor = 2\lambda_l \lfloor \mathbf{u}_l \rfloor^2. \tag{116}$$

Now, substituting the results in (114) and (116) into (113):

$$\sum_{p=1}^{M_p} \sum_{q=1}^{M_p} \mathbf{H}_{ll}^{p,q} = -\frac{2}{N^2} \mathbf{U}_l \mathbf{CFC} \mathbf{U}_l^T + \frac{2}{N} \mathbf{U}_l \mathbf{C} \mathbf{U}_l^T$$
$$+ \begin{bmatrix} 2\lambda_l \lfloor \mathbf{u}_l \rfloor^2 & \mathbf{0} \\ \mathbf{0} & \mathbf{0} \end{bmatrix} = 2\mathbf{U}_l \begin{bmatrix} \mathbf{A} & \mathbf{0} \\ \mathbf{0} & 0 \end{bmatrix} \mathbf{U}_l^T + \begin{bmatrix} 2\lambda_l \lfloor \mathbf{u}_l \rfloor^2 & \mathbf{0} \\ \mathbf{0} & \mathbf{0} \end{bmatrix}$$
$$= 2 \begin{bmatrix} \lambda_l \lfloor \mathbf{u}_l \rfloor^2 - \lfloor \mathbf{u}_l \rfloor \mathbf{A} \lfloor \mathbf{u}_l \rfloor & \mathbf{0} \\ \mathbf{0} & \mathbf{0} \end{bmatrix}. \tag{117}$$

For the second part in (109),

$$\sum_{p=1}^{M_p} \sum_{q=1}^{M_p} \left( \sum_{k=1, k\neq l}^{3} \frac{2}{\lambda_l - \lambda_k} (\mathbf{g}_{kl}^p)^T \mathbf{g}_{kl}^q \right)$$
$$= \sum_{k=1, k\neq l}^{3} \frac{2}{\lambda_l - \lambda_k} \left( \sum_{p=1}^{M_p} \mathbf{g}_{kl}^p \right)^T \left( \sum_{q=1}^{M_p} \mathbf{g}_{kl}^q \right) \tag{118}$$

where

$$\sum_{q=1}^{M_p} \mathbf{g}_{kl}^q = \frac{1}{N} \mathbf{u}_l^T \mathbf{S_P} \sum_{q=1}^{M_p} (\mathbf{T}_q \mathbf{C}_q \mathbf{T}_q^T - \frac{1}{N}\mathbf{CFC}_q \mathbf{T}_q^T) \mathbf{U}_k^T$$
$$+ \frac{1}{N} \mathbf{u}_k^T \mathbf{S_P} \sum_{q=1}^{M_p} (\mathbf{T}_q \mathbf{C}_q \mathbf{T}_q^T - \frac{1}{N}\mathbf{CFC}_q \mathbf{T}_q^T) \mathbf{U}_l^T$$
$$= \frac{1}{N} \mathbf{u}_l^T \mathbf{S_P} (\mathbf{C} - \frac{1}{N}\mathbf{CFC}) \mathbf{U}_k^T$$
$$+ \frac{1}{N} \mathbf{u}_k^T \mathbf{S_P} (\mathbf{C} - \frac{1}{N}\mathbf{CFC}) \mathbf{U}_l^T$$
$$= \mathbf{u}_l^T \begin{bmatrix} \mathbf{A} & \mathbf{0}_{3\times 1} \end{bmatrix} \mathbf{U}_k^T + \mathbf{u}_k^T \begin{bmatrix} \mathbf{A} & \mathbf{0}_{3\times 1} \end{bmatrix} \mathbf{U}_l^T$$
$$= \begin{bmatrix} \mathbf{u}_l^T \mathbf{A} \lfloor \mathbf{u}_k \rfloor & \mathbf{0}_{3\times 1} \end{bmatrix} + \begin{bmatrix} \mathbf{u}_k^T \mathbf{A} \lfloor \mathbf{u}_l \rfloor & \mathbf{0}_{3\times 1} \end{bmatrix}$$
$$= \begin{bmatrix} \lambda_l \mathbf{u}_l^T \lfloor \mathbf{u}_k \rfloor + \lambda_k \mathbf{u}_k^T \lfloor \mathbf{u}_l \rfloor & \mathbf{0}_{3\times 1} \end{bmatrix} \tag{119}$$

Therefore,

$$\sum_{k=1, k\neq l}^{3} \frac{2}{\lambda_l - \lambda_k} \left( \sum_{p=1}^{M_p} \mathbf{g}_{kl}^p \right)^T \left( \sum_{q=1}^{M_p} \mathbf{g}_{kl}^q \right)$$
$$= \sum_{k=1, k\neq l}^{3} \frac{2}{\lambda_l - \lambda_k} \begin{bmatrix} \mathbf{r}_{kl} & \mathbf{0}_{3\times 1} \\ \mathbf{0}_{1\times 3} & 0 \end{bmatrix} \tag{120}$$

where due to $\lfloor \mathbf{u}_k \rfloor \mathbf{u}_l = -\lfloor \mathbf{u}_l \rfloor \mathbf{u}_k$ and $\mathbf{u}_k^T \lfloor \mathbf{u}_l \rfloor = -\mathbf{u}_l^T \lfloor \mathbf{u}_k \rfloor$

$$
\begin{aligned}
\mathbf{r}_{kl} &= \left(\lambda_l \mathbf{u}_l^T \lfloor \mathbf{u}_k \rfloor + \lambda_k \mathbf{u}_k^T \lfloor \mathbf{u}_l \rfloor\right)\right)^T \left(\lambda_l \mathbf{u}_l^T \lfloor \mathbf{u}_k \rfloor + \lambda_k \mathbf{u}_k^T \lfloor \mathbf{u}_l \rfloor\right)\right) \\
&= -\lambda_l^2 \lfloor \mathbf{u}_k \rfloor \mathbf{u}_l \mathbf{u}_l^T \lfloor \mathbf{u}_k \rfloor - \lambda_l \lambda_k \lfloor \mathbf{u}_k \rfloor \mathbf{u}_l \mathbf{u}_k^T \lfloor \mathbf{u}_l \rfloor \\
&\quad - \lambda_k^2 \lfloor \mathbf{u}_l \rfloor \mathbf{u}_k \mathbf{u}_k^T \lfloor \mathbf{u}_l \rfloor - \lambda_l \lambda_k \lfloor \mathbf{u}_l \rfloor \mathbf{u}_k \mathbf{u}_l^T \lfloor \mathbf{u}_k \rfloor \\
&= -(\lambda_l - \lambda_k)^2 \lfloor \mathbf{u}_l \rfloor \mathbf{u}_k \mathbf{u}_k^T \lfloor \mathbf{u}_l \rfloor.
\end{aligned}
\tag{121}
$$

Thus,

$$
\begin{aligned}
\sum_{k=1, k\neq l}^{3} &\frac{2}{\lambda_l - \lambda_k} \left(\sum_{p=1}^{M_p} \mathbf{g}_{kl}^{\mathsf{p}}\right)^T \left(\sum_{q=1}^{M_p} \mathbf{g}_{kl}^{\mathsf{q}}\right) \\
&= 2 \sum_{k=1, k\neq l}^{3} \begin{bmatrix} (\lambda_k - \lambda_l)\lfloor \mathbf{u}_l \rfloor \mathbf{u}_k \mathbf{u}_k^T \lfloor \mathbf{u}_l \rfloor & \mathbf{0} \\ \mathbf{0} & 0 \end{bmatrix}
\end{aligned}
\tag{122}
$$

Now from (117) and (122), the equation (109) is turned into

$$
\delta \mathbf{T}^T \mathbf{H} \delta \mathbf{T} = \mathbf{w}^T \begin{bmatrix} 2\mathbf{L}_l & \mathbf{0}_{3\times 1} \\ \mathbf{0}_{1\times 3} & 0 \end{bmatrix} \mathbf{w},
\tag{123}
$$

where

$$
\begin{aligned}
\mathbf{L}_l &= \lambda_l \lfloor \mathbf{u}_l \rfloor^2 - \lfloor \mathbf{u}_l \rfloor \mathbf{A} \lfloor \mathbf{u}_l \rfloor + \sum_{k=1, k\neq l}^{3} (\lambda_k - \lambda_l)\lfloor \mathbf{u}_l \rfloor \mathbf{u}_k \mathbf{u}_k^T \lfloor \mathbf{u}_l \rfloor \\
&= \lambda_l \lfloor \mathbf{u}_l \rfloor^2 - \lfloor \mathbf{u}_l \rfloor \mathbf{A} \lfloor \mathbf{u}_l \rfloor \\
&\quad + \lfloor \mathbf{u}_l \rfloor \sum_{k=1, k\neq l}^{3} \lambda_k \mathbf{u}_k \mathbf{u}_k^T \lfloor \mathbf{u}_l \rfloor - \lambda_l \lfloor \mathbf{u}_l \rfloor \sum_{k=1, k\neq l}^{3} \mathbf{u}_k \mathbf{u}_k^T \lfloor \mathbf{u}_l \rfloor.
\end{aligned}
\tag{124}
$$

Since $\mathbf{u}_k$ $(k = 1, 2, 3)$ is the eigenvector (with eigenvalue $\lambda_k$) of matrix $\mathbf{A}$, which is symmetric, we have the following two conditions from the singular value decomposition of $\mathbf{A}$:

$$
\mathbf{A} = \sum_{k=1}^{3} \lambda_k \mathbf{u}_k \mathbf{u}_k^T, \quad \mathbf{I} = \sum_{k=1}^{3} \mathbf{u}_k \mathbf{u}_k^T,
\tag{125}
$$

which imply

$$
\sum_{k=1, k\neq l}^{3} \lambda_k \mathbf{u}_k \mathbf{u}_k^T = \mathbf{A} - \lambda_l \mathbf{u}_l \mathbf{u}_l^T, \quad \sum_{k=1, k\neq l}^{3} \mathbf{u}_k \mathbf{u}_k^T = \mathbf{I} - \mathbf{u}_l \mathbf{u}_l^T.
\tag{126}
$$

Substituting the above results into $\mathbf{L}_l$ in (124):

$$
\begin{aligned}
\mathbf{L}_l &= \lambda_l \lfloor \mathbf{u}_l \rfloor^2 - \lfloor \mathbf{u}_l \rfloor \mathbf{A} \lfloor \mathbf{u}_l \rfloor \\
&\quad + \lfloor \mathbf{u}_l \rfloor (\mathbf{A} - \lambda_l \mathbf{u}_l \mathbf{u}_l^T) \lfloor \mathbf{u}_l \rfloor - \lambda_l \lfloor \mathbf{u}_l \rfloor (\mathbf{I} - \mathbf{u}_l \mathbf{u}_l^T) \lfloor \mathbf{u}_l \rfloor \\
&= \lambda_l \lfloor \mathbf{u}_l \rfloor^2 - \lfloor \mathbf{u}_l \rfloor \mathbf{A} \lfloor \mathbf{u}_l \rfloor + \lfloor \mathbf{u}_l \rfloor \mathbf{A} \lfloor \mathbf{u}_l \rfloor - \lambda_l \lfloor \mathbf{u}_l \rfloor^2 = \mathbf{0}.
\end{aligned}
\tag{127}
$$

As a result,

$$
\delta \mathbf{T}^T \mathbf{H} \delta \mathbf{T} = \mathbf{w}^T \begin{bmatrix} 2\mathbf{L}_l & \mathbf{0}_{3\times 1} \\ \mathbf{0}_{1\times 3} & 0 \end{bmatrix} \mathbf{w} = 0.
\tag{128}
$$

Finally, for the results in (40), if p or q $\in \mathcal{I} \triangleq \{j | 1 \leq j \leq M_p, \mathbf{C}_j = \mathbf{0}\}$, we have $\mathbf{g}_{kl}^{\mathsf{p}} = \mathbf{0}_{1\times 6}$ (or $\mathbf{g}_{kl}^{\mathsf{q}} = \mathbf{0}_{1\times 6}$) and $\mathbf{H}_{kl}^{\mathsf{p},\mathsf{q}} = \mathbf{0}_{6\times 6}, \forall k, l$, from (19). As a result, $\mathbf{J}^{\mathsf{p}} = \mathbf{g}_{kl}^{\mathsf{p}} = \mathbf{0}$ and $\mathbf{H}^{\mathsf{p},\mathsf{q}} = \mathbf{H}_{ll}^{\mathsf{p},\mathsf{q}} + \sum_{k=1, k\neq l}^{3} \frac{2}{\lambda_l - \lambda_k} (\mathbf{g}_{kl}^{\mathsf{p}})^T \mathbf{g}_{kl}^{\mathsf{q}} = \mathbf{0}$. $\square$

## F. Derivation of pose covariance

The quantity $\delta \mathbf{C}_{f_{ij}}$ can be obtained by substituting (43) into the definition of $\mathbf{C}_{f_{ij}}^{\mathsf{gt}}$ in (44) and retaining only the first order items:

$$
\delta \mathbf{C}_{f_{ij}} = \begin{bmatrix} \delta \mathbf{P}_{f_{ij}} & \delta \mathbf{v}_{f_{ij}} \\ \delta \mathbf{v}_{f_{ij}}^T & 0 \end{bmatrix}, \quad \text{where}
\tag{129}
$$

$$
\delta \mathbf{P}_{f_{ij}} = \sum_{k=1}^{N_{ij}} (\mathbf{p}_{f_{ijk}} \delta \mathbf{p}_{f_{ijk}}^T + \delta \mathbf{p}_{f_{ijk}} \mathbf{p}_{f_{ijk}}^T),
\tag{130}
$$

$$
\delta \mathbf{v}_{f_{ij}} = \sum_{k=1}^{N_{ij}} \delta \mathbf{p}_{f_{ijk}}.
\tag{131}
$$

To derive $\frac{\partial \mathbf{J}^T (\mathbf{T}^\star, \mathbf{C}_f)}{\partial \mathbf{C}_f} \delta \mathbf{C}_f$ with involving any tensor, we parameterize the matrix $\mathbf{C}_{f_{ij}}$ by a column vector $\mathbf{c}_{f_{ij}}$, which consists of the independent elements in $\mathbf{C}_{f_{ij}}$:

$$
\begin{aligned}
\mathbf{c}_{f_{ij}} = \mathrm{vec}(\mathbf{C}_{f_{ij}}) \triangleq [ &\mathbf{e}_1^T \mathbf{C}_{f_{ij}} \mathbf{e}_1 \quad \mathbf{e}_1^T \mathbf{C}_{f_{ij}} \mathbf{e}_2 \quad \mathbf{e}_1^T \mathbf{C}_{f_{ij}} \mathbf{e}_3 \\
&\mathbf{e}_2^T \mathbf{C}_{f_{ij}} \mathbf{e}_2 \quad \mathbf{e}_2^T \mathbf{C}_{f_{ij}} \mathbf{e}_3 \quad \mathbf{e}_3^T \mathbf{C}_{f_{ij}} \mathbf{e}_3 \\
&\mathbf{e}_1^T \mathbf{C}_{f_{ij}} \mathbf{e}_4 \quad \mathbf{e}_2^T \mathbf{C}_{f_{ij}} \mathbf{e}_4 \quad \mathbf{e}_3^T \mathbf{C}_{f_{ij}} \mathbf{e}_4 ]^T \in \mathbb{R}^9
\end{aligned}
\tag{132}
$$

where $\mathrm{vec}(\cdot): \mathbb{S}^{4\times 4} \mapsto \mathbb{R}^9$ maps a symmetric matrix to its column vector representation, $\mathbf{e}_l \in \mathbb{R}^4$ $(l \in \{1, 2, 3, 4\})$ is a vector with all zero elements except for the $l$-th element being one. Note that the constant $N$ in the 4-th row, 4-th column of $\mathbf{C}_{f_{ij}}$ is not contained in $\mathbf{c}_{f_{ij}}$ since it is a constant number independent of the noise. Correspondingly, noises in $\mathbf{C}_{f_{ij}}$ becomes the noise of $\mathbf{c}_{f_{ij}}$ as below:

$$
\begin{aligned}
\delta \mathbf{c}_{f_{ij}} = [ &\mathbf{e}_1^T \delta \mathbf{C}_{f_{ij}} \mathbf{e}_1 \quad \mathbf{e}_1^T \delta \mathbf{C}_{f_{ij}} \mathbf{e}_2 \quad \mathbf{e}_1^T \delta \mathbf{C}_{f_{ij}} \mathbf{e}_3 \\
&\mathbf{e}_2^T \delta \mathbf{C}_{f_{ij}} \mathbf{e}_2 \quad \mathbf{e}_2^T \delta \mathbf{C}_{f_{ij}} \mathbf{e}_3 \quad \mathbf{e}_3^T \delta \mathbf{C}_{f_{ij}} \mathbf{e}_3 \\
&\mathbf{e}_1^T \delta \mathbf{C}_{f_{ij}} \mathbf{e}_4 \quad \mathbf{e}_2^T \delta \mathbf{C}_{f_{ij}} \mathbf{e}_4 \quad \mathbf{e}_3^T \delta \mathbf{C}_{f_{ij}} \mathbf{e}_4 ]^T
\end{aligned}
\tag{133}
$$

$$
= \sum_{k=1}^{N_{ij}} \begin{bmatrix} 2\mathbf{p}_{f_{ijk}}^T \mathbf{S}_{\mathsf{P}} \mathbf{E}_{11} \mathbf{S}_{\mathsf{P}}^T \delta \mathbf{p}_{f_{ijk}} \\ \mathbf{p}_{f_{ijk}}^T \mathbf{S}_{\mathsf{P}} \mathbf{E}_{12} \mathbf{S}_{\mathsf{P}}^T \delta \mathbf{p}_{f_{ijk}} \\ \mathbf{p}_{f_{ijk}}^T \mathbf{S}_{\mathsf{P}} \mathbf{E}_{13} \mathbf{S}_{\mathsf{P}}^T \delta \mathbf{p}_{f_{ijk}} \\ 2\mathbf{p}_{f_{ijk}}^T \mathbf{S}_{\mathsf{P}} \mathbf{E}_{22} \mathbf{S}_{\mathsf{P}}^T \delta \mathbf{p}_{f_{ijk}} \\ \mathbf{p}_{f_{ijk}}^T \mathbf{S}_{\mathsf{P}} \mathbf{E}_{23} \mathbf{S}_{\mathsf{P}}^T \delta \mathbf{p}_{f_{ijk}} \\ 2\mathbf{p}_{f_{ijk}}^T \mathbf{S}_{\mathsf{P}} \mathbf{E}_{33} \mathbf{S}_{\mathsf{P}}^T \delta \mathbf{p}_{f_{ijk}} \\ \delta \mathbf{p}_{f_{ijk}} \end{bmatrix} = \sum_{k=1}^{N_{ij}} \mathbf{B}_{f_{ijk}} \delta \mathbf{p}_{f_{ijk}},
\tag{134}
$$

where $\mathbf{S}_{\mathsf{P}} = [\mathbf{I}_{3\times 3}, \mathbf{0}_{3\times 1}]$, $\mathbf{E}_{kl} = \mathbf{e}_k \mathbf{e}_l^T + \mathbf{e}_l \mathbf{e}_k^T \in \mathbb{S}^{4\times 4}$, $k, l \in \{1, 2, 3, 4\}$, and

$$
\mathbf{B}_{f_{ijk}} = \begin{bmatrix} 2\mathbf{p}_{f_{ijk}}^T \mathbf{S}_{\mathsf{P}} \mathbf{E}_{11} \mathbf{S}_{\mathsf{P}}^T \\ \mathbf{p}_{f_{ijk}}^T \mathbf{S}_{\mathsf{P}} \mathbf{E}_{12} \mathbf{S}_{\mathsf{P}}^T \\ \mathbf{p}_{f_{ijk}}^T \mathbf{S}_{\mathsf{P}} \mathbf{E}_{13} \mathbf{S}_{\mathsf{P}}^T \\ 2\mathbf{p}_{f_{ijk}}^T \mathbf{S}_{\mathsf{P}} \mathbf{E}_{22} \mathbf{S}_{\mathsf{P}}^T \\ \mathbf{p}_{f_{ijk}}^T \mathbf{S}_{\mathsf{P}} \mathbf{E}_{23} \mathbf{S}_{\mathsf{P}}^T \\ 2\mathbf{p}_{f_{ijk}}^T \mathbf{S}_{\mathsf{P}} \mathbf{E}_{33} \mathbf{S}_{\mathsf{P}}^T \\ \mathbf{I}_{3\times 3} \end{bmatrix} \in \mathbb{R}^{9\times 3}.
\tag{135}
$$

With the column representation of each $\mathbf{C}_{f_{ijk}}$ contained in $\mathbf{C}_f$, $\frac{\partial \mathbf{J}^T (\mathbf{T}^\star, \mathbf{C}_f)}{\partial \mathbf{C}_f} \delta \mathbf{C}_f$ can now be computed as

$$
\frac{\partial \mathbf{J}^T (\mathbf{T}^\star, \mathbf{C}_f)}{\partial \mathbf{C}_f} \delta \mathbf{C}_f = \sum_{i=1}^{M_f} \sum_{j=1}^{M_p} \frac{\partial \mathbf{J}^T (\mathbf{T}^\star, \mathbf{C}_{f_{ij}})}{\partial \mathbf{c}_{f_{ij}}} \delta \mathbf{c}_{f_{ij}}.
\tag{136}
$$

To derive the quantity $\frac{\partial \mathbf{J}^T(\mathbf{T}^\star, \mathbf{C}_{f_{ij}})}{\partial c_{f_{ij}}}$, we give two lemmas which are useful for subsequent deduction.

**Lemma 3.** *For* $\mathbf{w} \in \mathbb{R}^4, \mathbf{C} \in \mathbb{S}^{4 \times 4}$ *and its vector form* $\mathbf{c} = \text{vec}(\mathbf{C})$, *we have*

$$\frac{\partial \mathbf{C}\mathbf{w}}{\partial \mathbf{c}} = \mathbf{g}_1(\mathbf{w}) \in \mathbb{R}^{4 \times 9},$$

*where*

$$\mathbf{g}_1(\mathbf{w}) = \begin{bmatrix} \mathbf{E}_{11}\mathbf{w} & \mathbf{E}_{12}\mathbf{w} & \mathbf{E}_{13}\mathbf{w} & \mathbf{E}_{22}\mathbf{w} & \mathbf{E}_{23}\mathbf{w} \\ & \mathbf{E}_{33}\mathbf{w} & \mathbf{E}_{14}\mathbf{w} & \mathbf{E}_{24}\mathbf{w} & \mathbf{E}_{34}\mathbf{w} \end{bmatrix},$$

*where* $\mathbf{e}_l \in \mathbb{R}^4$ $(l \in \{1, 2, 3, 4\})$ *is a vector with all zero elements except for the* $l$*-th element being one,* $\mathbf{E}_{kl} \in \mathbb{S}^{4 \times 4}, k, l \in \{1, 2, 3, 4\}$.

$$\mathbf{E}_{kl} = \begin{cases} \mathbf{e}_k\mathbf{e}_l^T + \mathbf{e}_l\mathbf{e}_k^T & (k \neq l) \\ \mathbf{e}_l\mathbf{e}_l^T & (k = l) \end{cases} \tag{137}$$

*Proof.* For the $k$-th row, $l$-th column element of $\mathbf{C}$, denoted by $\mathbf{C}_{k,l}$, we have

$$\frac{\partial \mathbf{C}\mathbf{w}}{\partial \mathbf{C}_{k,l}} = \mathbf{E}_{kl}\mathbf{w} \in \mathbb{R}^4.$$

Hence,

$$\frac{\partial \mathbf{C}\mathbf{w}}{\partial \mathbf{c}} = \begin{bmatrix} \frac{\partial \mathbf{C}\mathbf{w}}{\partial \mathbf{C}_{1,1}} & \frac{\partial \mathbf{C}\mathbf{w}}{\partial \mathbf{C}_{1,2}} & \frac{\partial \mathbf{C}\mathbf{w}}{\partial \mathbf{C}_{1,3}} & \frac{\partial \mathbf{C}\mathbf{w}}{\partial \mathbf{C}_{2,2}} & \frac{\partial \mathbf{C}\mathbf{w}}{\partial \mathbf{C}_{2,3}} \\ \frac{\partial \mathbf{C}\mathbf{w}}{\partial \mathbf{C}_{3,3}} & \frac{\partial \mathbf{C}\mathbf{w}}{\partial \mathbf{C}_{1,4}} & \frac{\partial \mathbf{C}\mathbf{w}}{\partial \mathbf{C}_{2,4}} & \frac{\partial \mathbf{C}\mathbf{w}}{\partial \mathbf{C}_{3,4}} \end{bmatrix}$$

$$= \begin{bmatrix} \mathbf{E}_{11}\mathbf{w} & \mathbf{E}_{12}\mathbf{w} & \mathbf{E}_{13}\mathbf{w} & \mathbf{E}_{22}\mathbf{w} & \mathbf{E}_{23}\mathbf{w} \\ & \mathbf{E}_{33}\mathbf{w} & \mathbf{E}_{14}\mathbf{w} & \mathbf{E}_{24}\mathbf{w} & \mathbf{E}_{34}\mathbf{w} \end{bmatrix}.$$

$\square$

**Lemma 4.** *For* $\mathbf{u}_l \in \mathbb{R}^3$ *and*

$$\mathbf{w} = \begin{bmatrix} \mathbf{w}_{1:3} \\ w_4 \end{bmatrix} \in \mathbb{R}^4, \quad \mathbf{U}_l = \begin{bmatrix} -\lfloor \mathbf{u}_l \rfloor & \mathbf{0}_{3 \times 1} \\ \mathbf{0}_{3 \times 3} & \mathbf{u}_l \end{bmatrix} \in \mathbb{R}^{6 \times 4}$$

*where* $\mathbf{w}_{1:3}$ *represents the first three elements in* $\mathbf{w}$, *we have*

$$\frac{\partial \mathbf{U}_l\mathbf{w}}{\partial \mathbf{u}_l} = \mathbf{g}_2(\mathbf{w}) = \begin{bmatrix} \lfloor \mathbf{w}_{1:3} \rfloor \\ w_4\mathbf{I}_{3 \times 3} \end{bmatrix} \in \mathbb{R}^{6 \times 3}.$$

*Proof.*

$$\mathbf{U}_l\mathbf{w} = \begin{bmatrix} -\lfloor \mathbf{u}_l \rfloor & \mathbf{0} \\ \mathbf{0} & \mathbf{u}_l \end{bmatrix} \begin{bmatrix} \mathbf{w}_{1:3} \\ w_4 \end{bmatrix} = \begin{bmatrix} \lfloor \mathbf{w}_{1:3} \rfloor \mathbf{u}_l \\ w_4\mathbf{u}_l \end{bmatrix} = \begin{bmatrix} \lfloor \mathbf{w}_{1:3} \rfloor \\ w_4\mathbf{I}_{3 \times 3} \end{bmatrix} \mathbf{u}_l$$

Thus,

$$\frac{\partial \mathbf{U}_l\mathbf{w}}{\partial \mathbf{u}_l} = \begin{bmatrix} \lfloor \mathbf{w}_{1:3} \rfloor \\ w_4\mathbf{I}_{3 \times 3} \end{bmatrix} = \mathbf{g}_2(\mathbf{w})$$

$\square$

With these two lemmas, next we will continue to derive $\frac{\partial \mathbf{J}^T(\mathbf{T}^\star, \mathbf{C}_{f_{ij}})}{\partial c_{f_{ij}}}$. Theorem 4 gives the Jacobian for one cost item, to distinguish the different cost item, we add a subscript $\nu$ to (30) to denotes the $\nu$-th item and replace $\mathbf{C}_j$ with the actual

point cluster notation $\mathbf{C}_{f_{\nu j}}$ corresponding to the $\nu$-th item, leading to:

$$\mathbf{J}_\nu = \begin{bmatrix} \cdots & \mathbf{J}_\nu^\mathrm{p} & \cdots \end{bmatrix} \in \mathbb{R}^{1 \times 6M_p} \tag{138}$$

$$\mathbf{J}_\nu^\mathrm{p} = \frac{2}{N} \mathbf{u}_{\nu l}^T \mathbf{S}_\mathbf{P} \left( \mathbf{T}_\mathrm{p} - \frac{1}{N} \mathbf{C}_\nu \mathbf{F} \right) \mathbf{C}_{f_{\nu p}} \mathbf{T}_\mathrm{p}^T \mathbf{U}_{\nu l}^T \in \mathbb{R}^{1 \times 6}.$$

The total Jacobian is hence

$$\mathbf{J} = \sum_{\nu=1}^{M_f} \mathbf{J}_\nu = \begin{bmatrix} \cdots & \mathbf{J}^\mathrm{p} & \cdots \end{bmatrix} \in \mathbb{R}^{1 \times 6M_p}, \tag{139}$$

$$\mathbf{J}^\mathrm{p} = \sum_{\nu=1}^{M_f} \mathbf{J}_\nu^\mathrm{p} \in \mathbb{R}^{1 \times 6},$$

$$= \sum_{\nu=1}^{M_f} \left( \frac{2}{N} \mathbf{u}_{\nu l}^T \mathbf{S}_\mathbf{P} \left( \mathbf{T}_\mathrm{p} - \frac{1}{N} \mathbf{C}_\nu \mathbf{F} \right) \mathbf{C}_{f_{\nu p}} \mathbf{T}_\mathrm{p}^T \mathbf{U}_{\nu l}^T \right). \tag{140}$$

Next, we calculate the partial derivative $\frac{\partial (\mathbf{J}^\mathrm{p})^T}{\partial c_{f_{ij}}}$. Note that in the summation of (140), only the $i$-th summation term (i.e., $\nu = i$) is related to $\mathbf{c}_{f_{ij}}$, hence,

$$\frac{\partial (\mathbf{J}^\mathrm{p})^T}{\partial \mathbf{c}_{f_{ij}}} = \frac{\partial}{\partial \mathbf{c}_{f_{ij}}} \left( \frac{2}{N_i} \mathbf{U}_{il} \mathbf{T}_\mathrm{p} \mathbf{C}_{f_{ip}} (\mathbf{T}_\mathrm{p}^T - \frac{1}{N_i} \mathbf{F} \mathbf{C}_i) \mathbf{S}_\mathbf{P}^T \mathbf{u}_{il} \right). \tag{141}$$

Since $\mathbf{C}_i = \sum_{\nu=1}^{M_p} \mathbf{T}_\nu \mathbf{C}_{f_{i\nu}} \mathbf{T}_\nu^T$, $\mathbf{u}_{il}$ is the eigenvector associated to the $l$-th largest eigenvalue of matrix $\mathbf{A}(\mathbf{C}_i)$, and $\mathbf{U}_{il} = \begin{bmatrix} -\lfloor \mathbf{u}_{il} \rfloor & \mathbf{0}_{3 \times 1} \\ \mathbf{0}_{3 \times 3} & \mathbf{u}_{il} \end{bmatrix}$, the derivative with respect to $\mathbf{c}_{f_{ij}}$ on the right hand side of (141) consists of four terms from $\mathbf{U}_{il}$, $\mathbf{C}_{f_{ip}}$ (only when $\mathrm{p} = j$), $\mathbf{C}_i$, and $\mathbf{u}_{il}$, respectively. Combining with Lemma 3 and Lemma 4, we have

$$\mathbf{L}_{ij}^\mathrm{p} \triangleq \frac{\partial (\mathbf{J}^\mathrm{p})^T}{\partial \mathbf{c}_{f_{ij}}} \in \mathbb{R}^{6 \times 9} \tag{142}$$

$$= \frac{2}{N_i} \left( \mathbf{g}_2(\mathbf{T}_\mathrm{p} \mathbf{C}_{f_{ip}} (\mathbf{T}_\mathrm{p}^T - \frac{1}{N_i} \mathbf{F} \mathbf{C}_i) \mathbf{S}_\mathbf{P}^T \mathbf{u}_{il}) \right.$$

$$+ \mathbf{U}_{il} \mathbf{T}_\mathrm{p} \mathbf{C}_{f_{ip}} (\mathbf{T}_\mathrm{p}^T - \frac{1}{N} \mathbf{F} \mathbf{C}_i) \mathbf{S}_\mathbf{P}^T) \frac{\partial \mathbf{u}_{il}}{\partial \mathbf{c}_{f_{ij}}}$$

$$- \frac{2}{N_i^2} \mathbf{U}_{il} \mathbf{T}_\mathrm{p} \mathbf{C}_{f_{ip}} \mathbf{F} \mathbf{T}_j \mathbf{g}_1(\mathbf{T}_j^T \mathbf{S}_\mathbf{P}^T \mathbf{u}_{il})$$

$$\left. + \frac{2}{N_i} \mathbb{1}_{\mathrm{p}=j} \left( \mathbf{U}_{il} \mathbf{T}_\mathrm{p} \mathbf{g}_1((\mathbf{T}_\mathrm{p}^T - \frac{1}{N_i} \mathbf{F} \mathbf{C}_i) \mathbf{S}_\mathbf{P}^T \mathbf{u}_{il}) \right) \right) \tag{143}$$

Only the term $\frac{\partial \mathbf{u}_{il}}{\partial \mathbf{c}_{f_{ij}}}$ is unknown in this formula. To compute it, we apply the second conclusion in Lemma 1 to all components of $\mathbf{c}_{f_{ij}}$ and use the notation trick similar to (81):

$$\frac{\partial \mathbf{u}_{il}}{\partial \mathbf{c}_{f_{ij}}} = \sum_{k=1, k \neq l}^3 \frac{1}{\lambda_{il} - \lambda_{ik}} \mathbf{u}_{ik} \mathbf{u}_{ik}^T \frac{\partial \mathbf{A}(\mathbf{C}_i) \mathbf{u}_{il}}{\partial \mathbf{c}_{f_{ij}}} \tag{144}$$

From Lemma 2,

$$\mathbf{u}_{ik}^T \mathbf{A}(\mathbf{C}_i) \mathbf{u}_{il} = \frac{1}{N_i} \mathbf{u}_{ik}^T \mathbf{S}_\mathbf{P}$$

$$\left( \sum_{\mu=1}^{M_p} \mathbf{T}_\mu \mathbf{C}_{f_{i\mu}} \mathbf{T}_\mu^T - \frac{1}{N_i} \sum_{\mu=1}^{M_p} \sum_{\nu=1}^{M_p} \mathbf{T}_\mu \mathbf{C}_{f_{i\mu}} \mathbf{F} \mathbf{C}_{f_{i\nu}} \mathbf{T}_\nu^T \right) \mathbf{S}_\mathbf{P}^T \mathbf{u}_{il}$$

Denote

$$\mathbf{G}_{kl}^{ij} \triangleq \frac{\partial \mathbf{u}_{ik}^T \mathbf{A}(\mathbf{C}_{f_i})\mathbf{u}_{il}}{\partial \mathbf{c}_{f_{ij}}} = \frac{1}{N_i}\mathbf{u}_{ik}^T \mathbf{S_P}\mathbf{T}_j \mathbf{g}_1(\mathbf{T}_j^T \mathbf{S_P}^T \mathbf{u}_{il})$$

$$- \frac{1}{N_i^2}\mathbf{u}_{ik}^T \mathbf{S_P} \sum_{\mu=1}^{M_p} \mathbf{T}_\mu \mathbf{C}_{f_{i\mu}} \mathbf{F}\mathbf{g}_1(\mathbf{T}_j^T \mathbf{S_P}^T \mathbf{u}_{il})$$

$$- \frac{1}{N_i^2}\mathbf{u}_{ik}^T \mathbf{S_P}\mathbf{T}_j \mathbf{g}_1\Big(\sum_{\nu=1}^{M_p} \mathbf{F}\mathbf{C}_{f_{i\nu}} \mathbf{T}_\nu^T \mathbf{S_P}^T \mathbf{u}_{il}\Big) \in \mathbb{R}^{1\times 9}$$

$$= \frac{1}{N_i}\mathbf{u}_{ik}^T \mathbf{S_P}\Big(\mathbf{T}_j \mathbf{g}_1(\mathbf{T}_j^T \mathbf{S_P}^T \mathbf{u}_{il}) - \frac{1}{N_i}\mathbf{C}_i \mathbf{F}\mathbf{g}_1(\mathbf{T}_j^T \mathbf{S_P}^T \mathbf{u}_{il})$$

$$- \frac{1}{N_i}\mathbf{T}_j \mathbf{g}_1(\mathbf{F}\mathbf{C}_i \mathbf{S_P}^T \mathbf{u}_{il})\Big) \tag{145}$$

Substitute it into $\mathbf{L}_{ij}^{\mathrm{p}}$:

$$\mathbf{L}_{ij}^{\mathrm{p}} = \frac{2}{N_i}\Big(\mathbf{g}_2(\mathbf{T_p}\mathbf{C}_{f_{ip}}(\mathbf{T_p}^T - \frac{1}{N_i}\mathbf{F}\mathbf{C}_i)\mathbf{S_P}^T \mathbf{u}_{il}) +$$

$$\mathbf{U}_{il}\mathbf{T_p}\mathbf{C}_{f_{ip}}(\mathbf{T_p}^T - \frac{1}{N}\mathbf{F}\mathbf{C}_i)\mathbf{S_P}^T\Big)\left(\sum_{k=1,k\neq l}^{3}\frac{\mathbf{u}_{ik}\mathbf{G}_{kl}^{ij}}{\lambda_{il} - \lambda_{ik}}\right)$$

$$- \frac{2}{N_i^2}\mathbf{U}_{il}\mathbf{T_p}\mathbf{C}_{f_{ip}}\mathbf{F}\mathbf{T}_j \mathbf{g}_1(\mathbf{T}_j^T \mathbf{S_P}^T \mathbf{u}_{il})$$

$$+ \frac{2}{N_i}\mathbb{1}_{\mathrm{p}=j}\left(\mathbf{U}_{il}\mathbf{T_p}\mathbf{g}_1\Big((\mathbf{T_p}^T - \frac{1}{N_i}\mathbf{F}\mathbf{C}_i)\mathbf{S_P}^T \mathbf{u}_{il}\Big)\right), \tag{146}$$

$$\mathbf{L}_{ij} = \frac{\partial \mathbf{J}^T\left(\mathbf{T}^\star, \mathbf{C}_{f_{ij}}\right)}{\partial \mathbf{c}_{f_{ij}}} = \begin{bmatrix} \vdots \\ \mathbf{L}_{ij}^{\mathrm{p}} \\ \vdots \end{bmatrix} \in \mathbb{R}^{6M_p \times 9}, \tag{147}$$

and

$$\delta \mathbf{T}^\star = -\mathbf{H}^{-1}\left(\sum_{i=1}^{M_f}\sum_{j=1}^{M_p}\mathbf{L}_{ij}\delta\mathbf{c}_{f_{ij}}\right) \in \mathbb{R}^{6M_p}, \tag{148}$$

$$\mathbf{\Sigma}_{\delta\mathbf{T}^\star} = \mathbf{H}^{-1}\left(\sum_{i=1}^{M_f}\sum_{j=1}^{M_p}\mathbf{L}_{ij}\mathbf{\Sigma}_{\mathbf{c}_{f_{ij}}}\mathbf{L}_{ij}^T\right)\mathbf{H}^{-1}, \tag{149}$$

where

$$\mathbf{\Sigma}_{\mathbf{c}_{f_{ij}}} = \sum_{k=1}^{N_{ij}}\mathbf{B}_{f_{ijk}}\mathbf{\Sigma}_{\mathbf{p}_{f_{ijk}}}\mathbf{B}_{f_{ijk}}^T, \tag{150}$$

which can be computed beforehand without enumerating each raw point in the run time.



\ifCLASSOPTIONcaptionsoff
  \newpage
\fi

\end{document}